\documentclass[twoside,11pt]{article}
\usepackage{jmlr2e}
\usepackage{amsmath}
\usepackage{amsfonts}
\usepackage{amssymb}
\usepackage{algorithm,algorithmic}
\usepackage{multirow}
\usepackage{array}
\usepackage{url}
\usepackage{subfigure}

\def\x{{\mathbf x}}

\def\l{{\ell}}
\def\1{{\mathbf 1}}
\def\a{{\mathbf a}}

\def\v{{\mathbf v}}
\def\X{{\mathbf X}}

\def\hatf{{ \hat f}}

\def\tildel{\tilde{\ell}}

\def\as{\text{~~a.s.}}

\def\Gammab{{\boldsymbol\Gamma}}

\def\epsilonb{{\boldsymbol\varepsilon}}
\def\alphab{{\boldsymbol\alpha}}

\def\Gammab{{\boldsymbol\Gamma}}

\def\w{{\mathbf w}}
\def\D{{\mathbf D}}
\def\Y{{\mathbf Y}}

\def\B{{\mathbf B}}

\def\d{{\mathbf d}}
\def\E{{\mathbb E}}
\def\O{{O}}

\def\C{{\mathcal C}}

\def\L{{\mathcal L}}
\def\P{{\mathcal P}}

\def\PPP{{\mathbb P}}

\def\d{{\mathbf d}}
\def\e{{\mathbf e}}
\def\b{{\mathbf b}}
\def\u{{\mathbf u}}

\def\hatf{{ \hat f}}

\def\Real{{\mathbb R}}

\def\U{{\mathbf U}}
\def\u{{\mathbf u}}
\def\A{{\mathbf A}}
\def\G{{\mathbb G}}

\def\FF{{\mathcal F}}

\def\I{{\mathbf I}}

\def\argmin{\operatornamewithlimits{arg\,min}}

\def\trace{\operatorname{Tr}}

\def\sign{\operatorname{sign}}
\def\FL{\operatorname{FL}}
\def\corr{\operatorname{corr}}
\def\cov{\operatorname{cov}}
\def\st{~~\text{s.t.}~~}
\def\defin{\stackrel{\vartriangle}{=}}

\begin{document}

\title{Online Learning for Matrix Factorization and Sparse Coding}

\author{\name Julien Mairal \email julien.mairal@inria.fr \\
        \name Francis Bach \email francis.bach@inria.fr \\
        \addr INRIA - WILLOW Project-Team\\
         Laboratoire d'Informatique de l'Ecole Normale Sup\'erieure (INRIA/ENS/CNRS UMR 8548)\\
         23, avenue d'Italie 75214 Paris CEDEX 13, France \AND
        \name Jean Ponce \email jean.ponce@ens.fr \\
        \addr Ecole Normale Sup\'erieure - WILLOW Project-Team\\
        Laboratoire d'Informatique de l'Ecole Normale Sup\'erieure (INRIA/ENS/CNRS UMR 8548)\\
        45, rue d'Ulm 75230 Paris CEDEX 05, France  \AND
        \name Guillermo Sapiro \email guille@umn.edu \\
        \addr Department of Electrical and Computer Engineering \\
        University of Minnesota \\
        200 Union Street SE, Minneapolis, MN 55455, USA}

\editor{~}

\hyphenation{dictionary}

\maketitle
\begin{abstract}%
   Sparse coding---that is, modelling data vectors as sparse linear
   combinations of basis elements---is widely used in machine learning,
   neuroscience, signal processing, and statistics. This paper focuses on the
   large-scale matrix factorization problem that consists of {\em learning}
   the basis set in order to adapt it to specific data. Variations of this
   problem include dictionary learning in signal processing, non-negative
   matrix factorization and sparse principal component analysis. In this
   paper, we propose to address these tasks with a new online optimization
   algorithm, based on stochastic approximations, which scales up gracefully
   to large data sets with millions of training samples, and extends naturally
   to various  matrix factorization formulations, making it suitable for a
   wide range of learning problems. A proof of convergence is presented,
   along with experiments with natural images and genomic data demonstrating
   that it leads to state-of-the-art performance in terms of speed and
   optimization for both small and large data sets.
\end{abstract}


\section{Introduction}
The linear decomposition of a signal using a few atoms of a \emph{learned}
dictionary instead of a predefined one---based on wavelets \citep{mallat} for
example---has recently led to state-of-the-art results in numerous low-level
signal processing tasks such as image denoising \citep{elad,mairal}, texture
synthesis \citep{peyre} and audio processing \citep{grosse,fevotte,zibulevsky}, as well as
higher-level tasks such as image classification \citep{raina,mairal3,mairal6,bradley,yang},
showing that sparse learned models are well adapted to natural signals.
Unlike decompositions based on principal component analysis and its variants,
these models do not impose that the basis vectors be orthogonal, allowing
more flexibility to adapt the representation to the data.\footnote{Note that
the terminology ``basis'' is slightly abusive here since the elements of the
dictionary are not necessarily linearly independent and the set can be overcomplete---that is, have more elements than the signal dimension.} In machine learning and
statistics, slightly different matrix factorization problems are formulated 
in order to obtain a few \emph{interpretable} basis elements from a set of data
vectors. This includes non-negative matrix factorization and its variants
\citep{lee2,hoyer,hoyer2,lin}, and sparse principal component analysis
\citep{zou2,aspremont,bach2,witten,zass}. As
shown in this paper, these problems have strong similarities; even though we
first focus on the problem of dictionary learning, the algorithm we propose
is able to address all of them.  While learning the dictionary has proven to
be critical to achieve (or improve upon) state-of-the-art results in signal
and image processing, effectively solving the corresponding optimization
problem is a significant computational challenge, particularly in the context
of large-scale data sets that may include millions of training samples.
Addressing this challenge and designing a generic algorithm which is 
capable of efficiently handling various matrix factorization problems,
is the topic of this paper. 

Concretely, consider a signal $\x$ in $\Real^m$. We say that it admits a
sparse approximation over a \mbox{{\em dictionary}} $\D$ in $\Real^{m \times k}$,
with $k$ columns referred to as {\em atoms}, when one can find a linear
combination of a ``few'' atoms from $\D$ that is ``close'' to the signal
$\x$.  Experiments have shown that modelling a signal with such a sparse
decomposition ({\em sparse coding}) is very effective in many signal
processing applications \citep{chen}.  For natural images, predefined
dictionaries based on various types of wavelets \citep{mallat} have also been used
for this task.  However, learning the dictionary instead of using
off-the-shelf bases has been shown to dramatically improve signal
reconstruction \citep{elad}.  Although some of the learned dictionary elements
may sometimes ``look like'' wavelets (or Gabor filters), they are tuned to the
input images or signals, leading to much better results in practice.

Most recent algorithms for dictionary
learning \citep{field,engan,lewicki,aharon,lee} are iterative
{\em batch} procedures, accessing the whole training set at each iteration in
order to minimize a cost function under some constraints, and cannot
efficiently deal with very large training sets \citep{bottou}, or dynamic
training data changing over time, such as video sequences. To address these
issues, we propose an {\em online} approach that processes the signals, one
at a time, or in mini-batches.  This is particularly important in the
context of image and video processing \citep{protter,mairal5}, where it is
common to learn dictionaries adapted to small patches, with training data
that may include several millions of these patches (roughly one per pixel and
per frame). In this setting, online techniques based on stochastic
approximations are an attractive alternative to batch
methods~(see, e.g., \citealp{bottou2,kushner,shwartz}).  For example, first-order stochastic gradient
descent with projections on the constraint set \citep{kushner} is sometimes
used for dictionary learning (see \citealp{aharon2,kavukcuoglu} for instance).
We show in this paper that it is possible to go further and exploit the
specific structure of sparse coding in the design of an optimization
procedure tuned to this problem, with low memory consumption and lower
computational cost than classical batch algorithms.  As demonstrated by our
experiments, it scales up gracefully to large data sets with millions of
training samples, is easy to use, and is faster than competitive methods.

The paper is structured as follows: Section~\ref{sec:diclearn} presents
the dictionary learning problem.  The proposed method is introduced in Section
\ref{sec:online}, with a proof of convergence in Section~\ref{sec:conv}.
Section~\ref{sec:variations} extends our algorithm to various matrix factorization
problems that generalize dictionary learning, and Section
\ref{sec:exp_dict} is devoted to experimental results, demonstrating
that our algorithm is suited to a wide class of learning problems.
\subsection{Contributions}
This paper makes four main contributions: 
\begin{itemize}
   \item We cast in Section~\ref{sec:diclearn} the dictionary learning
      problem as the optimization of  a smooth nonconvex objective function
      over a convex set,  minimizing the (desired) {\em expected} cost
      when the training set size goes
      to infinity, and propose in Section~\ref{sec:online} an iterative online algorithm
      that solves this problem by efficiently minimizing at each step a
      quadratic surrogate function of the empirical cost over the set of
      constraints.  This method is shown in Section~\ref{sec:conv} to
      converge almost surely to a stationary point of the objective
      function. 
   \item As shown experimentally in Section~\ref{sec:exp_dict}, our algorithm
      is significantly faster than previous approaches to dictionary learning
      on both small and large data sets of natural images. To demonstrate that
      it is adapted to difficult, large-scale image-processing tasks, we
      learn a dictionary on a $12$-Megapixel photograph and use it for
      inpainting---that is, filling some holes in the image.
   \item We show in Sections~\ref{sec:variations} and \ref{sec:exp_dict} that our approach
      is suitable to large-scale matrix factorization problems such as
      non-negative matrix factorization and sparse principal component
      analysis, while being still effective on small data sets.
   \item To extend our algorithm to several matrix factorization problems, 
      we propose in Appendix \ref{appendix:proj} efficient procedures for
      projecting onto two convex sets, which can be useful for other
      applications that are beyond the scope of this paper.
\end{itemize}
\subsection{Notation}
We define for $p \geq 1$ the $\ell_p$ norm of a vector $\x$ in
$\Real^m$ as $||\x||_p \defin (\sum_{i=1}^m |\x[i]|^p)^{{1}/{p}}$,
where~$\x[i]$ denotes the $i$-th coordinate of $\x$ and 
$||\x||_\infty \defin \max_{i=1,\ldots,m} |\x[i]| = \lim_{p \to \infty} ||\x||_p$.
We also define the $\ell_0$ pseudo-norm as the sparsity measure which counts
the number of nonzero elements in a vector:\footnote{Note that it would
be more proper to write $||\x||_0^0$ instead of $||\x||_0$ to be consistent with the traditional notation $||\x||_p$.
However, for the sake of simplicity, we will keep this
notation unchanged.} $||\x||_0 \defin \#\{i \st \x[i] \neq 0  \} = \lim_{p \to 0^+}  (\sum_{i=1}^m |\x[i]|^p)$.
We denote the Frobenius norm of a matrix $\X$ in $\Real^{m \times n}$ by
$||\X||_F \defin (\sum_{i=1}^m \sum_{j=1}^n \X[i,j]^2)^{{1}/{2}}$.
For a sequence of vectors (or
matrices) $\x_t$ and scalars $u_t$, we write $\x_t = \O(u_t)$
when there exists a constant $K > 0$ so that for all $t$,
$||\x_t||_2 \leq K u_t$.
Note that for finite-dimensional vector spaces, the choice of norm is essentially irrelevant (all norms are equivalent).
Given two matrices $\A$ in $\Real^{m_1 \times n_1}$ and $\B$ in $\Real^{m_2 \times n_2}$,
$\A \otimes \B$ denotes the Kronecker product between $\A$ and $\B$,
defined as the matrix in $\Real^{m_1m_2 \times n_1n_2}$, defined by blocks of
sizes $m_2 \times n_2$ equal to $\A[i,j] \B$.
For more details and properties of the Kronecker product, see \citet{golub}, and \citet{magnus}.
\section{Problem Statement} \label{sec:diclearn} 
Classical dictionary learning techniques for sparse representation \citep{field,engan,lewicki,aharon,lee}
consider a finite training set of signals $\X = [\x_1,\ldots,\x_n]$ in $\Real^{m \times n}$ 
and optimize the empirical cost function
\begin{equation}
   f_n(\D) \defin \frac{1}{n} \sum_{i=1}^n \l(\x_i,\D), \label{eq:empirical}
\end{equation}
where $\D$ in $\Real^{m \times k}$ is the dictionary, each column
representing a basis vector, and $\l$ is a loss function such that
$\l(\x,\D)$ should be small if $\D$ is ``good'' at representing the
signal~$\x$ in a sparse fashion. The number of samples
$n$ is usually large, whereas the signal dimension $m$ is relatively
small, for example, $m = 100$ for $10 \times 10$ image patches, and $n
\geq 100, 000$ for typical image processing applications.  In general,
we also have $k \ll n$ (e.g., $k=200$ for $n=100, 000$), but each
signal only uses a few elements of $\D$ in its representation, say
$10$ for instance. Note that, in this setting, overcomplete dictionaries with $k >
m$ are allowed. As others (see for example \citealp{lee}), we define
$\l(\x,\D)$ as the optimal value of the $\ell_1$ {\em sparse coding}
problem:
\begin{equation}
   \l(\x,\D) \defin \min_{\alphab \in \Real^k} \frac{1}{2}||\x-\D\alphab||_2^2 +
   \lambda ||\alphab||_1,\label{eq:ell1penalty}
\end{equation}
where $\lambda$ is a regularization parameter. 
This problem is also known as {\em basis pursuit} \citep{chen}, or the {\em
Lasso} \citep{tibshirani}.\footnote{To be more precise, the original formulation
of the Lasso is a constrained version of Eq.~(\ref{eq:ell1penalty}), with a
constraint on the $\ell_1$-norm of $\alphab$:
\begin{equation}
\min_{\alphab \in \Real^k} \frac{1}{2}||\x-\D\alphab||_2^2 \st ||\alphab||_1 \leq T.\label{eq:constraint}
\end{equation}
Both formulations are equivalent in the sense that for every $\lambda > 0$ (respectively every $T > 0$), there exists
a scalar $T$ (respectively $\lambda$) so that Equations (\ref{eq:ell1penalty}) and (\ref{eq:constraint}) admit
the same solutions.
 } It is well known that 
$\ell_1$ regularization yields a sparse solution for $\alphab$, but there is no
direct analytic link between the value of $\lambda$ and the corresponding
effective sparsity $||\alphab||_0$. To prevent $\D$ from having
arbitrarily large values (which would lead to arbitrarily small
values of $\alphab$), it is common to constrain its columns
$\d_1,\ldots,\d_k$ to have an $\ell_2$-norm less than or equal to one.
We will call $\C$ the convex set of matrices verifying this
constraint:
\begin{displaymath}
   \C \defin \{ \D \in \Real^{m \times k} \st \forall j=1,\ldots,k,~~ \d_j^T\d_j \leq 1 \}.
\end{displaymath}
Note that the problem of minimizing the empirical cost $f_n(\D)$ is
not convex with respect to~$\D$. It can be rewritten as a joint
optimization problem with respect to the dictionary~$\D$ and the
coefficients $\alphab = [\alphab_1,\ldots,\alphab_n]$ in $\Real^{k \times n}$ of the sparse
decompositions, which is not jointly convex, but convex with respect to
each of the two variables $\D$ and $\alphab$ when the other one is
fixed:
\begin{equation}
   \min_{\D \in \C,\alphab \in \Real^{k \times n}} \sum_{i=1}^n
   \Big(\frac{1}{2}||\x_i-\D\alphab_i||_2^2 + \lambda ||\alphab_i||_1\Big).
   \label{eq:joint} 
\end{equation}
This can be rewritten as a \emph{matrix factorization} problem with a sparsity penalty:
\begin{displaymath}
   \min_{\D \in \C,\alphab \in \Real^{k \times n}}
   \frac{1}{2}||\X-\D\alphab||_F^2 + \lambda ||\alphab||_{1,1},
\end{displaymath}
where, as before, $\X = [\x_1,\ldots,\x_n]$ is the matrix of data vectors,
and $||\alphab||_{1,1}$ denotes the $\ell_1$~norm of the matrix $\alphab$---that is, the sum of the magnitude of its coefficients.
A natural approach to solving this problem is to alternate between the
two variables, minimizing over one while keeping the other one fixed,
as proposed by \citet{lee} (see also \citealt{engan} and
\citealt{aharon}, who use $\ell_0$ rather than $\ell_1$
penalties, or \citealt{zou2} for the problem of sparse principal component analysis).\footnote{In our setting, as in \citet{lee}, we have
preferred to use the convex $\ell_1$ norm, that has empirically proven
to be better behaved in general than the $\ell_0$ pseudo-norm for
dictionary learning.} Since 
the computation of the coefficients vectors $\alphab_i$ 
dominates the cost of each iteration in this block-coordinate descent approach, a second-order optimization technique can be
used to accurately estimate $\D$ at each step when $\alphab$ is fixed.

As pointed out by \citet{bottou}, however, one is usually not interested in the minimization of the {\em empirical cost} $f_n(\D)$ with high precision, but instead in the
minimization of the {\em expected cost}
\begin{displaymath}
   f(\D) \defin \E_{\x}[\l(\x,\D)] = \lim_{n \to \infty} f_n(\D) \as,
\end{displaymath}
where the expectation (which is supposed finite) is taken relative to
the (unknown) probability distribution $p(\x)$ of the
data.\footnote{We use ``a.s.'' to denote almost sure convergence.}
In particular, given a finite training set,
one should not spend too much effort on accurately minimizing the
empirical cost, since it is only an approximation of the expected cost.
An ``inaccurate'' solution may indeed have the same or better
expected cost than a ``well-optimized'' one.
\citet{bottou} further show that stochastic gradient algorithms, whose
rate of convergence is very poor in conventional optimization terms,
may in fact in certain settings be shown both theoretically and
empirically to be faster in reaching a solution with low
expected cost than second-order batch methods.  With large training sets, the
risk of overfitting is lower, but classical optimization techniques may
become impractical in
terms of speed or memory requirements.

In the case of dictionary learning, the classical projected first-order
projected stochastic gradient descent algorithm (as used
by \citealt{aharon2,kavukcuoglu} for instance) consists of a sequence of
updates of~$\D$:
\begin{displaymath}
   \D_{t} = \Pi_\C\Big[\D_{t-1} -
   \delta_t\nabla_{\D}\l(\x_{t},\D_{t-1})\Big], 
\end{displaymath}
where $\D_t$ is the estimate of the optimal dictionary at iteration $t$,
$\delta_t$ is the gradient step, $\Pi_\C$~is the orthogonal projector onto
$\C$, and the vectors~$\x_t$ are i.i.d. samples of the (unknown) distribution~$p(\x)$. Even though it is often difficult to obtain such i.i.d. samples,
the vectors~$\x_t$ are in practice obtained by cycling on a randomly
permuted training set.
As shown in Section \ref{sec:exp_dict}, we have observed that this
method can be competitive in terms of speed compared to batch methods when
the training set is large and when $\delta_t$ is carefully chosen.
In particular, good results are obtained using a learning rate of the form
$\delta_t \defin a / (t+b)$, where $a$ and $b$ have to be well chosen in a
data set-dependent way.
Note that first-order stochastic gradient descent has also been used for other
matrix factorization problems (see \citealp{koren} and references therein).

The optimization method we present in the next section falls into the
class of online algorithms based on stochastic approximations, processing one
sample at a time (or a mini-batch), but further exploits the specific
structure of the problem to efficiently solve it by sequentially minimizing a
quadratic local surrogate of the expected cost.
As shown in Section \ref{subsec:link}, it uses second-order information of the cost 
function, allowing the optimization without any explicit learning rate
tuning.
\section{Online Dictionary Learning} \label{sec:online}
We present in this section the basic components of our online
algorithm for dictionary learning
(Sections~\ref{sec:outline}--\ref{subsec:dicoupdate}), as well as a
few minor variants which speed up our implementation in practice
(Section~\ref{sec:variants}) and an interpretation in terms of
a Kalman algorithm (Section~\ref{subsec:link}).
\subsection{Algorithm Outline\label{sec:outline}}
Our procedure is summarized in Algorithm
\ref{fig:algoonline}. Assuming that the training set is composed of i.i.d.~samples 
of a distribution $p(\x)$, its inner loop draws one element $\x_t$ 
at a time, as in stochastic gradient descent, and alternates
classical sparse coding steps for computing the decomposition~$\alphab_t$
of~$\x_t$ over the dictionary~$\D_{t-1}$ obtained at the previous
iteration, with dictionary update steps where the new dictionary $\D_t$
is computed by minimizing over $\C$ the function
\begin{equation}
   \hatf_t(\D) \defin \frac{1}{t}\sum_{i=1}^t\Big(
   \frac{1}{2}||\x_i-\D\alphab_i||_2^2 + \lambda ||\alphab_i||_1\Big),
   \label{eq:surrog}
\end{equation}
and the vectors $\alphab_i$ for $i < t$ have been computed during the previous steps of
the algorithm.
The motivation behind this approach is twofold:
\begin{itemize}
   \item  The function $\hatf_t$, which is quadratic in $\D$, aggregates the
      past information with a few sufficient statistics obtained during the
previous steps of the algorithm, namely the vectors $\alphab_i$, and it is easy
to show that it upperbounds the empirical cost $f_t(\D_t)$ from
Eq.~(\ref{eq:empirical}). One key aspect of our convergence analysis will be to
show that $\hatf_t(\D_t)$ and $f_t(\D_t)$ converge almost surely to the same
limit, and thus that $\hatf_t$ acts as a {\em surrogate} for $f_t$. 
   \item Since $\hatf_t$ is close to $\hatf_{t-1}$ for large values of $t$, so are $\D_t$ and $\D_{t-1}$,
      under suitable assumptions,
      which makes it efficient to use $\D_{t-1}$ as warm restart for computing $\D_t$.
\end{itemize}
\begin{algorithm}[t]
   \caption{Online dictionary learning.}
   \label{fig:algoonline}
   \begin{algorithmic}[1]
      \REQUIRE $\x \in \Real^m \sim p(\x)$ (random variable and an algorithm
      to draw i.i.d samples of~$p$), $\lambda \in \Real$ (regularization
      parameter), $\D_0 \in \Real^{m \times k}$ (initial dictionary), $T$
      (number of iterations).
      \STATE $\A_0 \in \Real^{k \times k} \leftarrow 0$, 
      $\B_0 \in \Real^{m \times k} \leftarrow 0$ (reset the ``past''
      information). 
      \FOR {$t= 1$ to $T$}
      \STATE Draw $\x_t$ from $p(\x)$. 
      \STATE Sparse coding: compute using LARS 
      \begin{displaymath} 
         \alphab_t  \defin  \argmin_{\alphab \in
         \Real^k} \frac{1}{2}||\x_t-\D_{t-1}\alphab||_2^2 +
         \lambda||\alphab||_1. 
      \end{displaymath}
      \STATE $\A_t \leftarrow \A_{t-1} + \alphab_t \alphab_t^T$. 
      \STATE $\B_t \leftarrow \B_{t-1} + \x_t \alphab_t^T$.
      \STATE Compute $\D_t$ using Algorithm
      \ref{fig:algoupdate}, with $\D_{t-1}$ as warm restart, so that
      \begin{eqnarray}
         \D_t & \defin &\argmin_{\D \in \C} \frac{1}{t}\sum_{i=1}^t
         \Big(\frac{1}{2}||\x_i-\D\alphab_i||_2^2 + \lambda||\alphab_i||_1\Big),\nonumber \\
         & = & \argmin_{\D \in \C} \frac{1}{t}\Big( \frac{1}{2}\trace(\D^T\D\A_t) - \trace(\D^T\B_t)\Big).
	  \label{eq:dicoupdate}
      \end{eqnarray}
   \ENDFOR
   \RETURN $\D_T$ (learned dictionary).
\end{algorithmic}
\end{algorithm}
\begin{algorithm}[hbtp]
   \caption{Dictionary Update.}
   \label{fig:algoupdate}
   \begin{algorithmic}[1]
      \REQUIRE $\D = [\d_1,\ldots,\d_k] \in \Real^{m \times k}$ (input dictionary), \\
      $\A = [\a_1,\ldots,\a_k] \in \Real^{k \times k}$ \\
      $\B =[\b_1,\ldots,\b_k] \in \Real^{m \times k}$
      \REPEAT
      \FOR {$j = 1$ to $k$}
      \STATE Update the $j$-th column to optimize for (\ref{eq:dicoupdate}): 
      \begin{equation} 
         \begin{split}
            \u_j &\leftarrow \frac{1}{\A[j,j]}(\b_j-\D\a_j) + \d_j, \\
            \d_j &\leftarrow \frac{1}{\max(||\u_j||_2,1)}\u_j. \label{eq:updated}
         \end{split}
      \end{equation}
   \ENDFOR
   \UNTIL {\bf convergence}
   \RETURN $\D$ (updated dictionary).
\end{algorithmic}
\end{algorithm}
\subsection{Sparse Coding} \label{sec:sparsecoding}
The sparse coding problem of Eq.~(\ref{eq:ell1penalty}) with fixed dictionary
is an $\ell_1$-regularized linear least-squares problem.  A number of recent
methods for solving this type of problems are based on coordinate descent with
soft thresholding \citep{fu,friedman,wu}.  When the columns of the dictionary
have low correlation, we have observed that these simple methods are very efficient.
However, the columns of learned dictionaries are in general highly correlated,
and we have empirically observed that these 
algorithms become much slower in this setting.  This has led us to
use instead the LARS-Lasso algorithm, a homotopy method
\citep{osborne2,efron} that provides the whole regularization
path---that is, the solutions for all possible values of
$\lambda$. With an efficient Cholesky-based implementation
(see \citealp{efron,zou}) for brief descriptions of such implementations), it has
proven experimentally at least as fast as approaches based on soft
thresholding, while providing the solution with a higher accuracy and
being more robust as well since it does not require an arbitrary stopping criterion.
\subsection{Dictionary Update\label{subsec:dicoupdate}}
Our algorithm for updating the dictionary uses block-coordinate descent with
warm restarts (see \citealp{bertsekas}). One of its main advantages is that it is
parameter free and does not require any learning rate tuning. 
Moreover, the procedure does not require to store all the vectors $\x_i$ and $\alphab_i$,
but only the matrices $\A_t=\sum_{i=1}^t \alphab_i \alphab_i^T$ in $\Real^{k \times k}$ and
$\B_t=\sum_{i=1}^t \x_i\alphab_i^T$ in $\Real^{m \times k}$.
 Concretely,
Algorithm~\ref{fig:algoupdate} sequentially updates each column of $\D$.
A simple calculation shows 
that solving (\ref{eq:dicoupdate}) with respect to the $j$-th column $\d_j$, while keeping the other ones fixed under the constraint $\d_j^T\d_j \leq
1$, amounts to an orthogonal projection of the vector $\u_j$ defined in Eq.~(\ref{eq:updated}), onto the
constraint set, namely the $\ell_2$-ball here, which is solved by
Eq.~(\ref{eq:updated}).  Since the convex optimization
problem (\ref{eq:dicoupdate}) admits separable
constraints in the updated blocks (columns), convergence to a global
optimum is guaranteed~\mbox{\citep{bertsekas}}. In practice, the vectors
$\alphab_i$ are sparse and the coefficients of the matrix $\A_t$ are 
often concentrated on the diagonal, which makes the
block-coordinate descent more efficient.\footnote{We have observed that this is true when the columns of $\D$ are not too correlated. When a group of columns in $\D$ are highly correlated, the coefficients of the matrix $\A_t$ concentrate instead on the corresponding principal submatrices of~$\A_t$.}
After a few iterations of our algorithm, using the value of $\D_{t-1}$ as a warm restart for computing~$\D_t$ 
becomes effective, and a single iteration of Algorithm \ref{fig:algoupdate} has
empirically found to be sufficient to achieve
convergence of the dictionary update step. Other approaches have been proposed to update $\D$: For
instance,
\citet{lee} suggest using a Newton method on the dual of
Eq.~(\ref{eq:dicoupdate}), but this requires inverting a $k
\times k$ matrix at each Newton iteration, which is impractical for
an online algorithm.
\subsection{Optimizing the Algorithm\label{sec:variants}}
We have presented so far the basic building blocks of our algorithm.
This section discusses a few simple improvements that significantly
enhance its performance.
\subsubsection{Handling Fixed-Size Data Sets}
In practice, although it may be very large, the size of the training
set often has a predefined finite size (of course this may not be the case when the data
must be treated on the fly like a video stream for
example). In this situation, the same data points may be examined
several times, and it is very common in online algorithms to simulate an
i.i.d.~sampling of $p(\x)$ by cycling over a randomly permuted
training set (see \citealp{bottou} and references therein). This method works experimentally well in our
setting but, when the training set is small enough, it is
possible to further speed up convergence: In Algorithm~\ref{fig:algoonline},
the matrices~$\A_t$ and~$\B_t$ carry all the
information from the past coefficients
$\alphab_1,\ldots,\alphab_{t}$. Suppose that at time $t_0$, a
signal $\x$ is drawn and the vector~$\alphab_{t_0}$ is computed. If the
same signal $\x$ is drawn again at time $t > t_0$, then it is natural
to replace the ``old'' information~$\alphab_{t_0}$ by the new
vector $\alphab_t$ in the matrices $\A_t$ and $\B_t$---that is, $\A_{t} \leftarrow \A_{t-1} +
\alphab_{t}\alphab_{t}^T - \alphab_{t_0}\alphab_{t_0}^T$ and $\B_{t} \leftarrow \B_{t-1} +
\x_{t}\alphab_{t}^T - \x_t\alphab_{t_0}^T$.
In this setting, which requires storing all the past coefficients
$\alphab_{t_0}$, this method amounts to a block-coordinate descent for the
problem of minimizing Eq.~(\ref{eq:joint}).
When dealing with large but finite sized training sets, storing
all coefficients~$\alphab_{i}$ is impractical, but it is
still possible to partially exploit the same idea, by removing
the information from $\A_t$ and $\B_t$ that is older than two
{\em epochs} (cycles through the data), through the use of two auxiliary
matrices~$\A_t'$ and~$\B_t'$ of size $k \times k$ and $m \times k$
respectively.  These two matrices should be built with the same rules as~$\A_t$
and~$\B_t$, except that at the end of an epoch, $\A_t$ and $\B_t$ are
respectively replaced by~$\A_t'$ and $\B_t'$, while~$\A_t'$ and $\B_t'$ are set
to $0$. Thanks to this strategy, $\A_t$ and $\B_t$ do not carry any coefficients~$\alphab_i$ older than two epochs.
\subsubsection{Scaling the ``Past'' Data}
At each iteration, the ``new'' information
$\alphab_t$ that is added to the matrices~$\A_t$ and~$\B_t$ has the same weight as the ``old'' one. A simple and
natural modification to the algorithm is to rescale the ``old'' information
so that newer coefficients~$\alphab_t$ have more weight, which 
is classical in online learning. For instance,
\citet{neal} present an online algorithm for EM, where sufficient statistics
are aggregated over time, and an exponential decay is used to forget
out-of-date statistics.
In this paper, we propose to
replace lines $5$ and $6$ of Algorithm \ref{fig:algoonline} by
\begin{displaymath}
\begin{split}
   \A_t & \leftarrow \beta_t \A_{t-1} + \alphab_{t}\alphab_{t}^T, \\
\B_t & \leftarrow \beta_t \B_{t-1} + \x_t\alphab_{t}^T, \\
\end{split}
\end{displaymath}
where $\beta_t \defin \big(1-\frac{1}{t}\big)^\rho$,
and $\rho$ is a new parameter.
In practice, one can apply this strategy after a few iterations, once $\A_t$
is well-conditioned. Tuning $\rho$
improves the convergence
rate, when the training sets are large, even though, as shown in
Section \ref{sec:exp_dict}, it is not critical.  To understand better the effect
of this modification, note that Eq.~(\ref{eq:dicoupdate}) becomes
\begin{displaymath}
   \begin{split}
   \D_t  &\defin \argmin_{\D \in \C} \frac{1}{\sum_{j=1}^t ( j/t)^\rho}\sum_{i=1}^t
   \Big(\frac{i}{t}\Big)^\rho\Big(\frac{1}{2}||\x_i-\D\alphab_i||_2^2 + \lambda||\alphab_i||_1\Big), \\
   &= \argmin_{\D \in \C} \frac{1}{\sum_{j=1}^t ( j/t)^\rho  }\Big( \frac{1}{2}\trace(\D^T\D\A_t) - \trace(\D^T\B_t)\Big).
\end{split}
\end{displaymath}
When $\rho=0$, we obtain the original version of the algorithm. Of course,
different strategies and heuristics could also be investigated.
In practice, this parameter $\rho$ is useful for large data sets only ($n \geq 100\,000$).
For smaller data sets, we have not observed a better performance when using 
this extension.
\subsubsection{Mini-Batch Extension}
In practice, we can also improve the convergence speed of our algorithm by
drawing $\eta > 1$ signals at each iteration instead of a single
one, which is a classical heuristic in stochastic gradient descent
algorithms. In our case, this is further motivated by the fact that
the complexity of computing $\eta$ vectors $\alphab_i$ is not
linear in $\eta$. A Cholesky-based implementation of LARS-Lasso for decomposing
a single signal has a complexity of $\O(kms+ks^2)$, where $s$ is
the number of nonzero coefficients. When decomposing~$\eta$ signals,
it is possible to pre-compute the Gram matrix $\D_t^T\D_t$
and the total complexity becomes $\O(k^2m + \eta(km+ks^2))$, which 
is much cheaper than $\eta$ times the previous complexity when $\eta$ is
large enough and $s$ is small.  Let us denote by $\x_{t,1},\ldots,\x_{t,\eta}$ the signals drawn at
iteration~$t$. We can now replace lines $5$ and $6$ of
Algorithm
\ref{fig:algoonline} by
\begin{displaymath}
\begin{split}
   \A_t & \leftarrow \A_{t-1} + \frac{1}{\eta}\sum_{i=1}^\eta \alphab_{t,i}\alphab_{t,i}^T, \\
\B_t & \leftarrow \B_{t-1} + \frac{1}{\eta}\sum_{i=1}^\eta \x_{t,i}\alphab_{t,i}^T. \\
\end{split}
\end{displaymath}
\subsubsection{Slowing Down the First Iterations}
As in the case of stochastic gradient descent, the first iterations of our algorithm may
update the parameters with large steps, immediately leading to large
deviations from the initial dictionary.  To prevent this phenomenon,
classical implementations of stochastic gradient descent use gradient steps
of the form $a / (t+b)$, where~$b$ ``reduces'' the step size.  An
initialization of the form $\A_0 = t_0 \I$ and $\B_0 = t_0 \D_0$ with $t_0 \geq
0$ also slows down the first steps of our algorithm by forcing the solution of the
dictionary update to stay close to $\D_0$. As shown in Section
\ref{sec:exp_dict}, we have observed that our method does not require this extension
to achieve good results in general.
\subsubsection{Purging the Dictionary from Unused Atoms}
Every dictionary learning technique sometimes encounters situations where
some of the dictionary atoms are never (or very seldom) used, which typically
happens with a very bad initialization.  A common practice is to replace
these during the optimization by randomly chosen elements of the training
set, which solves in practice the problem in most cases. For more difficult
and highly regularized cases, it is also possible to choose 
a continuation strategy consisting of starting
from an easier, less regularized problem, and gradually increasing $\lambda$.
This continuation method has not been used in this paper.
\subsection{Link with Second-order Stochastic Gradient Descent} \label{subsec:link}
For unconstrained learning problems with twice differentiable expected cost, 
the second-order stochastic gradient descent algorithm (see
\citealp{bottou} and references therein) improves upon its first-order version,
by replacing the learning rate by the inverse of the Hessian. When this matrix 
can be computed or approximated efficiently, this method usually yields a faster
convergence speed and removes the problem of tuning the learning
rate. However, it cannot be applied easily to constrained optimization problems
and requires at every iteration an inverse of the Hessian.
For these two reasons, it cannot be used for the dictionary learning problem,
but nevertheless it shares some similarities with our algorithm, which we
illustrate with the example of a different problem.

Suppose that two major modifications are brought to our original formulation: (i) the
vectors~$\alphab_t$ are independent of the dictionary $\D$---that is, they are
drawn at the same time as $\x_t$;  (ii) the optimization is unconstrained---that is, $\C =
\Real^{m \times k}$.  This setting leads to the least-square estimation
problem
\begin{equation}
   \min_{\D \in \Real^{m \times k}} \E_{(\x,\alphab)}\big[||\x-\D\alphab||_2^2\big],\label{eq:kalman}
\end{equation}
which is of course different from the original dictionary learning formulation.
Nonetheless, it is possible to address Eq.~(\ref{eq:kalman}) with our method and show that
it amounts to using the recursive formula 
\begin{displaymath}
   \D_t \leftarrow \D_{t-1} + (\x_t-\D_{t-1}\alphab_t)\alphab_t^T \Big(
   \sum_{i=1}^{t} \alphab_i \alphab_i^T \Big)^{-1},
\end{displaymath}
which is equivalent to a second-order stochastic gradient descent
algorithm: The gradient obtained at
$(\x_t,\alphab_t)$ is the term \mbox{$-(\x_t-\D_{t-1}\alphab_t)\alphab_t^T$}, and the sequence \mbox{$(1/t)\sum_{i=1}^{t}\alphab_i\alphab_i^T$} converges to the Hessian of the objective function.
Such sequence of updates admit a fast implementation called Kalman
algorithm (see \citealp{kushner} and references therein).
\section{Convergence Analysis} \label{sec:conv}
The main tools used in our proofs are
the convergence of empirical processes \citep{vaart} and,
following \citet{bottou2}, the convergence of quasi-martingales \citep{fisk}.
Our analysis is limited to the basic version of the algorithm, although it
can in principle be carried over to the optimized versions discussed in
Section~\ref{sec:variants}. 
Before proving our main result,
let us first discuss the (reasonable) assumptions under
which our analysis holds.
\subsection{Assumptions} \label{subsec:assumptions}
{\bf (A) The data admits a distribution with compact support $K$}.
Assuming a compact support for the data is natural in audio, image, and video
processing applications, where it is imposed by the data acquisition process.

\noindent {\bf (B) The quadratic surrogate functions $\hatf_t$ are strictly convex with
lower-bounded Hessians.}
We assume that the smallest eigenvalue of the
positive semi-definite matrix~$\frac{1}{t}\A_t$ defined in Algorithm~\ref{fig:algoonline} is greater than or equal to some constant $\kappa_1$. As a consequence,
$\A_t$ is invertible and $\hatf_t$ is strictly convex with Hessian $\I \otimes \frac{2}{t}\A_t
$. This  
hypothesis is in practice verified experimentally after a few iterations
of the algorithm when the initial dictionary is reasonable, consisting for example of a
few elements from the training set, or any common dictionary,
such as DCT (bases of cosines products) or wavelets \citep{mallat}. Note that it is
easy to enforce this assumption by adding a term $\frac{\kappa_1}{2}||\D||_F^2$ to the
objective function, which is equivalent to replacing the
positive semi-definite matrix $\frac{1}{t}\A_t$ by
$\frac{1}{t}\A_t+\kappa_1\I$. We have omitted for simplicity this
penalization in our analysis.\\ 
{\bf (C) A particular sufficient condition for the uniqueness of the sparse coding solution is
satisfied.} 
Before presenting this assumption, let us briefly recall classical optimality conditions for the $\ell_1$ decomposition problem in Eq.~(\ref{eq:ell1penalty}) \citep{fuchs}. For $\x$ in~$K$ and $\D$ in $\C$, $\alphab$ in $\Real^k$ is 
a solution of Eq.~(\ref{eq:ell1penalty}) if and only if
\begin{equation}
   \begin{split}
      \d_j^T(\x-\D\alphab) &=\lambda \sign(\alphab[j]) ~~\text{if}~~ \alphab[j] \neq 0, \\
      |\d_j^T(\x-\D\alphab)| &\leq\lambda ~~\text{otherwise}.
   \end{split} \label{eq:optimell1}
\end{equation}
Let $\alphab^\star$ be such a solution. Denoting by $\Lambda$ the set of
indices $j$ such that $|\d_j^T(\x-\D\alphab^\star)|=\lambda$, and $\D_\Lambda$
the matrix composed of the columns from $\D$ restricted to the set $\Lambda$, it is easy
to see from Eq.~(\ref{eq:optimell1}) that the solution $\alphab^\star$ is necessary unique
if $(\D_\Lambda^T\D_\Lambda)$ is invertible and that
\begin{equation}
   \alphab^\star_\Lambda =
   (\D_\Lambda^T\D_\Lambda)^{-1}(\D_\Lambda^T\x-\lambda \epsilonb_\Lambda), \label{eq:closed}
\end{equation}
where $\alphab^\star_\Lambda$ is the vector containing the values of $\alphab^\star$ corresponding to the set 
$\Lambda$ and $\epsilonb_\Lambda[j]$ is equal to the sign of $\alphab^\star_\Lambda[j]$ for all $j$. 
With this preliminary uniqueness condition in hand, we can 
now formulate our assumption:
\emph{We assume that there exists $\kappa_2 > 0$ such that, for
all~$\x$ in~$K$ and all dictionaries~$\D$ in the subset of $\C$ considered
by our algorithm, the smallest eigenvalue of $\D_{\Lambda}^T\D_\Lambda$ is
greater than or equal to $\kappa_2$.} 
This guarantees the invertibility of~$(\D_\Lambda^T\D_\Lambda)$ and
therefore the uniqueness of the solution of Eq.~(\ref{eq:ell1penalty}).
It is of course easy to build a
dictionary $\D$ for which this assumption fails.  However, having
$\D_{\Lambda}^T\D_\Lambda$ invertible is a common assumption in
linear regression and in methods such as the LARS algorithm aimed at
solving Eq.~(\ref{eq:ell1penalty}) \citep{efron}. 
It is also possible to enforce this condition using an elastic net
penalization \citep{zou}, replacing $||\alphab||_1$ by
$||\alphab||_1+\frac{\kappa_2}{2}||\alphab||_2^2$ and thus improving the numerical
stability of homotopy algorithms, which is the choice
made by \citet{zou2}. Again, we
have omitted this penalization in our analysis.
\subsection{Main Results}
Given assumptions {\bf (A)}--{\bf (C)}, let us now show that our algorithm
converges to a stationary point of the objective function.  Since this paper
is dealing with non-convex optimization, neither our algorithm nor any one in
the literature is guaranteed to find the global optimum of the optimization
problem. However, such stationary points have often been found to
be empirically good enough for practical applications, for example, for image
restoration \citep{elad,mairal}.

Our first result (Proposition \ref{prop:regul} below) states that given {\bf
(A)}--{\bf (C)}, 
$f(\D_t)$ converges almost surely and 
$f(\D_t)-\hatf_t(\D_t)$ converges almost surely to $0$, meaning that
$\hatf_t$ acts as a converging surrogate of~$f$.
First, we prove a lemma to show that $\D_t-\D_{t-1}=O(1/t)$. It 
does not ensure the convergence of~$\D_t$, but guarantees 
the convergence of the positive sum $\sum_{t=1}^\infty ||\D_t-\D_{t-1}||_F^2$,
a classical condition in gradient descent convergence proofs
\citep{bertsekas}.  
\begin{lemma}{\bf [Asymptotic variations of $\D_t$].\\} \label{lemma:prelim}
   Assume {\bf (A)}--{\bf (C)}. Then,
   \begin{displaymath}
      \D_{t+1}-\D_{t} = \O\Big(\frac{1}{t}\Big) \as
   \end{displaymath}
\end{lemma}
\begin{proof}
   This proof is inspired by Prop 4.32 of \citet{bonnans2} on the Lipschitz
   regularity of solutions of optimization problems.
   Using assumption {\bf (B)}, for all $t$, the surrogate $\hatf_t$ is 
   strictly convex with a Hessian lower-bounded by~$\kappa_1$. Then, a short
   calculation shows that it verifies the {\em second-order growth condition}
   \begin{equation}
      \hatf_t(\D_{t+1})-\hatf_t(\D_t) \geq \kappa_1 ||\D_{t+1}-\D_t||_F^2. \label{eq:lipsh1}
   \end{equation}
   Moreover,
   \begin{displaymath}
      \begin{split}
      \hatf_t(\D_{t+1})-\hatf_t(\D_t) &=
      \hatf_t(\D_{t+1})-\hatf_{t+1}(\D_{t+1}) + \hatf_{t+1}(\D_{t+1}) -
      \hatf_{t+1}(\D_{t}) + \hatf_{t+1}(\D_{t}) - \hatf_t(\D_t) \\
      & \leq  \hatf_t(\D_{t+1})-\hatf_{t+1}(\D_{t+1}) + \hatf_{t+1}(\D_{t}) - \hatf_t(\D_t),
   \end{split}
   \end{displaymath}
   where we have used that $\hatf_{t+1}(\D_{t+1})-\hatf_{t+1}(\D_{t}) \leq 0$ because $\D_{t+1}$ minimizes
   $\hatf_{t+1}$ on $\C$.
   Since $\hatf_t(\D) = \frac{1}{t}(\frac{1}{2}\trace(\D^T\D\A_t)-\trace(\D^T\B_t))$, and $||\D||_F \leq \sqrt{k}$, 
   it is possible to show that
   $\hatf_t - \hatf_{t+1}$ is Lipschitz with constant $c_t=(1/t)(||\B_{t+1}-\B_t||_F+\sqrt{k}||\A_{t+1}-\A_t||_F)$,
   which gives
   \begin{equation}
      \hatf_t(\D_{t+1})-\hatf_t(\D_t)  \leq c_t||\D_{t+1}-\D_t||_F. \label{eq:lipsh2}
   \end{equation}
   From Eq.~(\ref{eq:lipsh1}) and (\ref{eq:lipsh2}), we obtain
   \begin{displaymath}
      ||\D_{t+1}-\D_t||_F \leq \frac{c_t}{\kappa_1}.
   \end{displaymath}
   Assumptions {\bf (A)}, {\bf (C)} and Eq.~(\ref{eq:closed}) ensure that the
   vectors $\alphab_i$ and $\x_i$ are bounded with probability one and therefore
   $c_t = \O(1/t) \as$
\end{proof}

We can now state and prove our first proposition, which shows that we are indeed minimizing a smooth function.
\begin{proposition}{\bf [Regularity of $f$]. \\} \label{prop:regul}
   Assume {\bf (A)} to {\bf (C)}. For $\x$ in the support $K$ of the probability distribution $p$, and $\D$ in the feasible set~$\C$, let us define 
   \begin{equation}
      \alphab^\star(\x,\D)=\argmin_{\alphab \in \Real^k}
      \frac{1}{2}||\x-\D\alphab||_2^2 + \lambda ||\alphab||_1. \label{eq:optim}
   \end{equation}
   Then, 
   \begin{enumerate}
      \item the function $\l$ defined in Eq.~(\ref{eq:ell1penalty}) is
         continuously differentiable and 
         \begin{displaymath}
            \nabla_{\D}\l(\x,\D) = -(\x-\D\alphab^\star(\x,\D))\alphab^{\star}(\x,\D)^T.
      \end{displaymath}
      \item $f$ is continuously differentiable and $\nabla f(\D) =
         \E_\x\big[\nabla_{\D}\l(\x,\D)\big]$;
      \item $\nabla f(\D)$ is Lipschitz on $\C$.
   \end{enumerate}
\end{proposition}
\begin{proof}
Assumption {\bf (A)} ensures that the vectors $\alphab^\star$ are
bounded for $\x$ in~$K$ and $\D$ in~$\C$. Therefore, one can
restrict the optimization problem~(\ref{eq:optim}) to a
compact subset of $\Real^k$.  Under assumption {\bf (C)}, the solution of
Eq.~(\ref{eq:optim}) is unique and $\alphab^\star$ is
well-defined.  Theorem \ref{theo:bonnans} in Appendix \ref{appendix:th} from
\citet{bonnans} can be applied and gives us
directly the first statement.  Since $K$ is compact, and $\l$ is continuously differentiable, the
second statement follows immediately.

To prove the third claim, we will show that for all~$\x$ in~$K$,
$\alphab^\star(\x,.)$ is Lipschitz with a constant independent of
$\x$,\footnote{ From now on, for a vector $\x$ in $\Real^m$, 
$\alphab^\star(\x,.)$ denotes the function that associates to a matrix
$\D$ verifying Assumption~{\bf (C)}, the optimal solution $\alphab^\star(\x,\D)$.
For simplicity, we will use these slightly abusive
notation in the rest of the paper.}
which is a sufficient condition for $\nabla f$ to be
Lipschitz.  First, the function optimized in Eq.~(\ref{eq:optim}) is continuous
in $\alphab$, $\D$, $\x$ and has a unique minimum, implying that
$\alphab^\star$ is continuous in $\x$ and $\D$.

Consider a matrix $\D$ in $\C$ and $\x$ in $K$ and denote by $\alphab^\star$ the vector $\alphab^\star(\x,\D)$, and
again by $\Lambda$ the set of indices $j$
such that $|\d_j^T(\x-\D\alphab^\star)|=\lambda$.
Since $\d_j^T(\x-\D\alphab^\star)$ is 
continuous in $\D$ and $\x$, there exists an open neighborhood $V$ around $(\x,\D)$ such that
for all $(\x',\D')$ in $V$, and $j \notin \Lambda$, $|{\d_j^T}'(\x'-\D'{\alphab^\star}')| < \lambda$
and ${\alphab^\star}'[j]=0$, where
${\alphab^\star}'=\alphab^\star(\x',\D')$.

Denoting by $\U_{\Lambda}$ the matrix composed of the columns of a matrix $\U$ 
corresponding to the index set $\Lambda$ and similarly by $\u_\Lambda$ the vector composed of the values of a vector $\u$ corresponding to $\Lambda$, 
we consider the function $\tildel$ 
\begin{displaymath}
   \tildel(\x,\D_\Lambda,\alphab_\Lambda) \defin \frac{1}{2}||\x-\D_\Lambda\alphab_\Lambda||_2^2+\lambda||\alphab_\Lambda||_1,
\end{displaymath}
Assumption {\bf (C)} tells us that $\tildel(\x,\D_\Lambda,.)$ is strictly convex with
a Hessian lower-bounded by~$\kappa_2$. Let us consider $(\x',\D')$ in $V$. A simple calculation shows that 
\begin{displaymath}
   \tildel(\x,\D_\Lambda,{\alphab^\star_\Lambda}')-\tildel(\x,\D_\Lambda,\alphab^\star_\Lambda)
   \geq \kappa_2 ||{\alphab^\star_\Lambda}'-\alphab^\star_\Lambda||_2^2.
\end{displaymath}
Moreover, it is easy to show that $\tildel(\x,\D_\Lambda,.)-\tildel(\x',\D_\Lambda',.)$ is Lipschitz with constant  
 $e_1||\D_\Lambda-\D_\Lambda'||_F+e_2||\x-\x'||_2$, where $e_1,e_2$ are constants independent of
 $\D,\D',\x,\x'$ and then, one can show that 
\begin{displaymath}
   ||{\alphab^\star}'-\alphab^\star||_2 = ||{\alphab^\star_\Lambda}'-\alphab^\star_\Lambda||_2 \leq \frac{1}{\kappa_2}\big(e_1||\D-\D'||_F + e_2||\x-\x'||_2\big).
\end{displaymath}
Therefore, $\alphab^\star$ is locally Lipschitz. Since $K \times \C$ is compact, $\alphab^\star$ is uniformly Lipschitz on $K \times \C$, which concludes the proof. 
\end{proof}

Now that we have shown that $f$ is a smooth function, we can state our first result showing that the sequence of functions $\hatf_t$ acts asymptotically as a surrogate of $f$ and that $f(\D_t)$ converges almost surely in the following proposition.
\begin{proposition}{\bf [Convergence of $f(\D_t)$ and of the
surrogate function].} \label{prop:conv1}
Let $\hatf_t$ denote the surrogate function defined in Eq.~(\ref{eq:surrog}).
Assume {\bf (A)} to {\bf (C)}. Then,
\begin{enumerate}
   \item $\hatf_t(\D_t)$  converges almost surely;
   \item $f(\D_t)-\hatf_t(\D_t)$ converges almost surely to $0$; 
   \item $f(\D_t)$ converges almost surely.
\end{enumerate}
\end{proposition}
\begin{proof}
Part of this proof is inspired by \citet{bottou2}.
We prove the convergence of the sequence~$\hatf_t(\D_t)$ by
showing that the stochastic positive process
\begin{displaymath}
   u_t \defin \hatf_t(\D_t) \geq 0,
\end{displaymath}
is a quasi-martingale and use Theorem \ref{theo:martingales} from \citet{fisk} (see Appendix \ref{appendix:th}),
which states that if the sum of the ``positive'' variations of $u_t$ are
bounded, $u_t$ is a quasi-martingale, which converges with
probability one (see Theorem \ref{theo:martingales} for details).  Computing the variations of $u_t$, we obtain
\begin{equation}
   \begin{split}
      u_{t+1}-u_t &= \hatf_{t+1}(\D_{t+1}) - \hatf_t(\D_t) \\
      &= \hatf_{t+1}(\D_{t+1}) - \hatf_{t+1}(\D_{t})  + \hatf_{t+1}(\D_t)-
      \hatf_t(\D_t) \\
      &= \hatf_{t+1}(\D_{t+1}) - \hatf_{t+1}(\D_{t})  +
      \frac{\l(\x_{t+1},\D_t)-f_t(\D_t)}{t+1} + \frac{f_t(\D_t) -
      \hatf_t(\D_t)}{t+1}, \label{eq:ut}
   \end{split}
\end{equation}
using the fact that $\hatf_{t+1}(\D_t) = \frac{1}{t+1}\l(\x_{t+1},\D_t) + \frac{t}{t+1}\hatf_t(\D_t)$.
Since $\D_{t+1}$ minimizes $\hatf_{t+1}$ on~$\C$ and $\D_t$ is in $\C$,
$\hatf_{t+1}(\D_{t+1}) - \hatf_{t+1}(\D_{t}) \leq 0$.
Since the surrogate $\hatf_t$ upperbounds the empirical cost $f_t$,
we also have $f_t(\D_t) -\hatf_t(\D_t) \leq 0$.
To use Theorem \ref{theo:martingales}, we consider the filtration of
the past information~$\FF_t$ and take the expectation
of Eq.~(\ref{eq:ut}) conditioned on $\FF_t$, obtaining the following bound
\begin{displaymath}
   \begin{split}
      \E[u_{t+1}-u_t | \FF_t] & \leq
      \frac{\E[\l(\x_{t+1},\D_t)|\FF_t] -f_t(\D_t)}{t+1} \\ & \leq
      \frac{f(\D_t) -f_t(\D_t)}{t+1} \\
      &\leq \frac{||f-f_t||_\infty}{t+1},  
   \end{split}
\end{displaymath}
For a specific matrix $\D$, the central-limit theorem states that 
$\E[\sqrt{t}(f(\D_t) -f_t(\D_t))]$ is bounded.
However, we need here a stronger result on empirical processes to show that
$\E[\sqrt{t}||f-f_t||_\infty]$ is bounded. To do so,
we use the Lemma \ref{lemma:donsker} in Appendix \ref{appendix:th},
which is a corollary of Donsker theorem \citep[see][chap. 19.2]{vaart}.
It is easy to show that in our case, all the hypotheses are verified,
namely, $\l(\x,.)$ is uniformly Lipschitz and bounded since it is continuously differentiable on a compact set,
the set $\C \subset \Real^{m \times k}$ is bounded, and
$\E_\x[\l(\x,\D)^2]$ exists and is uniformly bounded.
Therefore, Lemma \ref{lemma:donsker} applies and there exists a constant $\kappa>0$
such that
\begin{displaymath}
   \E[\E[u_{t+1}-u_t | {\mathcal F}_t]^+] \leq \frac{\kappa}{t^{\frac{3}{2}}}.
\end{displaymath}
Therefore, defining $\delta_t$ as in Theorem \ref{theo:martingales},
we have 
\begin{displaymath}
   \sum_{t=1}^\infty \E[\delta_t(u_{t+1}-u_t)] = \sum_{t=1}^\infty
   \E[\E[u_{t+1}-u_t | {\mathcal F}_t]^+] < +\infty.
\end{displaymath}
Thus, we can apply Theorem \ref{theo:martingales}, which proves that
$u_t$ converges almost surely and that 
\begin{displaymath}
   \sum_{t=1}^\infty |\E[u_{t+1}-u_t | {\mathcal F}_t]| < +\infty \as
\end{displaymath}
Using Eq.~(\ref{eq:ut}) we can show that it implies the
almost sure convergence of the positive sum
\begin{displaymath}
   \sum_{t=1}^\infty \frac{\hatf_t(\D_t)-f_t(\D_t)}{t+1}.
\end{displaymath}
Using Lemma \ref{lemma:prelim} and the fact that the functions $f_t$ and
$\hatf_t$ are
bounded and Lipschitz, with a constant independent of
$t$, it is easy to show that the hypotheses of Lemma
\ref{lemma:converg} in Appendix \ref{appendix:th} are satisfied. Therefore
\begin{displaymath}
   f_t(\D_t)-\hatf_t(\D_t) \underset{t \to + \infty}{\longrightarrow} 0 \as
\end{displaymath}
Since $\hatf_t(\D_t)$ converges almost surely, this shows that
$f_t(\D_t)$ converges in probability to the same limit. Note that
we have in addition $||f_t-f||_\infty \to_{t\to+\infty} 0 \as$
\citep[see][Theorem 19.4 (Glivenko-Cantelli)]{vaart}.
Therefore, 
\begin{displaymath}
   f(\D_t)-\hatf_t(\D_t) \underset{t \to + \infty}{\longrightarrow} 0 \as
\end{displaymath}
and $f(\D_t)$ converges almost surely, which proves the second and third points.
\end{proof}

With Proposition \ref{prop:conv1} in hand, we can now prove our final and
strongest result, namely that first-order necessary optimality conditions are
verified asymptotically with probability one. 
\begin{proposition}{\bf [Convergence to a stationary point].}
   \label{prop:conv2} 
   Under assumptions {\bf (A)} to {\bf (C)}, the distance between $\D_t$ and
   the set of stationary points of the dictionary learning problem converges
   almost surely to $0$ when $t$ tends to infinity.
\end{proposition}
\begin{proof}
   Since the sequences of matrices $\A_t,\B_t$ are in a compact set, it is
   possible to extract converging subsequences. Let us assume for a moment
   that these sequences converge respectively to two matrices $\A_\infty$ and
   $\B_\infty$. In that case, $\D_t$ converges to a matrix $\D_\infty$ in~$\C$. 
   Let $\U$ be a matrix in $\Real^{m \times k}$. Since
   $\hatf_t$ upperbounds $f_t$ on $\Real^{m \times k}$, for all $t$,
   \begin{displaymath}
      \hatf_t(\D_t+\U) \geq f_t(\D_t+\U).
   \end{displaymath}
   Taking the limit when $t$ tends to infinity, 
   \begin{displaymath}
      \hatf_\infty(\D_\infty+\U) \geq f(\D_\infty+\U).
   \end{displaymath}
   Let $h_t > 0$ be a sequence that converges to $0$.
   Using a first order Taylor expansion, and using the fact that $\nabla f$ is
   Lipschitz and $\hatf_\infty(\D_\infty) = f(\D_\infty) \as$, we have
   \begin{displaymath}
      f(\D_\infty)+\trace(h_t\U^T\nabla \hatf_\infty(\D_\infty)) + o(h_t\U) \geq f(\D_\infty)+\trace(h_t\U^T\nabla f(\D_\infty)) + o(h_t\U),
   \end{displaymath}
   and it follows that
   \begin{displaymath}
      \trace\Big(\frac{1}{||\U||_F}\U^T\nabla \hatf_\infty(\D_\infty)\Big) \geq \trace\Big(\frac{1}{||\U_t||_F}\U^T\nabla f(\D_\infty)\Big),
   \end{displaymath}
   Since this inequality is true for all $\U$, $\nabla\hatf_\infty(\D_\infty)=\nabla f(\D_\infty)$. 
   A first-order necessary optimality condition for $\D_\infty$ being an optimum of $\hatf_\infty$ is that
   $-\nabla \hatf_\infty$ is in the {\em normal cone} of the set $\C$ at $\D_\infty$ \citep{borwein}.
   Therefore, this first-order necessary conditions is verified for $f$ at $\D_\infty$ as well.
   Since $\A_t,\B_t$ are  asymptotically close to their accumulation points, $-\nabla f(\D_t)$ is asymptotically close
   the normal cone at $\D_t$ and these first-order optimality conditions are verified asymptotically with probability one.
\end{proof}
\section{Extensions to Matrix Factorization} \label{sec:variations}
In this section, we present variations of the basic online algorithm to
address different optimization problems. We first present different
possible regularization terms for~$\alphab$ and~$\D$, which can be used
with our algorithm, and then detail some specific cases such as
non-negative matrix factorization, sparse principal component
analysis, constrained sparse coding,
and simultaneous sparse coding.
\subsection{Using Different Regularizers for $\alphab$}
In various applications, different priors for
the coefficients $\alphab$ may lead to different regularizers $\psi(\alphab)$.
As long as the assumptions of Section \ref{subsec:assumptions} are verified,
our algorithm can be used with:
\begin{itemize}
   \item Positivity constraints on $\alphab$ that are added to the $\ell_1$-regularization.
      The homotopy method presented in \citet{efron} is able to handle such constraints.
   \item The Tikhonov regularization, $\psi(\alphab) = \frac{\lambda_1}{2}||\alphab||_2^2$,
      which does not lead to sparse solutions.
   \item The elastic net \citep{zou}, $\psi(\alphab)= \lambda_1
      ||\alphab||_1 + \frac{\lambda_2}{2} ||\alphab||_2^2$, leading to a formulation relatively close to \citet{zou2}.
   \item The group Lasso \citep{yuan,turlach,bach}, $\psi(\alphab) = \sum_{i=1}^s
      ||\alphab_i||_2$, where $\alphab_i$ is a vector corresponding to a 
      group of variables.
\end{itemize}
Non-convex regularizers such as the $\ell_0$ pseudo-norm, $\ell_p$ pseudo-norm
with $p < 1$
can be used as well. However, as with any classical dictionary
learning techniques exploiting non-convex regularizers
(e.g., \citealp{field,engan,aharon}), there is no theoretical 
convergence results in these cases.
Note also that convex smooth approximation of sparse regularizers \citep{bradley},
or structured sparsity-inducing regularizers \citep{jenatton,jacob} 
could be used as well even though we have not tested them.
\subsection{Using Different Constraint Sets for $\D$}
In the previous subsection, we have claimed that our algorithm could be used
with different regularization terms on $\alphab$. 
For the dictionary learning problem, we have considered an $\ell_2$-regularization
on $\D$ by forcing its columns to have less than unit $\ell_2$-norm. 
We have shown that with this constraint set, the dictionary update step 
can be solved efficiently using a block-coordinate descent approach. Updating
the $j$-th column of~$\D$, when keeping the other ones fixed is solved by 
orthogonally projecting the vector $\u_j = \d_j+(1/\A[j,j])(\b_j-\D\a_j)$
on the constraint set $\C$, which in the classical dictionary learning case
amounts to a projection of $\u_j$ on the $\ell_2$-ball. 

It is easy to show that this procedure can be extended to different convex
constraint sets $\C'$ as long as the constraints are a union of independent constraints
on each column of~$\D$ and the orthogonal projections of the vectors $\u_j$ onto the
set $\C'$ can be done efficiently.
Examples of different sets $\C'$ that we propose as an alternative to $\C$ are
\begin{itemize}
   \item The ``non-negative'' constraints: 
      \begin{displaymath}
         \C' = \{ \D \in \Real^{m \times k} \st \forall j=1,\ldots,k,~~ ||\d_j||_2 \leq 1 ~~\text{and}~~ \d_j \geq 0 \}.
      \end{displaymath}
   \item The ``elastic-net'' constraints:
      \begin{displaymath}
         \C' \defin \{ \D \in \Real^{m \times k} \st \forall
         j=1,\ldots,k,~~ ||\d_j||_2^2 + \gamma||\d_j||_1 \leq 1
         \}.
      \end{displaymath}
      These constraints induce sparsity in the dictionary $\D$ (in addition to the
      sparsity-inducing regularizer on the vectors $\alphab_i$).
      By analogy with the regularization proposed by \citet{zou}, we call these
      constraints ``elastic-net constraints.'' Here, $\gamma$ is a new parameter,
      controlling the sparsity of the dictionary $\D$. Adding
      a non-negativity constraint is also possible in this case. 
      Note that the presence of the $\ell_2$ regularization is important here.
      It has been shown
      by \citet{bach3} that using the $\ell_1$-norm only in such problems
      lead to trivial solutions when $k$ is large enough.
      The combination of $\ell_1$ and $\ell_2$ constraints has also been proposed
      recently for the problem of matrix factorization by \citet{witten}, but
      in a slightly different setting. 
   \item The ``fused lasso'' \citep{tibshirani2} constraints. When one is
      looking for a dictionary whose columns are sparse and piecewise-constant, a fused lasso regularization can be used.
      For a vector $\u$ in $\Real^m$, we consider the $\ell_1$-norm of 
      the consecutive differences of $\u$ denoted by
      \begin{displaymath}
         \FL(\u) \defin \sum_{i=2}^m |\u[i]-\u[i-1]|,
      \end{displaymath}
      and define the ``fused lasso'' constraint set
      \begin{displaymath}
         \C' \defin \{ \D \in \Real^{m \times k} \st \forall
         j=1,\ldots,k,~~ ||\d_j||_2^2 + \gamma_1||\d_j||_1 + \gamma_2 \FL(\d_j) \leq 1
         \}.
      \end{displaymath}
      This kind of regularization has proven to be useful for 
      exploiting genomic data such as CGH arrays \citep{tibshirani3}.
\end{itemize}

In all these settings, replacing the projections of the vectors $\u_j$ onto the $\ell_2$-ball by
the projections onto the new constraints, our algorithm is still guaranteed to converge and find a
stationary point of the optimization problem.  The orthogonal projection onto the
``non negative'' ball is simple (additional thresholding) but the projection onto the two
other sets is slightly more involved. In
Appendix~\ref{appendix:proj}, we propose two algorithms for efficiently solving these
problems. The first one is presented in Section \ref{appendix:sec:elas} and computes the 
projection of a vector onto the elastic-net constraint in linear time, by
extending the efficient projection onto the $\ell_1$-ball from \citet{maculan} and \citet{duchi}.
The second one is a homotopy method, which solves the projection on the fused
lasso constraint set in $O(ks)$, where $s$ is the number of
piecewise-constant parts in the solution.  This method also solves
efficiently the fused lasso signal approximation problem
presented in \citet{friedman}:
\begin{displaymath}
   \min_{\u \in \Real^n} \frac{1}{2}||\b-\u||_2^2 + \gamma_1||\u||_1 +
   \gamma_2 \FL(\u) + \gamma_3||\u||_2^2.
\end{displaymath}
Being able to solve this problem efficiently has also numerous applications,
which are beyond the scope of this paper. For instance, it allows us to use
the fast algorithm of \citet{nesterov} for solving the more general fused
lasso problem \citep{tibshirani2}.
Note that the proposed method could be used as well with more complex constraints
for the columns of~$\D$, which we have not tested in this paper, 
addressing for instance the problem of structured sparse PCA \citep{jenatton2}.

Now that we have presented a few possible regularizers for $\alphab$ and $\D$,
that can be used within our algorithm, we focus on a few classical problems
which can be formulated as dictionary learning problems with 
specific combinations of such regularizers.
\subsection{Non Negative Matrix Factorization}
Given a matrix $\X = [\x_1,\ldots,\x_n]$ in $\Real^{m \times n}$, 
\citet{lee2} have proposed the non 
negative matrix factorization problem (NMF), which consists of minimizing 
the following cost
\begin{displaymath}
 \min_{\D \in \C, \alphab \in \Real^{k \times n}}
   \sum_{i=1}^n \frac{1}{2}||\x_i-\D\alphab_i||_2^2 \st \D \geq 0,~\forall~i,~~ \alphab_i \geq 0. 
\end{displaymath}
With this formulation, the matrix $\D$ and the vectors $\alphab_i$ are forced
to have non negative components, which leads to sparse solutions.
When applied to images, such as faces,
\citet{lee2} have shown that the learned features are more localized than the
ones learned with a classical singular value decomposition. 
As for dictionary learning, classical approaches for addressing this problem
are batch algorithms, such as the multiplicative update rules of
\citet{lee2}, or the projected gradient descent algorithm of \citet{lin}.

Following this
line of research, \citet{hoyer,hoyer2} has proposed non negative sparse coding
(NNSC), which extends non-negative matrix factorization by adding a
sparsity-inducing penalty to the objective function to further control the sparsity of the vectors~$\alphab_i$:
\begin{displaymath}
   \min_{\D \in \C, \alphab \in \Real^{k \times n}} \sum_{i=1}^n \Big(\frac{1}{2}||\x_i-\D\alphab_i||_2^2 + \lambda \sum_{j=1}^k \alphab_i[j]\Big) \st \D \geq 0,~\forall~i,~~ \alphab_i \geq 0.
\end{displaymath}
When $\lambda = 0$, this formulation is equivalent to NMF.
The only difference with the dictionary learning problem is that
non-negativity constraints are imposed on $\D$ and the vectors $\alphab_i$.
A simple modification of our algorithm, presented above, allows us to handle
these constraints, while guaranteeing to
find a stationary point of the optimization problem. Moreover,
our approach can work in the setting when $n$ is large.
\subsection{Sparse Principal Component Analysis}
Principal component analysis (PCA) is a classical tool for data analysis,
which can be interpreted as a method for finding orthogonal directions maximizing
the variance of the data, or as a low-rank matrix approximation method.
\citet{jolliffe}, \citet{zou2}, \citet{aspremont}, \citet{bach2}, \citet{witten} and \citet{zass} have proposed different formulations
for sparse principal component analysis (SPCA), which extends PCA by
estimating sparse vectors maximizing the variance of the data,
some of these formulations enforcing orthogonality between the sparse
components, whereas some do not.  In this paper, we formulate SPCA as a
sparse matrix factorization which is equivalent to the dictionary learning
problem with eventually sparsity constraints on the dictionary---that is, we
use the $\ell_1$-regularization term for $\alphab$ and the ``elastic-net''
constraint for $\D$ (as used in a penalty term by \citealt{zou2}):
\begin{displaymath}
   \min_{\alphab \in \Real^{k \times n}} \sum_{i=1}^n
   \Big(\frac{1}{2}||\x_i-\D\alphab_i||_2^2 + \lambda||\alphab_i||_1\Big) \st \forall
   j=1,\ldots,k,~~ ||\d_j||_2^2 + \gamma||\d_j||_1 \leq 1.
\end{displaymath}
As detailed above, our dictionary update procedure amounts to successive
orthogonal projection of the vectors $\u_j$ on the constraint set. More
precisely, the update of $\d_j$ becomes 
\begin{displaymath}
\begin{split}
   \u_j &\leftarrow \frac{1}{\A[j,j]}(\b_j-\D\a_j) + \d_j, \\
   \d_j &\leftarrow \argmin_{\d \in \Real^m} ||\u_j-\d||_2^2 \st  ||\d||_2^2 + \gamma||\d||_1 \leq 1,
\end{split}
\end{displaymath}
which can be solved in linear time using Algorithm \ref{fig:proj_elastic}
presented in Appendix \ref{appendix:proj}.  In addition to that, our SPCA
method can be used with fused Lasso constraints as well.

\subsection{Constrained Sparse Coding}
Constrained sparse coding problems are often encountered in the literature, and lead to 
different loss functions such as
\begin{equation}
   \l'(\x,\D) = \min_{\alphab \in \Real^k} ||\x-\D\alphab||_2^2 \st
   ||\alphab||_1 \leq T, \label{eq:mina}
\end{equation}
or
\begin{equation}
   \l''(\x,\D) = \min_{\alphab \in \Real^k}  ||\alphab||_1  \st
   ||\x-\D\alphab||_2^2 \leq \varepsilon, \label{eq:minb}
\end{equation}
where $T$ and $\varepsilon$ are pre-defined thresholds. 
Even though these loss functions lead to equivalent optimization problems 
in the sense that for given $\x,\D$ and $\lambda$, there exist $\varepsilon$ and
$T$ such that $\l(\x,\D)$, $\l'(\x,\D)$ and $\l''(\x,\D)$ admit
the same solution $\alphab^\star$, the problems of learning $\D$ using $\l$,
$\l'$ of $\l''$ are not equivalent. For instance, using $\l''$ has proven
experimentally to be particularly well adapted to image denoising \citep{elad,mairal}. 

For all $T$, the same analysis as for $\l$ can be carried for $\l'$, and the
simple modification which consists of computing $\alphab_t$ using
Eq.~(\ref{eq:mina}) in the sparse coding step leads to the minimization of the
expected cost $\min_{\D \in \C} \E_{\x}[\l'(\x,\D)]$.

Handling the case $\l''$ is a bit different. We propose to use the same
strategy as for $\l'$---that is, using our algorithm but computing $\alphab_t$
solving Eq.~(\ref{eq:minb}). Even though our analysis does not apply
since we do not have a quadratic surrogate of the expected cost,
experimental evidence shows that this approach is efficient in practice.
\subsection{Simultaneous Sparse Coding}
In some situations, the signals $\x_i$ are not i.i.d samples of an unknown
probability distribution, but are structured in groups (which are however
independent from each other), and one may want to address the problem of
simultaneous sparse coding, which appears also in the literature under various names such
as group sparsity or grouped variable
selection \citep{cotter2,turlach,yuan,obozinski,obozinski2,zhang,tropp2,tropp3}.  Let $\X =
[\x_1,\ldots,\x_q] \in \Real^{m \times q}$ be a set of signals.
Suppose one wants to obtain sparse decompositions of
the signals on the dictionary $\D$ that share the same active set (non-zero coefficients).
Let \mbox{$\alphab = [\alphab_1,\ldots,\alphab_q]$} in $\Real^{k \times q}$ be the matrix
composed of the coefficients. One way of imposing this \emph{joint sparsity}
is to penalize the number of non-zero rows of $\alphab$. A classical convex
relaxation of this joint sparsity measure is to consider the $\ell_{1,2}$-norm
on the matrix~$\alphab$
\begin{displaymath}
||\alphab||_{1,2} \defin \sum_{j=1}^k ||\alphab^j||_2,
\end{displaymath}
where $\alphab^j$ is the $j$-th row of $\alphab$. In that setting, the
$\ell_{1,2}$-norm of $\alphab$ is the $\ell_1$-norm of the $\ell_2$-norm of the rows
of $\alphab$.

The problem of jointly decomposing the signals $\x_i$ can be written as a
$\ell_{1,2}$-sparse decomposition problem, which is a subcase of the group
Lasso \citep{turlach,yuan,bach}, by defining the cost function 
\begin{displaymath}
  \l'''(\X,\D) = \min_{\alphab \in \Real^{k \times q}} \frac{1}{2}||\X-\D\alphab||_F^2 + \lambda ||\alphab||_{1,2},
\end{displaymath}
which can be computed using a block-coordinate descent approach \citep{friedman} or
an active set method \citep{roth2}.

Suppose now that we are able to draw groups of signals $\X_i$, $i=1,\ldots,n$ which have
bounded size and are independent from each other and identically distributed,
one can learn an adapted dictionary by solving the optimization problem 
\begin{displaymath}
   \min_{\D \in \C} \lim_{n \to \infty} \frac{1}{n} \sum_{i=1}^n \l'''(\X_i,\D).
\end{displaymath}
Being able to solve this optimization problem is important for many applications.
For instance, in \citet{mairal8}, state-of-the-art results in image denoising
and demosaicking are achieved with this formulation.
The extension of our algorithm to this case is relatively easy, computing at each
sparse coding step a matrix of coefficients $\alphab$, and keeping the updates
of $\A_t$ and $\B_t$ unchanged. 

All of the variants of this section have been implemented. Next section evaluates some of
them experimentally. An efficient C++ implementation with a Matlab interface 
of these variants is available on the Willow project-team web page \url{http://www.di.ens.fr/willow/SPAMS/}.
\section{Experimental Validation} \label{sec:exp_dict}
In this section, we present experiments on natural images and genomic data to demonstrate the
efficiency of our method for dictionary learning, non-negative matrix factorization, and
sparse principal component analysis.
\subsection{Performance Evaluation for Dictionary Learning} \label{subsec:exp:perf}
For our experiments, we have randomly selected $1.25\times 10^6$ patches from
images in the Pascal VOC'06 image database \citep{pascal2}, which is composed of
varied natural images; $10^6$ of these are kept
for training, and the rest for testing. 
We used these patches to create three
data sets $A$, $B$, and $C$ with increasing patch and dictionary sizes
representing various settings which are typical in image processing applications:
\begin{table}[hbtp] 
   \centering
\begin{tabular}{|l|c|c|l|}\hline
Data set& Signal size $m$ & Nb $k$ of atoms  & Type\\ \hline
$A$ & $8\times 8=64$ & 256 & b\&w\\ \hline
$B$ & $12\times 12\times 3=432$ & 512 & color\\ \hline
$C$ & $16\times 16=256$ & 1024 & b\&w \\ \hline
\end{tabular}
\end{table}
We have centered and normalized the patches to have unit $\ell_2$-norm and used the
regularization parameter $\lambda=1.2/\sqrt{m}$ in all of our experiments. The~$1/\sqrt{m}$
term is a classical normalization factor \citep{tsybakov}, and the
constant $1.2$ has shown to yield about 10 nonzero
coefficients for data set A and 40 for data sets B and C in these experiments.
We have implemented the proposed algorithm in C++ with a Matlab interface. All
the results presented in this section use the refinements from Section
\ref{sec:variants} since this has lead empirically to speed improvements.
Although our implementation is multithreaded, our experiments have been run for
simplicity on a single-CPU, single-core 2.66Ghz machine.

The first parameter to tune is $\eta$, the number of signals drawn at each
iteration.  Trying different powers of 2 for this variable has shown that
$\eta=512$ was a good choice (lowest objective function values on the training
set---empirically, this setting also yields the lowest values on the test
set). Even though this parameter is fairly easy to tune since values of 64,
128, 256 and 1024 have given very similar performances, the difference with
the choice $\eta=1$ is significant.

Our implementation can be used in both the online setting it is intended for,
and in a regular batch mode where it uses the entire data set at each iteration.
We have also
implemented a first-order stochastic gradient descent algorithm that shares
most of its code with our algorithm, except for the dictionary update step.
This setting allows us to draw meaningful comparisons between our algorithm
and its batch and stochastic gradient alternatives, which would have been
difficult otherwise. For example, comparing our algorithm to the Matlab
implementation of the batch approach from \citet{lee} developed by its authors
would have been unfair since our C++ program has a built-in speed
advantage.\footnote{Both LARS and the feature-sign algorithm \citep{lee} require
a large number of low-level operations which are not well optimized in Matlab.
We have indeed observed that our C++ implementation of LARS is up to $50$ times
faster than the Matlab implementation of the feature-sign algorithm of
\citet{lee} for our experiments.}
To measure and compare the performances of the three tested methods, we have
plotted the value of the objective function on {\em the test set},  acting as
a surrogate of the expected cost, as a function of the corresponding training
time.
\subsubsection{Online vs. Batch} 
The left column of Figure~\ref{fig:comparedict}~compares the online and batch
settings of our implementation. The full training set consists of $10^6$
samples.  The online version of our algorithm draws samples from the entire
set, and we have run its batch version on the full data set as well as subsets
of size $10^4$ and $10^5$ (see Figure~\ref{fig:comparedict}). The online setting systematically
outperforms its batch counterpart for every training set size and desired
precision.  We use a logarithmic scale for the computation time, which shows
that in many situations, the difference in performance can be dramatic. Similar
experiments have given similar results on smaller data sets.
Our algorithm uses all the speed-ups from Section \ref{sec:variants}.
The parameter~$\rho$ was chosen by trying the values $0,5,10,15,20,25$, and
$t_0$ by trying different powers of $10$. We have selected
$(t_0=0.001,\rho=15)$, which has given the best performance in terms of
objective function evaluated on the training set for the three data sets.
We have plotted three curves for our method: \textsf{OL1} corresponds
to the optimal setting $(t_0=0.001,\rho=15)$. 
Even though tuning two parameters might seem cumbersome, we have plotted two
other curves showing that, on the contrary, our method is very easy to use.
The curve \textsf{OL2}, corresponding to the setting
$(t_0=0.001,\rho=10)$, is very difficult to distinguish from the first curve
and we have observed a similar behavior with the setting $(t_0=0.001,\rho=20)$.
showing that our method is \emph{robust to the choice of the parameter $\rho$}.
We have also observed that the parameter $\rho$ is useful for large data sets only.
When using smaller ones ($n \leq 100, 000$), it did not bring any benefit.

Moreover, the curve \textsf{OL3} is obtained
without using a tuned parameter $t_0$---that is, $\rho=15$ and $t_0=0$, and shows
that its influence is very limited since very good results are obtained
without using it. On the other hand, we have observed that using a parameter $t_0$ too big, could
slightly slow down our algorithm during the first epoch (cycle on the training set).
\begin{figure}[hbtp]
 \centering
     \includegraphics[width=0.49\linewidth]{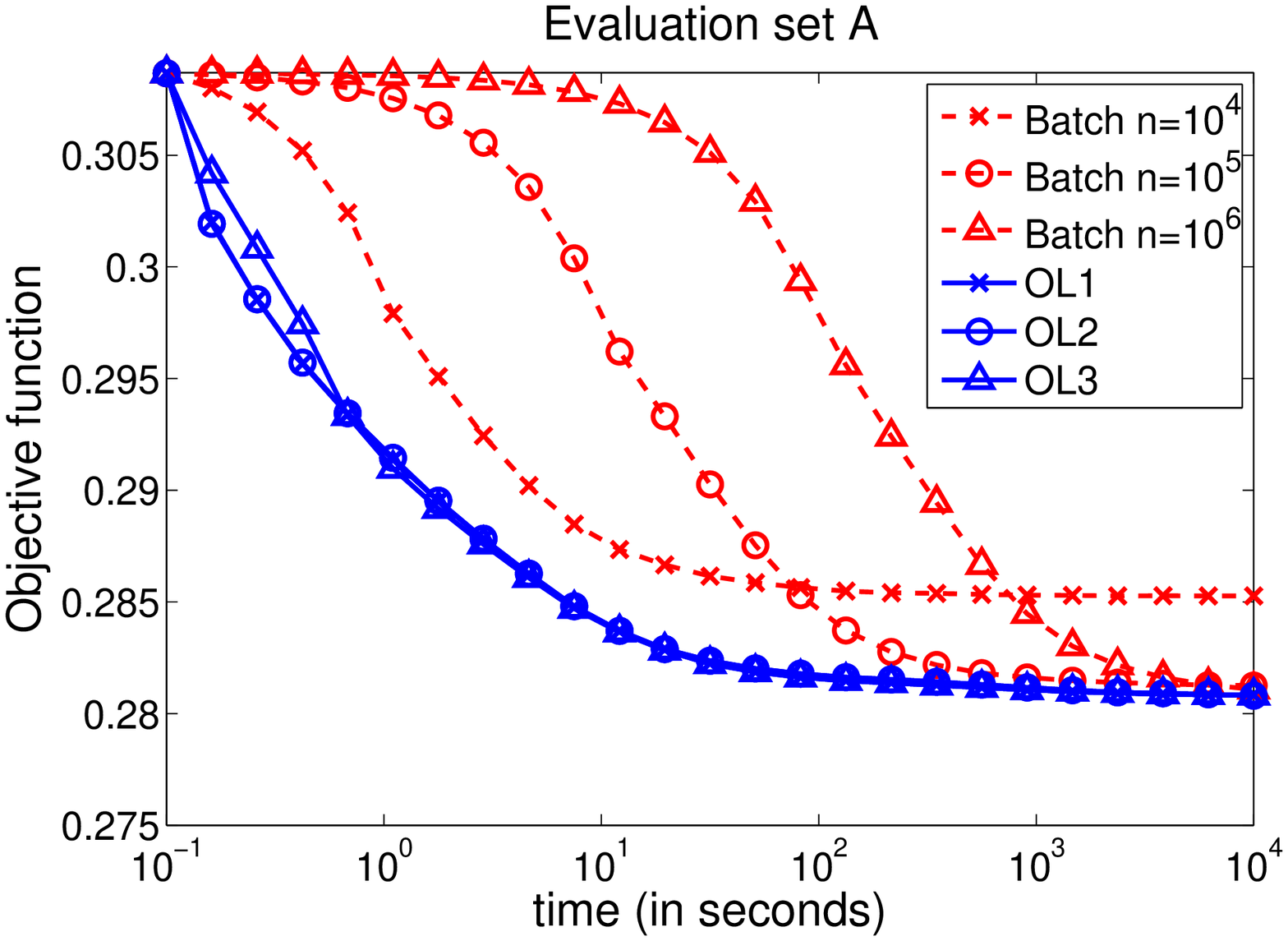} \hfill
     \includegraphics[width=0.49\linewidth]{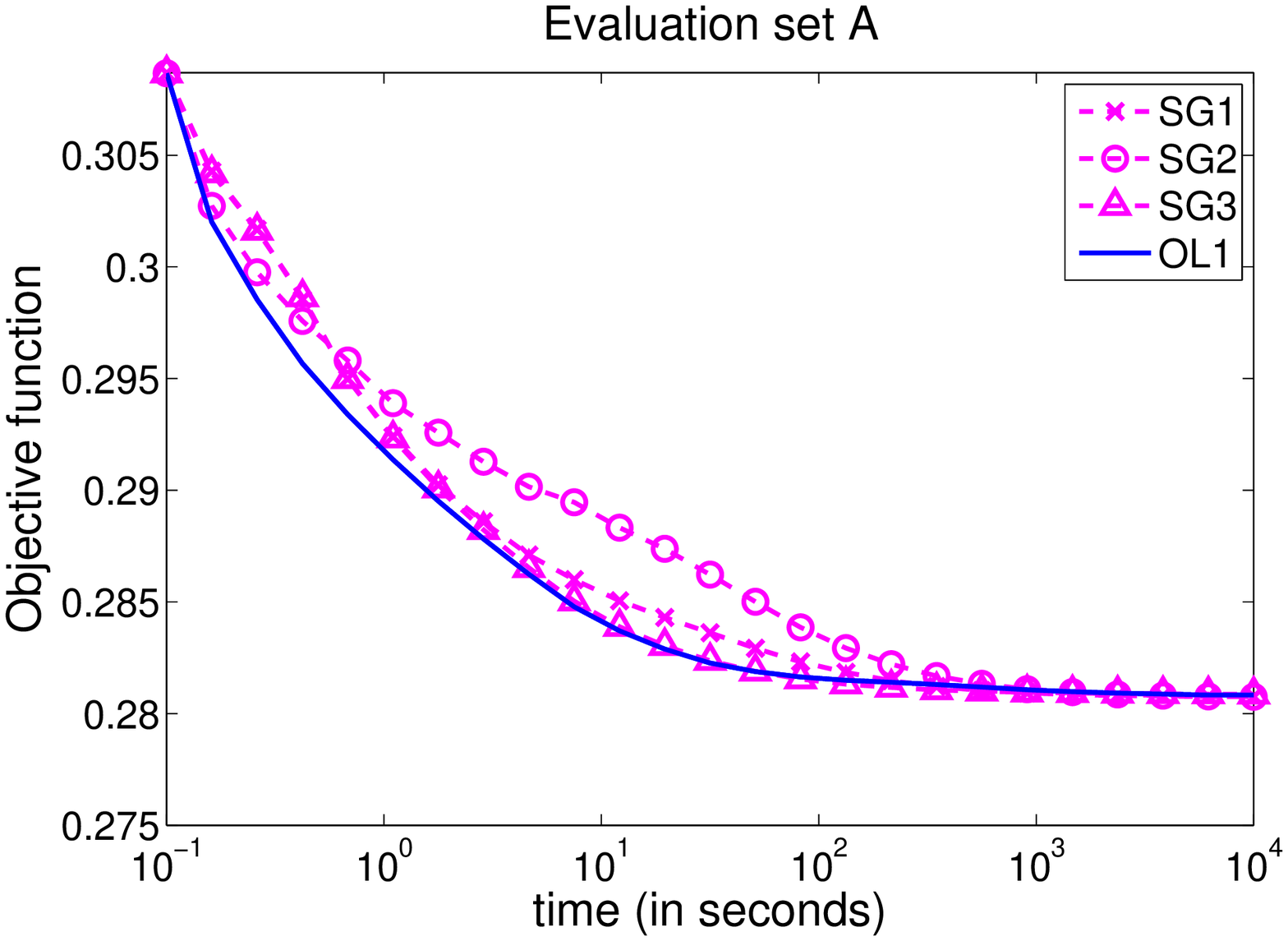}  \\
     \vspace*{0.3cm}
     \includegraphics[width=0.49\linewidth]{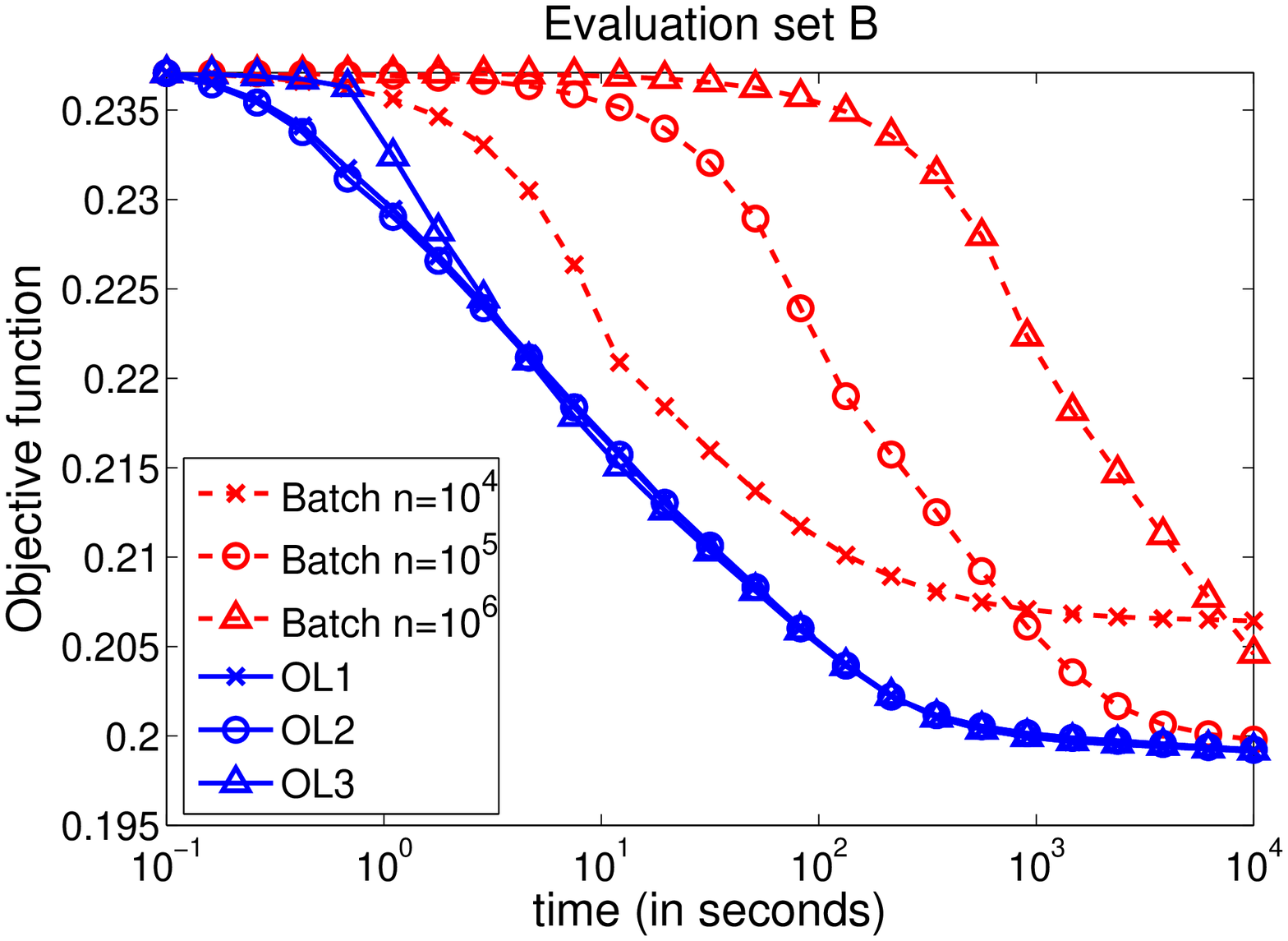} \hfill
     \includegraphics[width=0.49\linewidth]{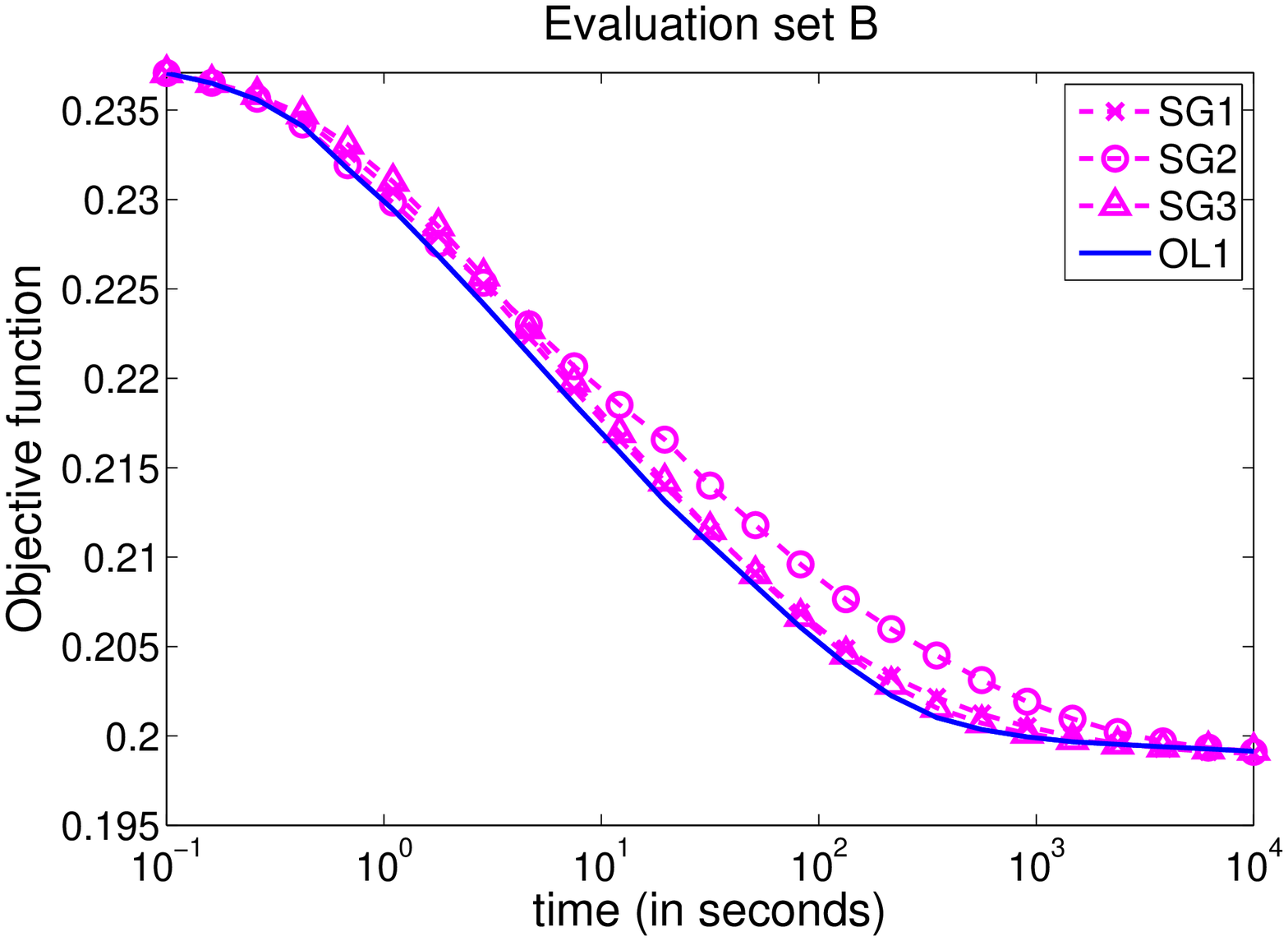} \\
     \vspace*{0.3cm}
     \includegraphics[width=0.49\linewidth]{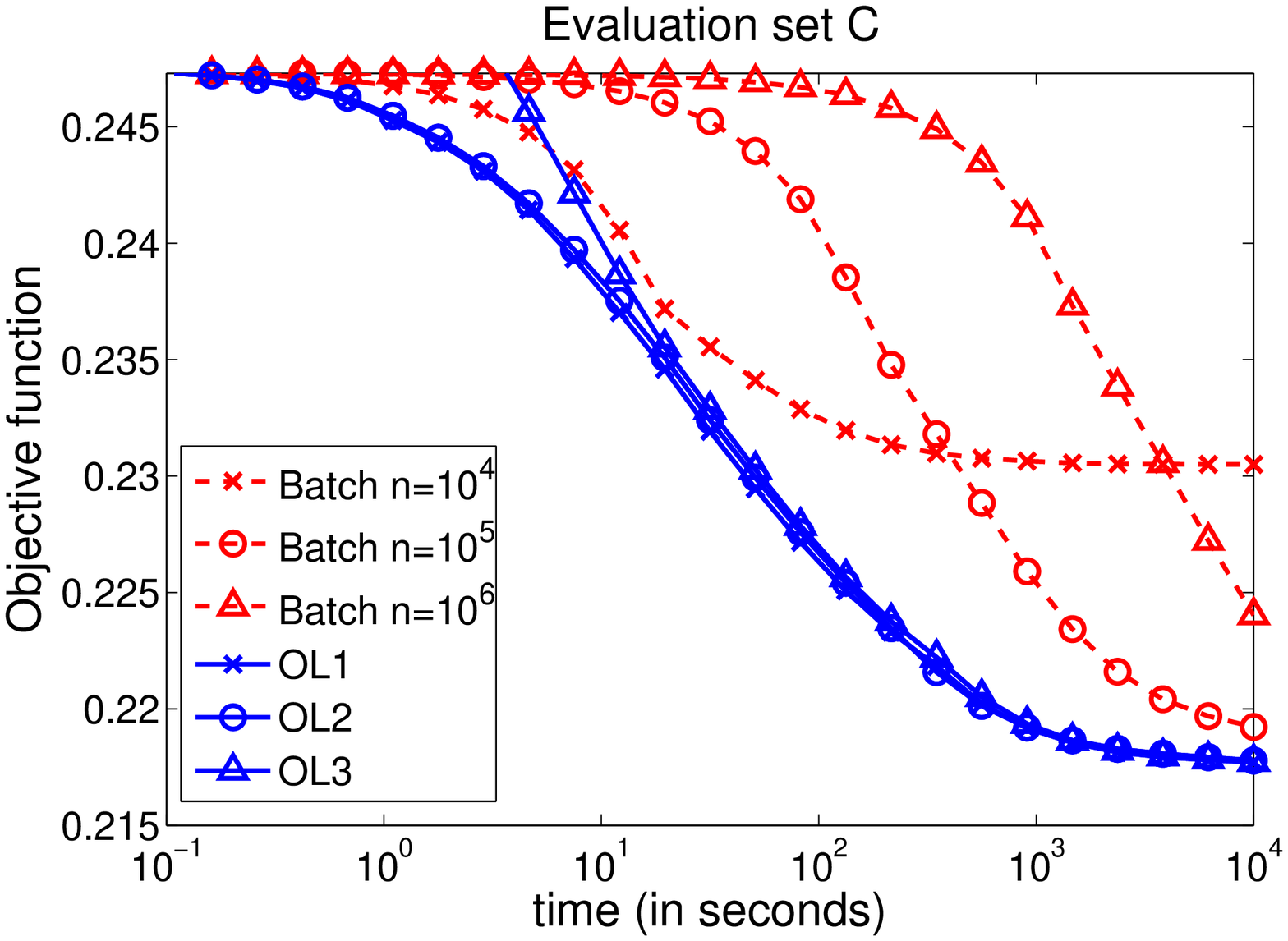} \hfill
     \includegraphics[width=0.49\linewidth]{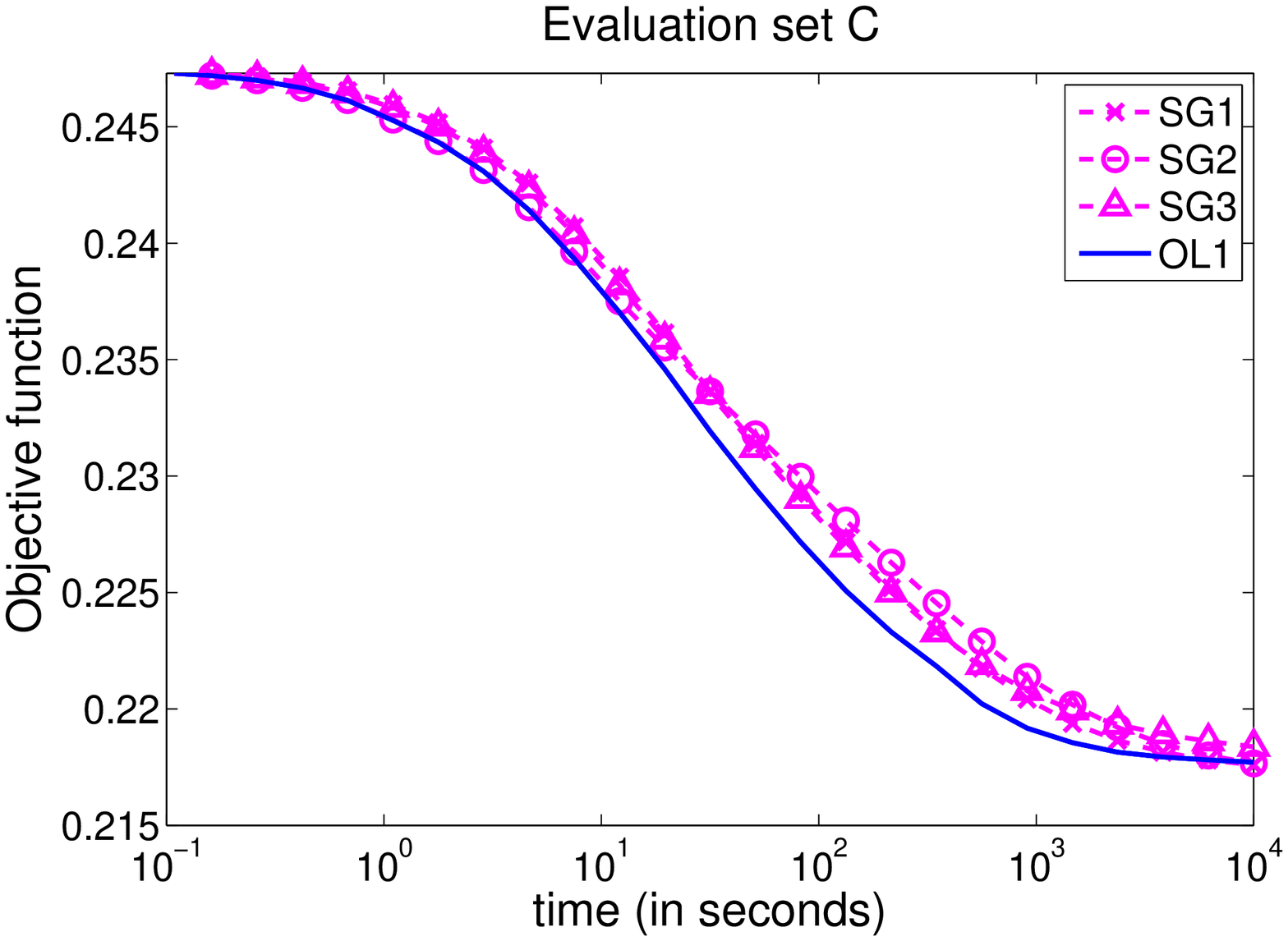}
     \caption{Left: Comparison between our method and the batch approach for dictionary learning. Right: Comparison between our method and
     stochastic gradient descent. The results are reported for three data sets as a function of computation time on a logarithmic
     scale. Note that the times of computation that are less than $0.1s$ are not reported. See text for details.
     }  \label{fig:comparedict}
 \end{figure}
\subsubsection{Comparison with Stochastic Gradient Descent} 
Our experiments have shown that obtaining good performance with stochastic
gradient descent requires using both the mini-batch heuristic {\em and}
carefully choosing a learning rate of the form $a/(\eta t+b)$.  To give the fairest comparison possible,
we have thus optimized these parameters. As for our algorithm, sampling $\eta$
values among powers of 2 (as before) has shown that $\eta=512$ was a good value
and gives a significant better performance than $\eta=1$.

In an earlier version of this work \citep{mairal7}, we have proposed
a strategy for our method which does not require any parameter tuning except
the mini-batch $\eta$ and compared it with the stochastic gradient descent
algorithm (SGD) with a learning rate of the form $a/(\eta t)$. 
While our method has improved in performance using the new parameter~$\rho$,
SGD has also proven to provide much better results when using
a learning rate of the form $a/(\eta t+b)$ instead of $a/(\eta t)$, at the cost
of an extra parameter $b$ to tune. Using the learning rate $a/(\eta t)$ with
a high value for $a$ results indeed in too large initial steps of the
algorithm increasing dramatically the value of the objective function, and a small value
of $a$ leads to bad asymptotic results, while a learning rate of the form
$a/(\eta t+b)$ is a good compromise.

We have tried different powers of $10$ for $a$ and $b$. First selected the
couple ($a=100, 000, b=100, 000$) and then refined it, trying the values $100,
000\times2^i$ for $i=-3,\ldots,3$.  Finally, we have selected $(a=200, 000, b=400,
000)$.  As shown on the right column of Figure \ref{fig:comparedict}, this
setting represented by the curve \textsf{SG1} leads to similar results as our
method.  The curve \textsf{SG2} corresponds to the parameters $(a=400, 000,
b=400, 000)$ and shows that increasing slightly the parameter $a$ makes the
curves worse than the others during the first iterations (see for instance the
curve between $1$ and $10^2$ seconds for data set A), but still lead to good
asymptotic results.  The curve \textsf{SG3} corresponds to a situation where
$a$ and $b$ are slightly too small $(a=50, 000, b=100, 000)$. It is as good as
\textsf{SG1} for data sets A and~B, but asymptotically slightly below the others
for data set C.  All the curves are obtained as the average of three experiments
with different initializations. Interestingly, even though the problem is not
convex, the different initializations have led to very similar values of the
objective function and the variance of the experiments was always insignificant
after $10$ seconds of computations.
\subsection{Non Negative Matrix Factorization and Non Negative Sparse Coding} \label{subsec:exp:nmf}
In this section, we compare our method with the classical algorithm of
\citet{lee2} for NMF and the non-negative sparse coding algorithm of
\citet{hoyer} for NNSC. The experiments have been 
carried out on three data sets with different sizes:
\begin{itemize}
\item Data set D is composed of $n=2, 429$ face images of size $m = 19 \times 19$
   pixels from the the MIT-CBCL Face Database $\# 1$ \citep{sung}.
\item Data set E is composed of $n=2, 414$ face images of size $m= 192 \times 168$
   pixels from the Extended Yale B Database \citep{georghiades,lee3}.
\item Data set F is composed of $n=100, 000$ natural image patches of size $m=16
   \times 16$ pixels from the Pascal VOC'06 image database \citep{pascal2}.
\end{itemize}
We have used the Matlab implementations of NMF and NNSC of P. Hoyer, which are
freely available at \url{http://www.cs.helsinki.fi/u/phoyer/software.html}.
Even though our C++ implementation has a built-in advantage in terms of speed
over these Matlab implementations, most of the computational time of NMF and
NNSC is spent on large matrix multiplications, which are typically well optimized in
Matlab.  All the experiments have been run for simplicity on a
single-CPU, single-core
2.4GHz machine, without using the parameters $\rho$ and $t_0$ presented in
Section \ref{sec:variants}---that is, $\rho=0$ and $t_0=0$.  As in
Section~\ref{subsec:exp:perf}, a minibatch of size $\eta=512$ is chosen.
Following the original experiment of \citet{lee2} on data set D, we have chosen to learn $k=49$ basis vectors for the face images data sets D and E, and we have chosen $k=64$ for data set~F. Each input vector is normalized to
have unit $\ell_2$-norm.

The experiments we present in this section compare the value of the objective
function on the data sets obtained with the different algorithms as a
function of the computation time. Since our algorithm learns the matrix
$\D$ but does not provide the matrix $\alphab$, the computation times reported for our approach
include two steps: First, we run our algorithm to obtain $\D$. Second,
we run one sparse coding step over all the input vectors to obtain~$\alphab$.
Figure \ref{fig:expnnsc} presents the results
for NMF and NNSC. The gradient step for the algorithm of \citet{hoyer} was
optimized for the best performance and $\lambda$ was set to $\frac{1}{\sqrt{m}}$.
Both $\D$ and~$\alphab$ were initialized
randomly. The values reported are those obtained for more than
$0.1$s of computation. Since the random initialization provides
an objective value which is by far greater than the value obtained at
convergence, the curves are all truncated to present significant objective
values.  All the results are obtained using the average of 3 experiments with
different initializations.  As shown on Figure \ref{fig:expnnsc}, our algorithm
provides a significant improvement in terms of speed compared to the other
tested methods, even though the results for NMF and NNSC could be
improved a bit using a C++ implementation.
\begin{figure}[hbtp]
 \centering
     \includegraphics[width=0.48\linewidth]{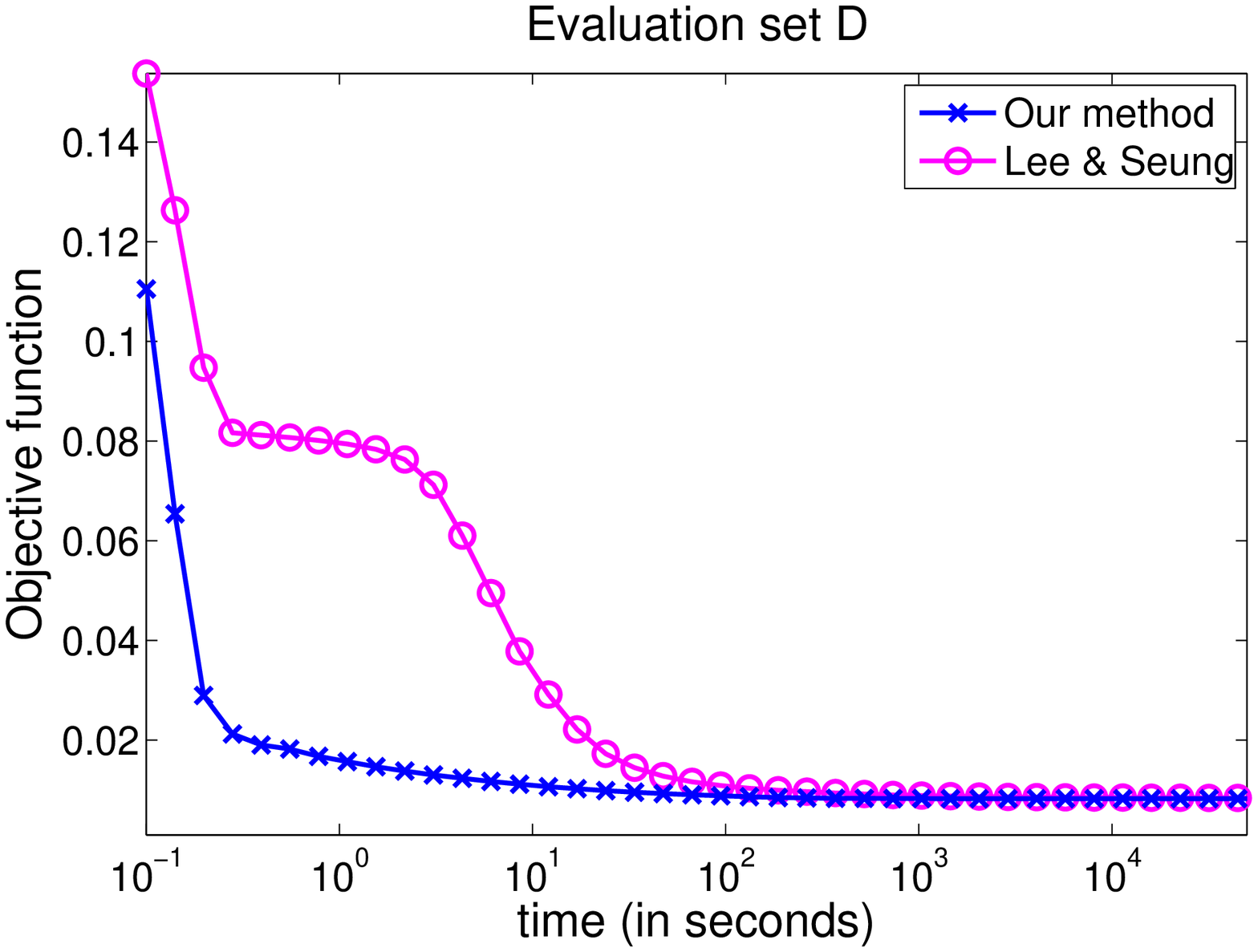} \hfill
     \includegraphics[width=0.49\linewidth]{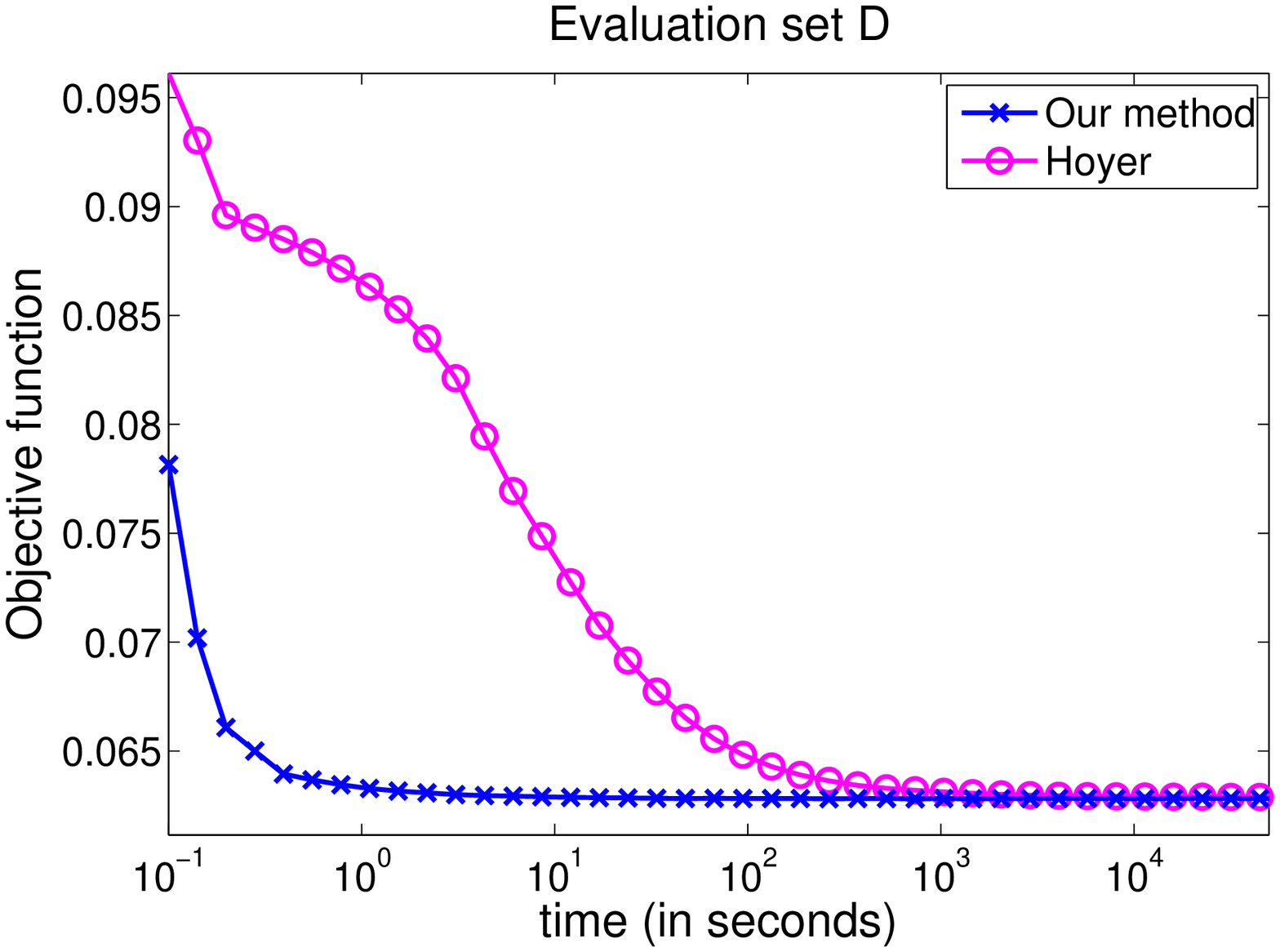}  \\
     \vspace*{0.3cm}
     \includegraphics[width=0.48\linewidth]{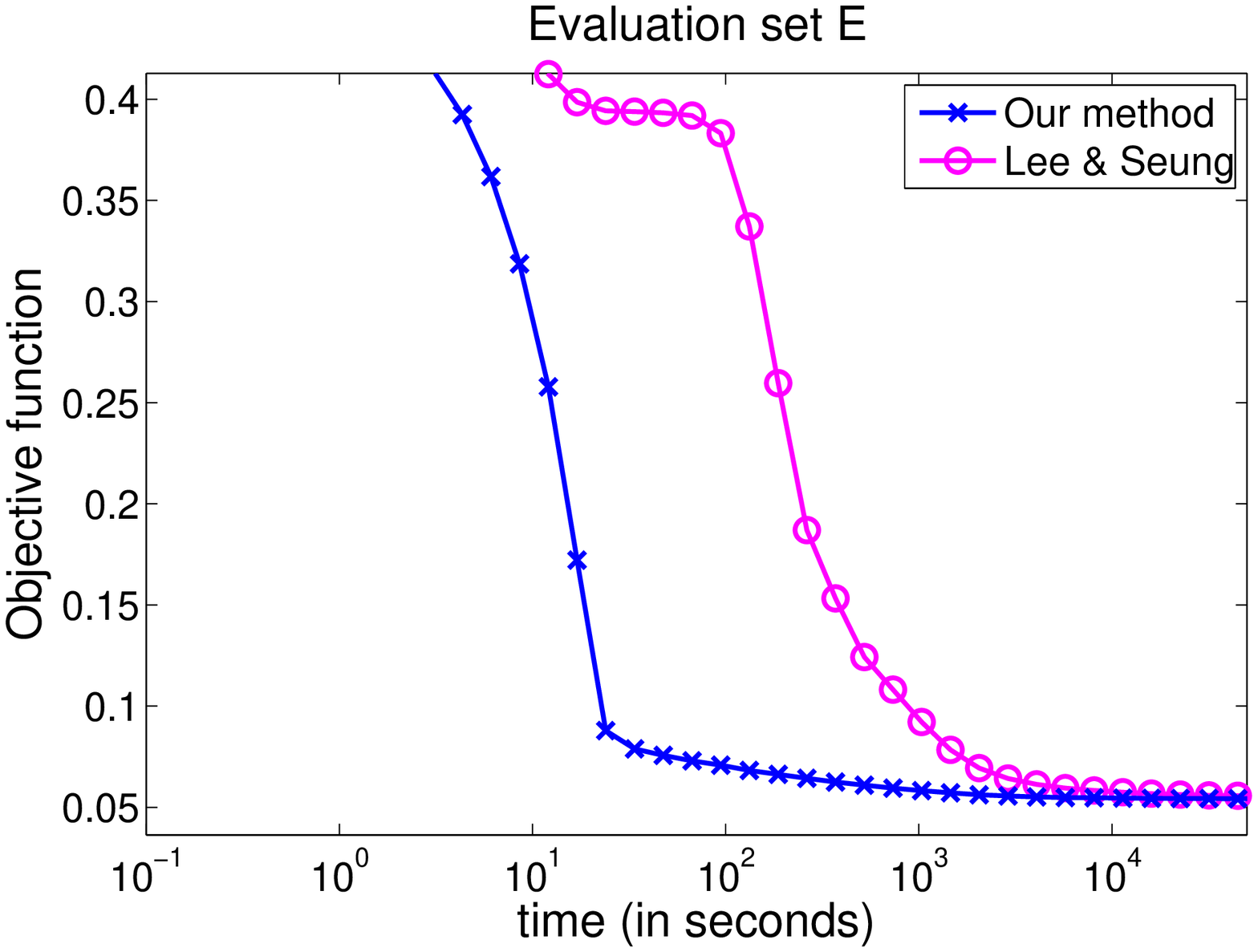} \hfill
     \includegraphics[width=0.49\linewidth]{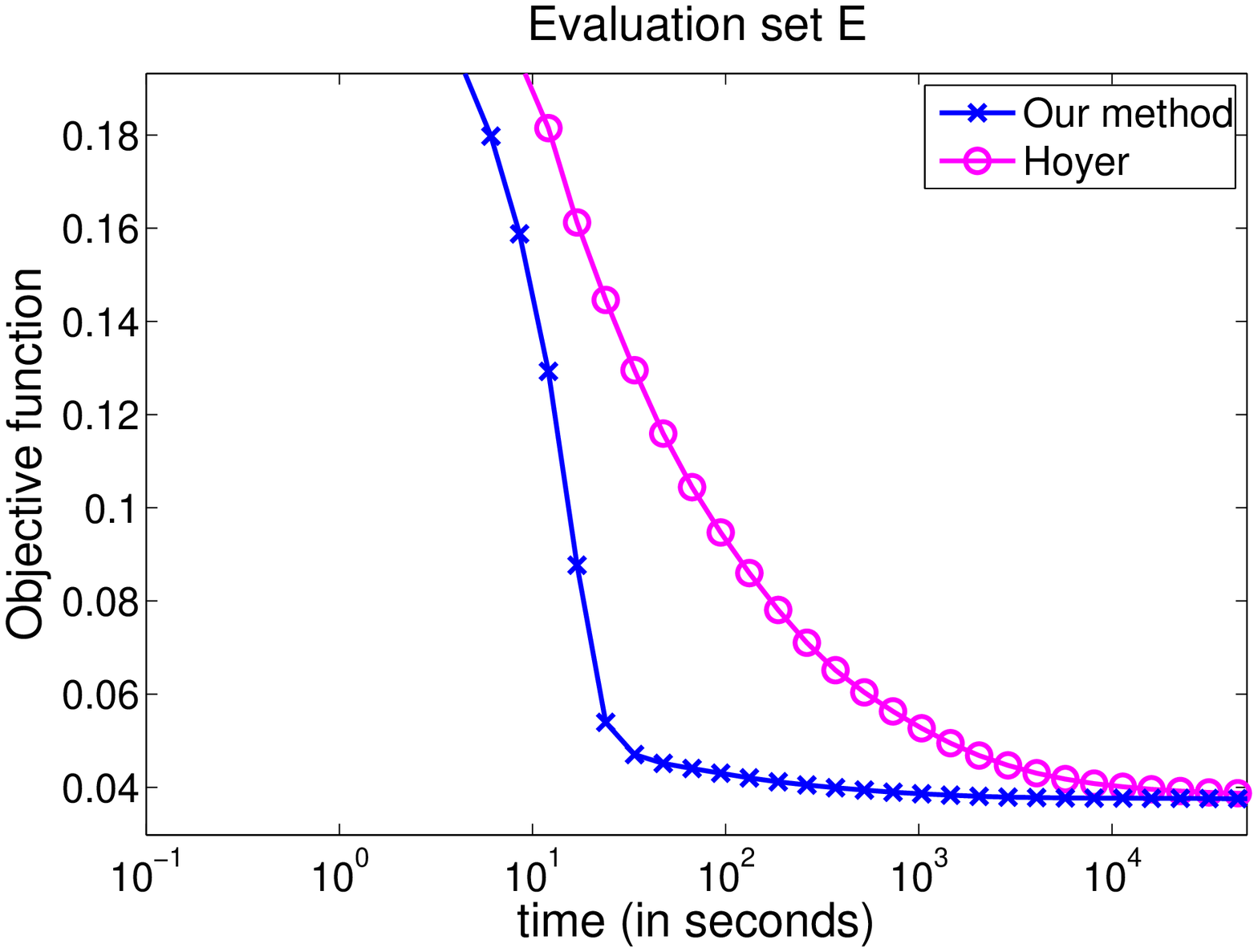} \\
     \vspace*{0.3cm}
     \includegraphics[width=0.48\linewidth]{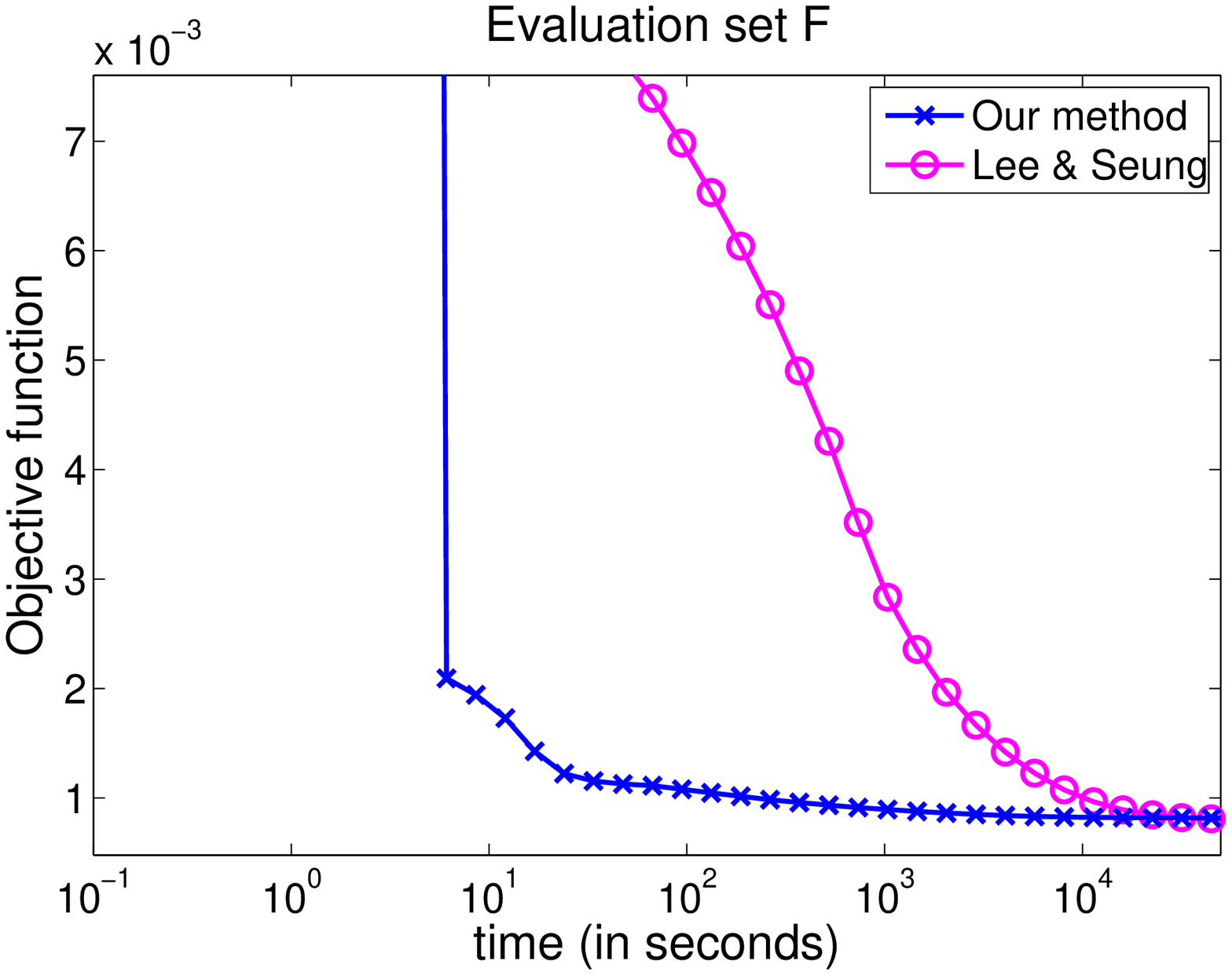} \hfill
     \includegraphics[width=0.51\linewidth]{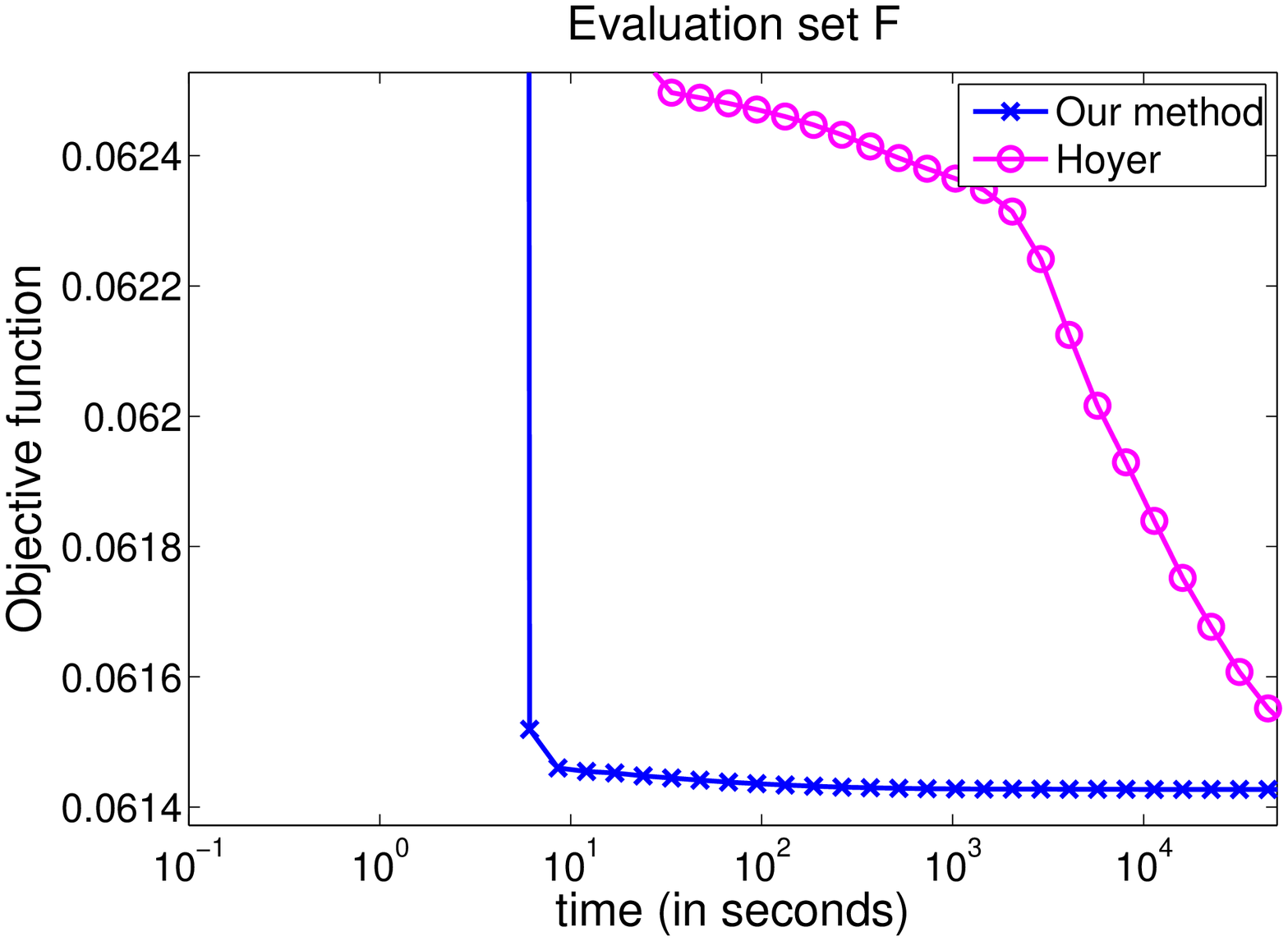}
     \caption{Left: Comparison between our method and the approach of \citet{lee2} for NMF. Right: Comparison between our method and the approach of \citet{hoyer} for NNSC. The value of the objective function is reported
     for three data sets as a function of computation time on a logarithmic
     scale.}  \label{fig:expnnsc}
 \end{figure}
\subsection{Sparse Principal Component Analysis}
We present here the application of our method addressing SPCA with various types of data: faces, natural image patches, and genomic data.
\subsubsection{Faces and Natural Patches} \label{subsubsec:patches}
In this section, we compare qualitatively the results obtained by PCA, NMF,
our dictionary learning and our sparse principal component analysis algorithm on the
data sets used in Section \ref{subsec:exp:nmf}. 
For dictionary learning, PCA and SPCA, the input vectors are first centered
and normalized to have a unit norm. Visual results are presented on figures
\ref{fig:spca:data1}, \ref{fig:spca:data2} and \ref{fig:spca:data3}, respectively
for the data sets D, E and F.
The parameter $\lambda$ for dictionary learning and SPCA was set so that the
decomposition of each input signal has approximately $10$ nonzero coefficients.
The results for SPCA are presented for various values of the parameter
$\gamma$, yielding different levels of sparsity. The scalar $\tau$ indicates
the percentage of nonzero values of the dictionary.

Each image is composed of $k$ small images each representing one learned
feature vector. Negative values are blue, positive values are red and the zero
values are represented in white.  Confirming earlier observations from
\citet{lee2}, PCA systematically produces features spread out over the
images, whereas NMF
produces more localized features on the face databases D and E. However,
neither PCA, nor NMF are able to learn localized features on the set of
natural patches F.  On the other hand, the dictionary learning technique is
able to learn localized features on data set F, and SPCA
is the only tested method that allows
controlling the level of sparsity among the learned matrices.
\begin{figure}
   \centering
   \subfigure[PCA]{\fbox{\includegraphics[width=0.39\linewidth]{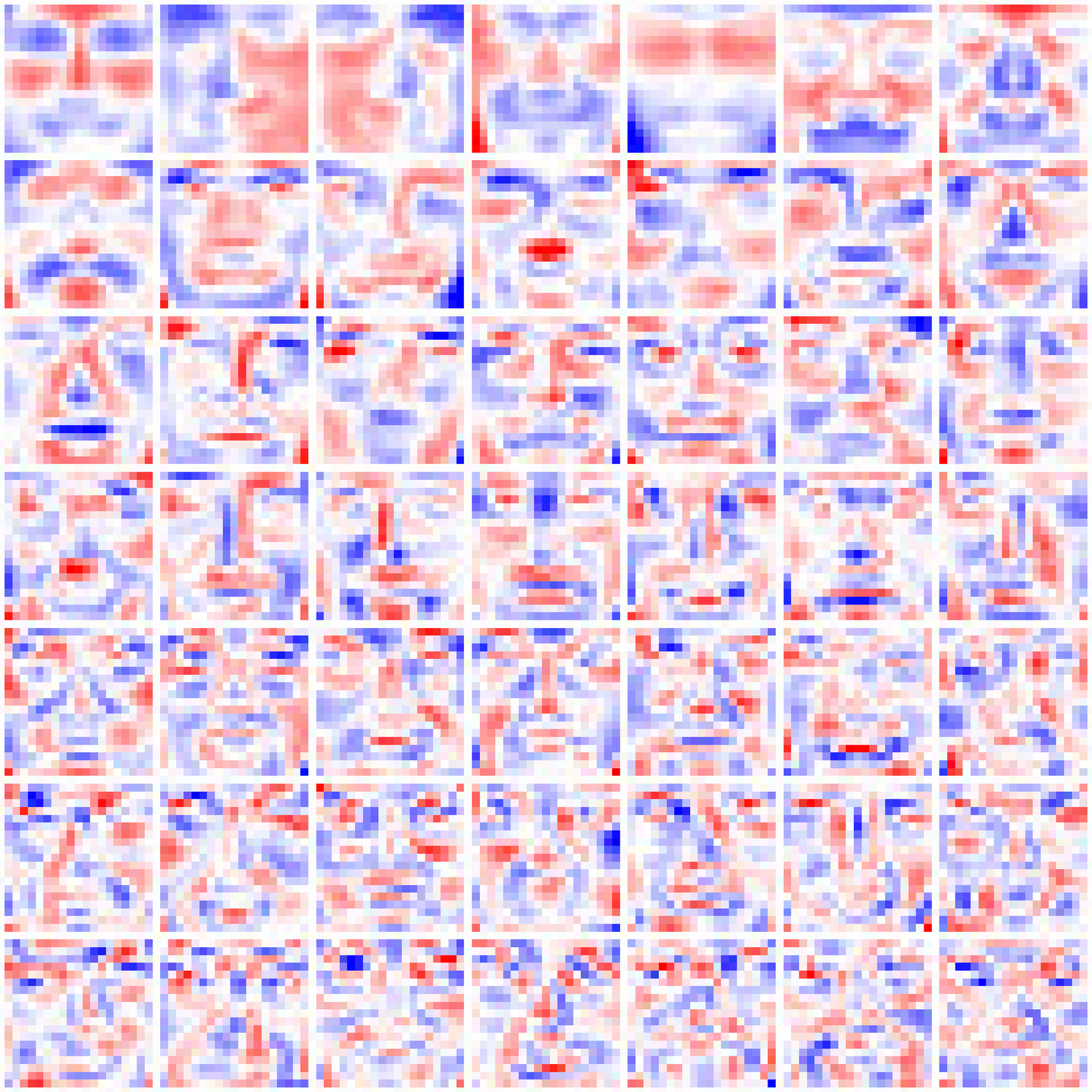}}}
   \subfigure[SPCA, $\tau=70\%$]{\fbox{\includegraphics[width=0.39\linewidth]{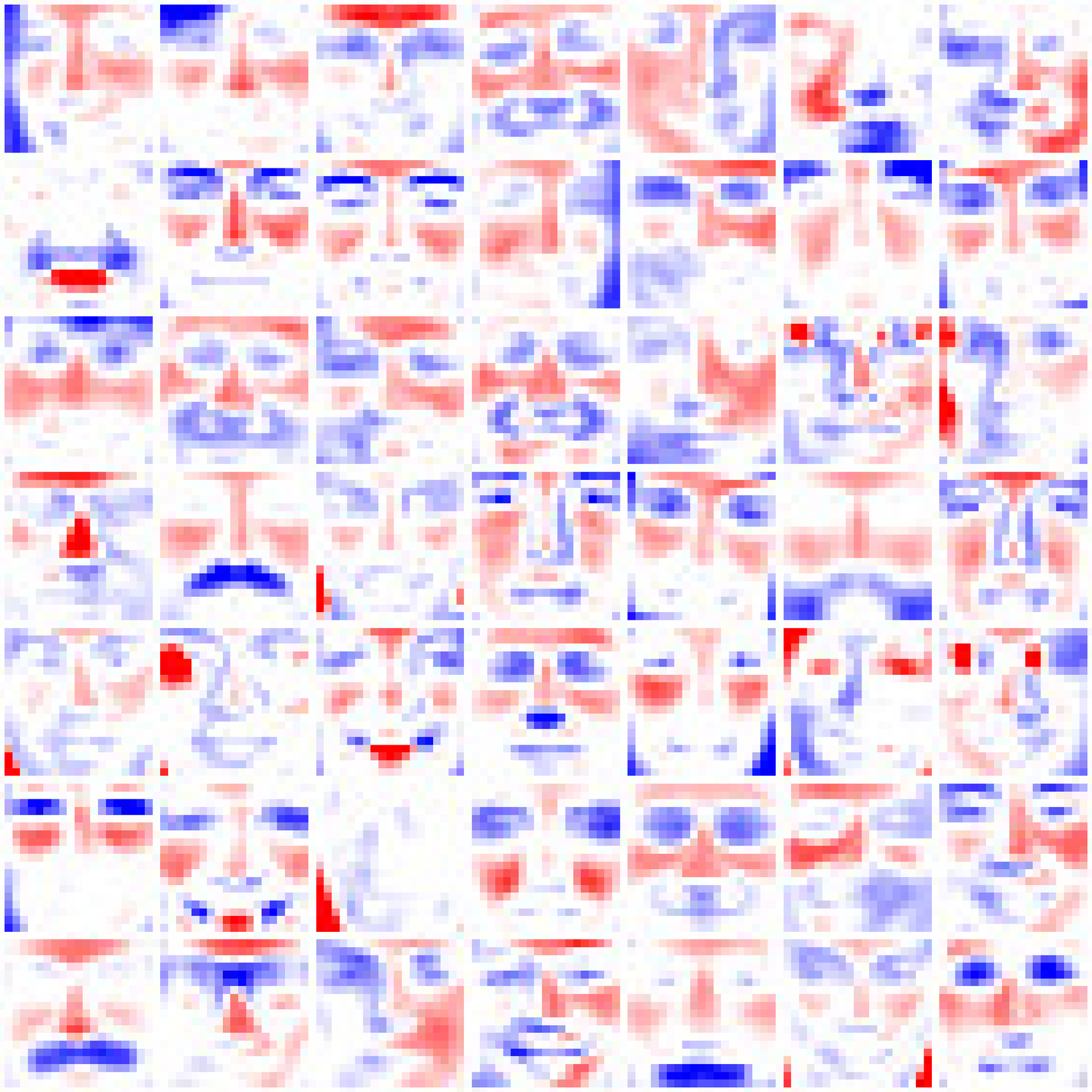}}} \\
   \subfigure[NMF]{\fbox{\includegraphics[width=0.39\linewidth]{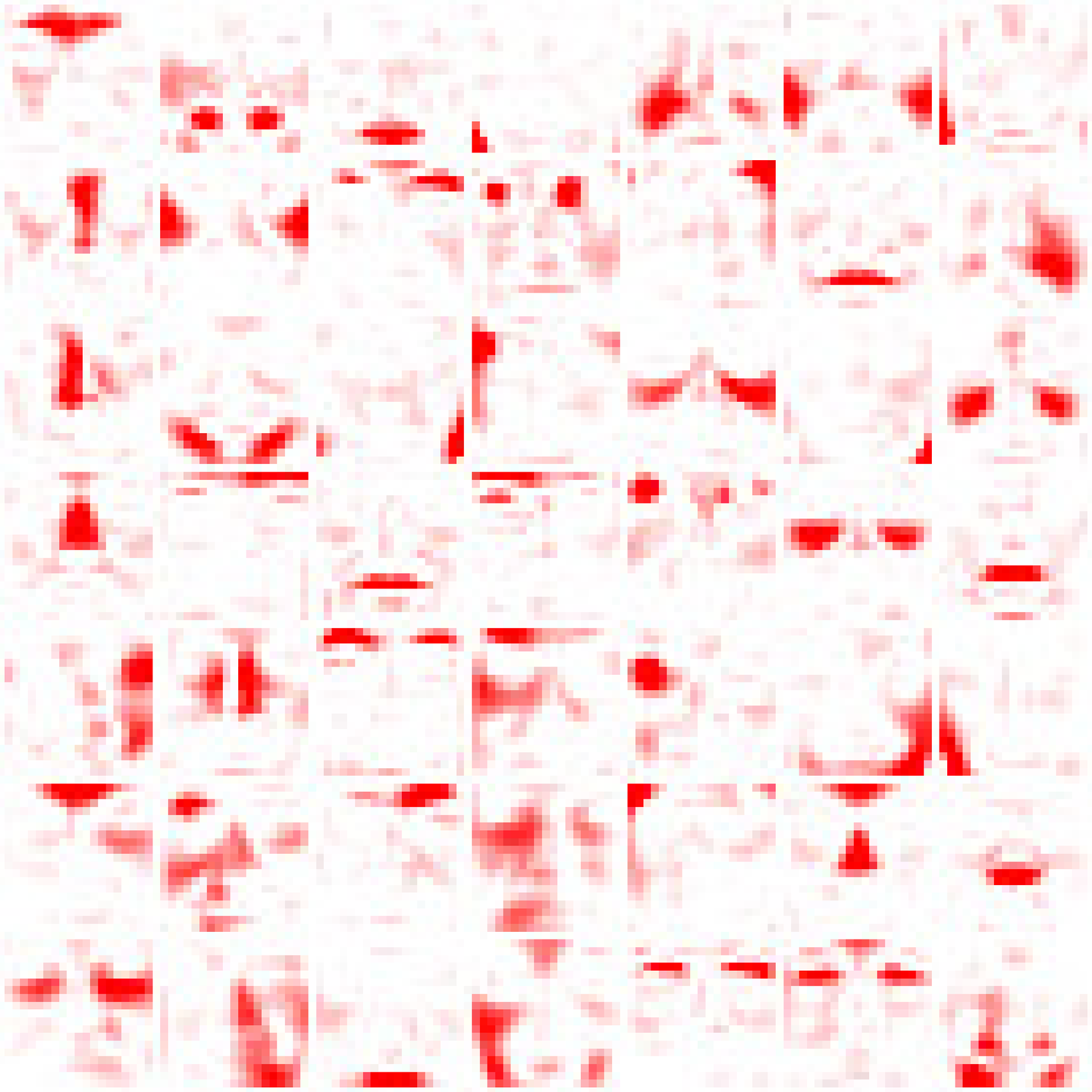}}}
   \subfigure[SPCA, $\tau=30\%$]{\fbox{\includegraphics[width=0.39\linewidth]{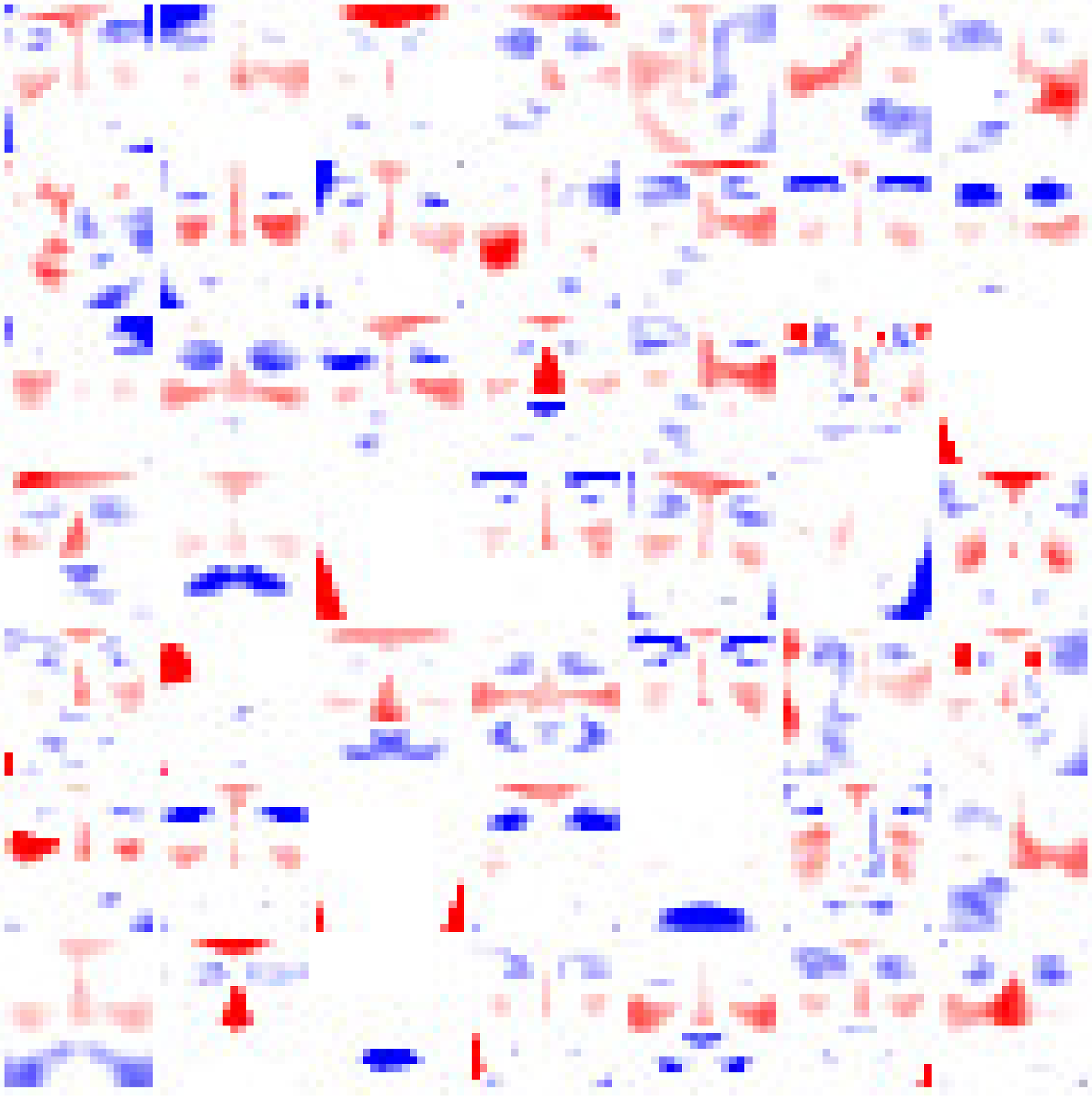}}} \\
   \subfigure[Dictionary Learning]{\fbox{\includegraphics[width=0.39\linewidth]{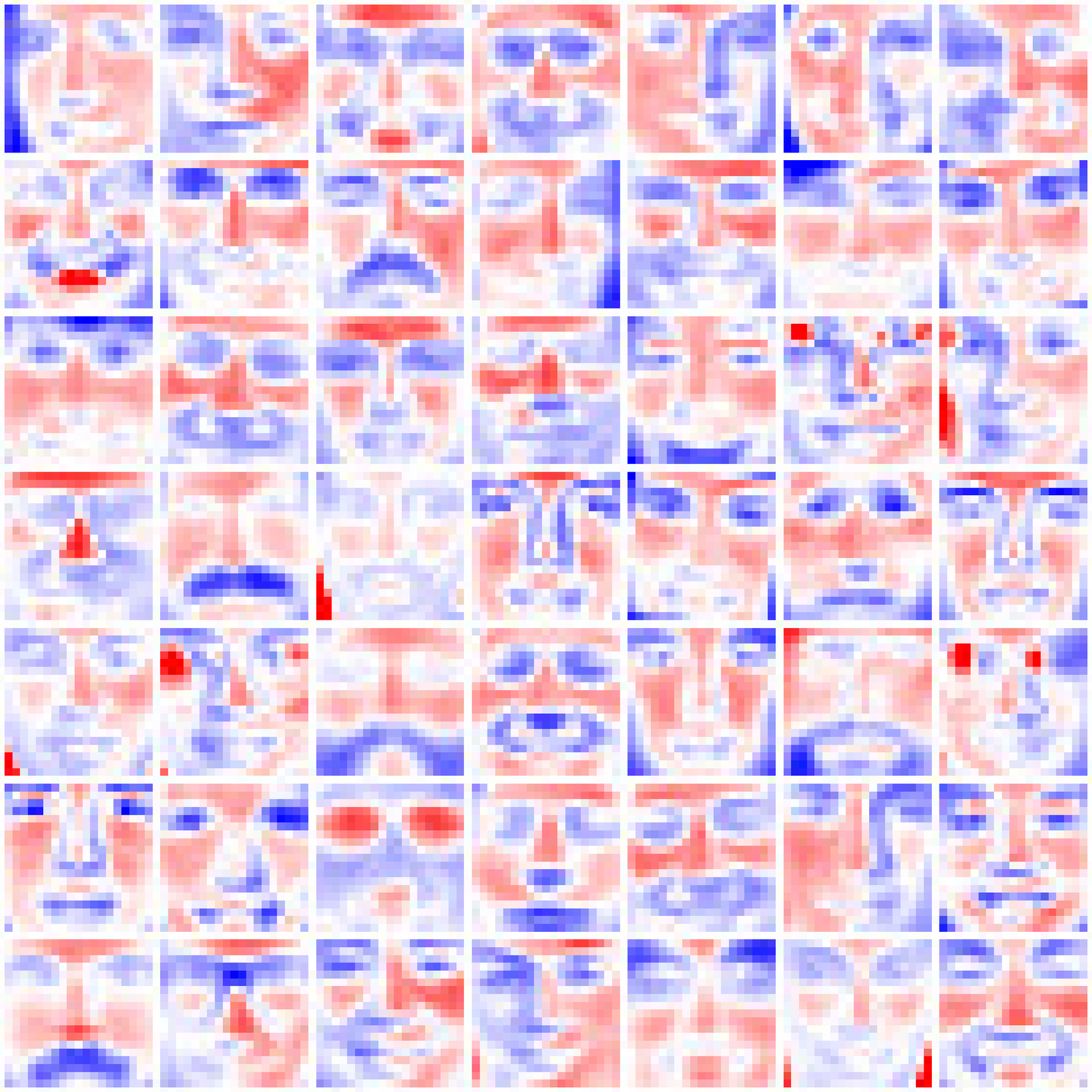}}} 
   \subfigure[SPCA, $\tau=10\%$]{\fbox{\includegraphics[width=0.39\linewidth]{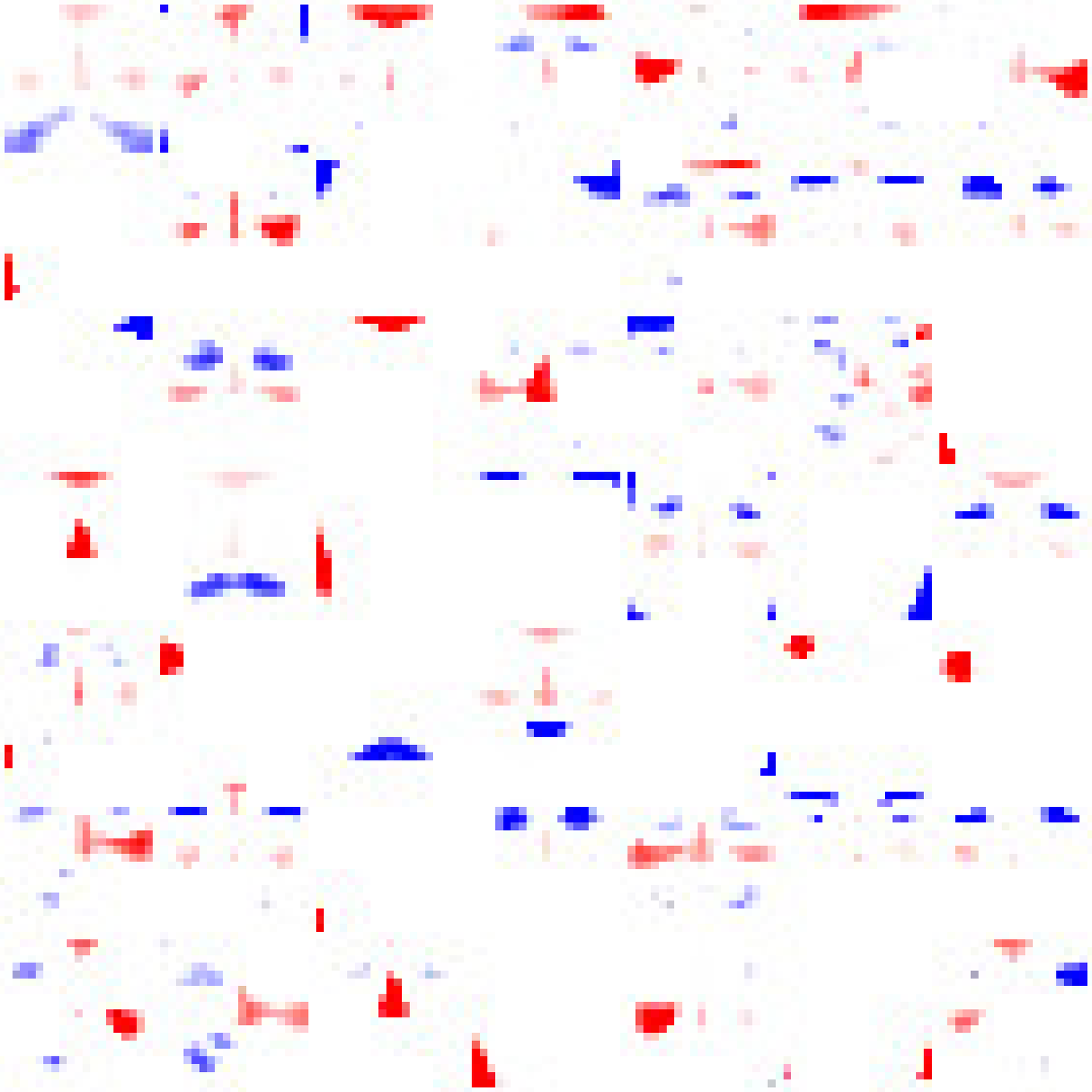}}}
   \caption{Results obtained by PCA, NMF, dictionary learning, SPCA for data set D.}
   \label{fig:spca:data1}
\end{figure}
\begin{figure}
   \centering
   \subfigure[PCA]{\fbox{\includegraphics[height=0.39\linewidth]{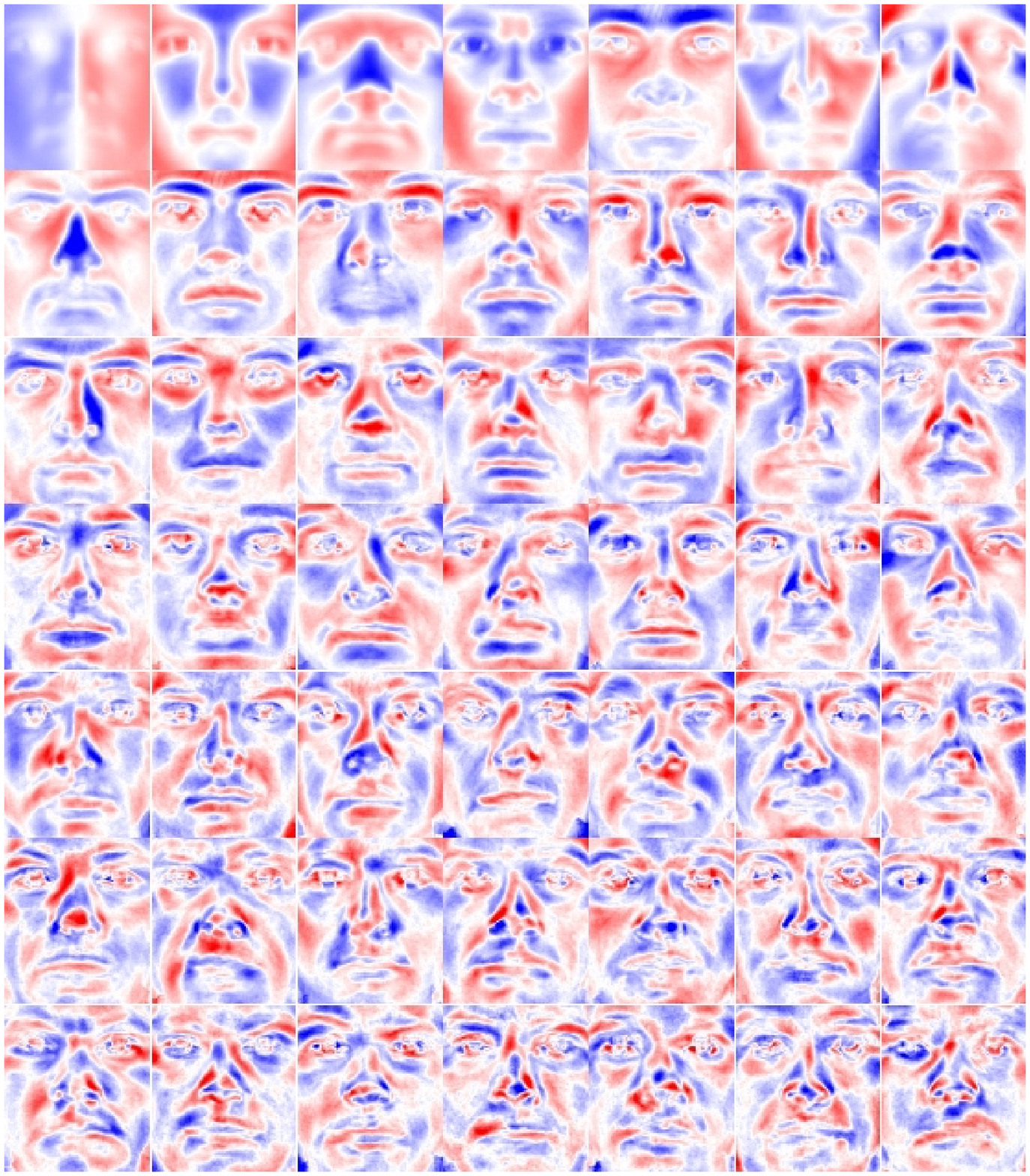}}}
   \subfigure[SPCA, $\tau=70\%$]{\fbox{\includegraphics[height=0.39\linewidth]{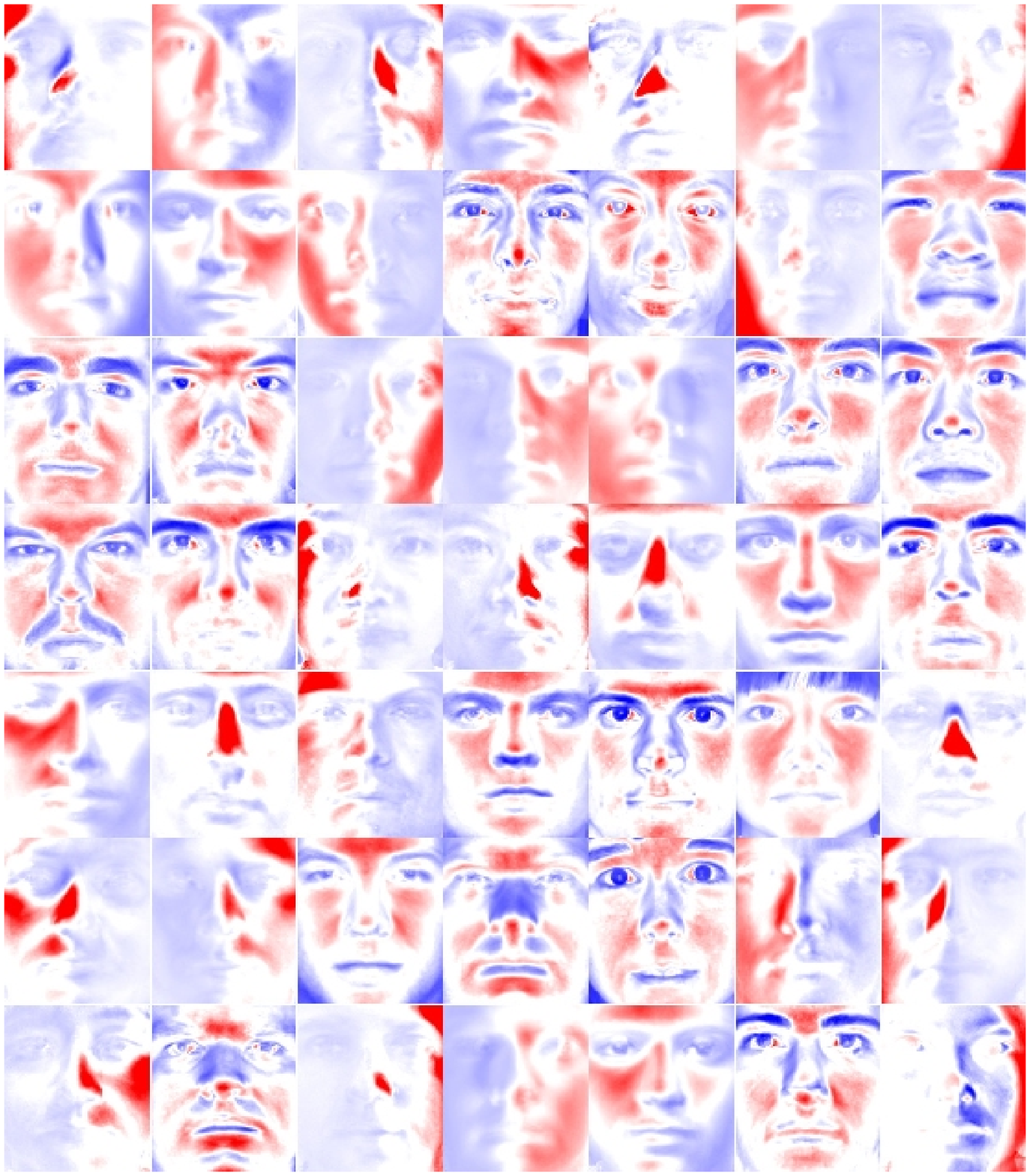}}} \\
   \subfigure[NMF]{\fbox{\includegraphics[height=0.39\linewidth]{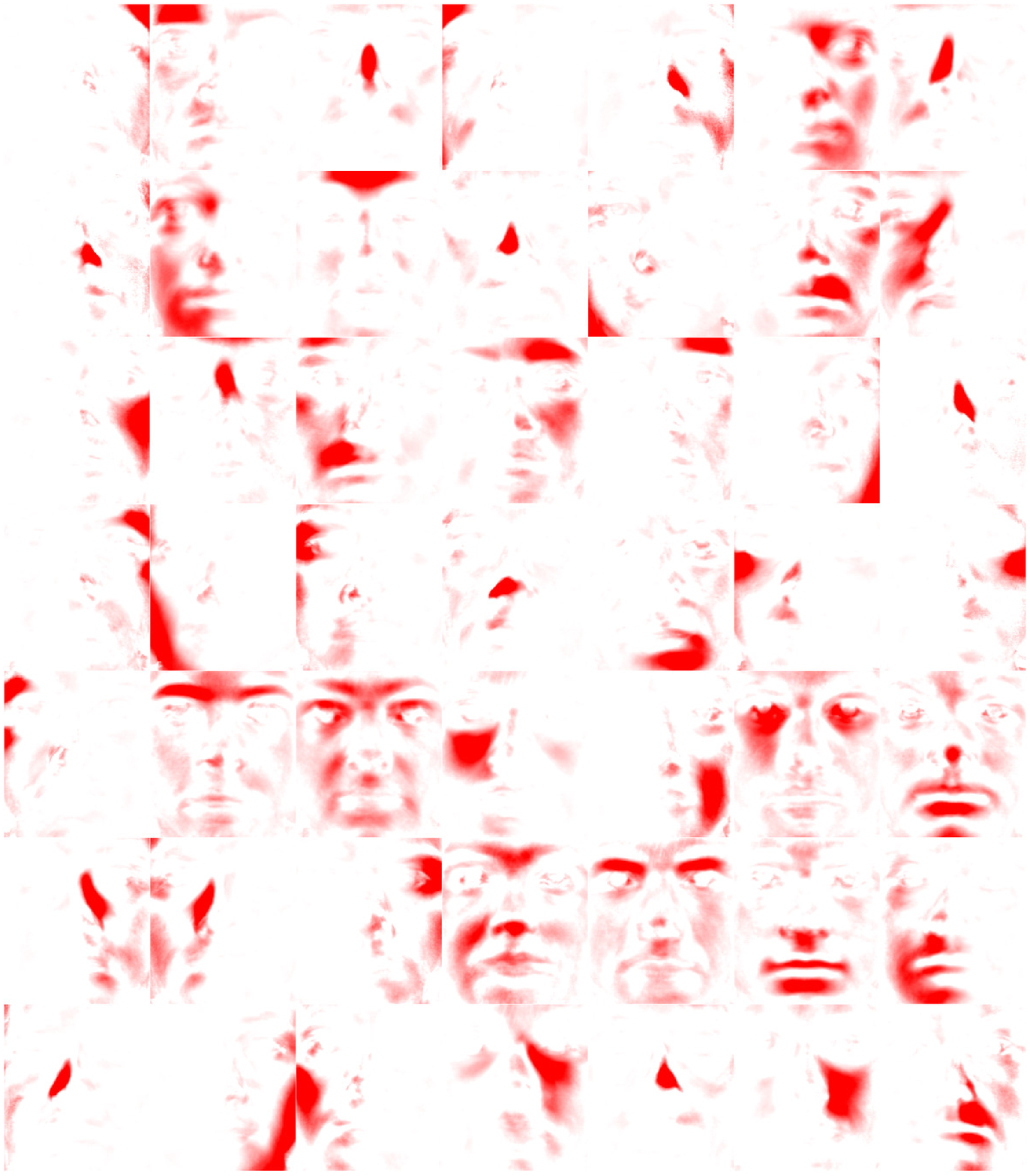}}}
   \subfigure[SPCA, $\tau=30\%$]{\fbox{\includegraphics[height=0.39\linewidth]{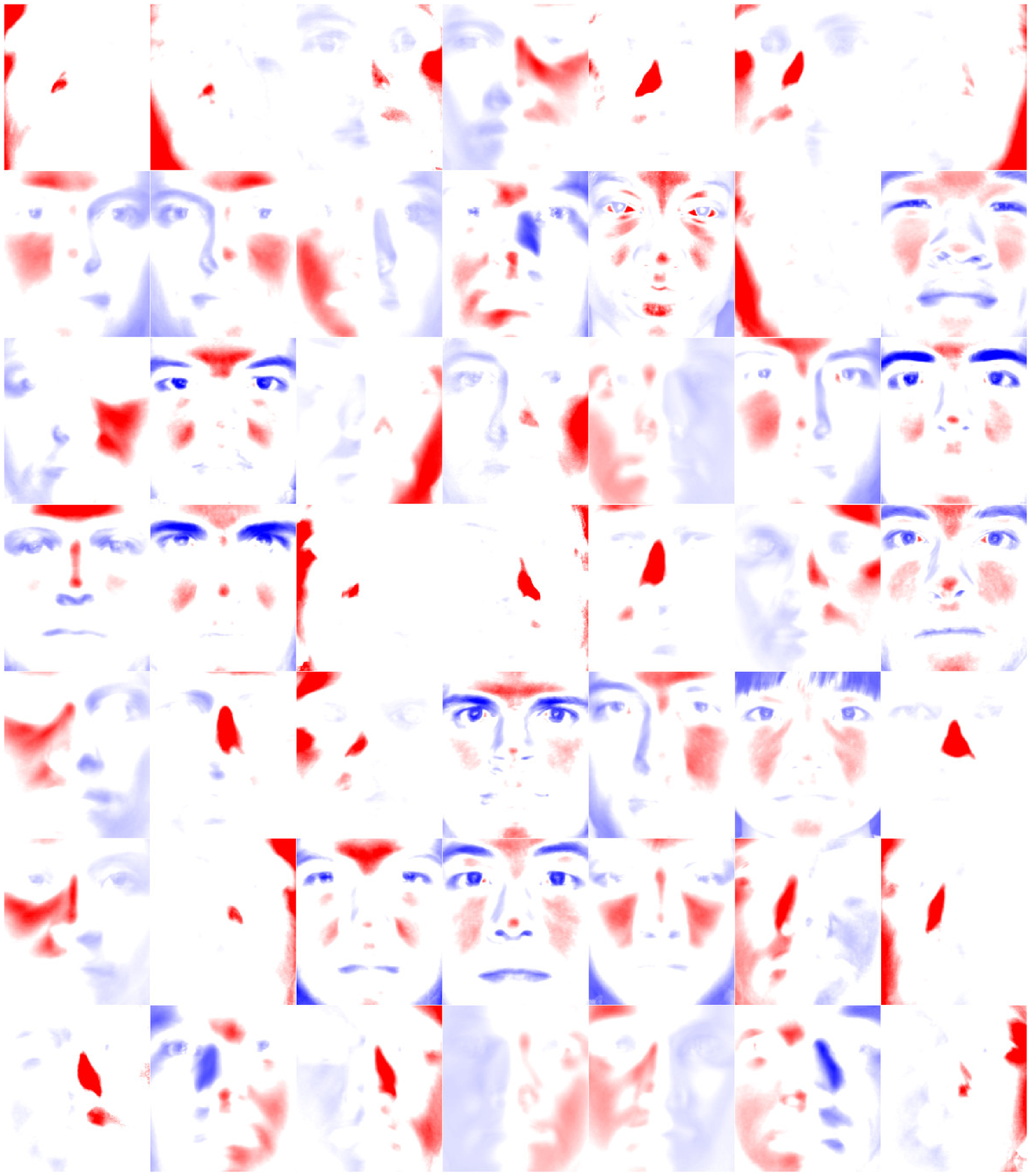}}} \\
   \subfigure[Dictionary Learning]{\fbox{\includegraphics[height=0.39\linewidth]{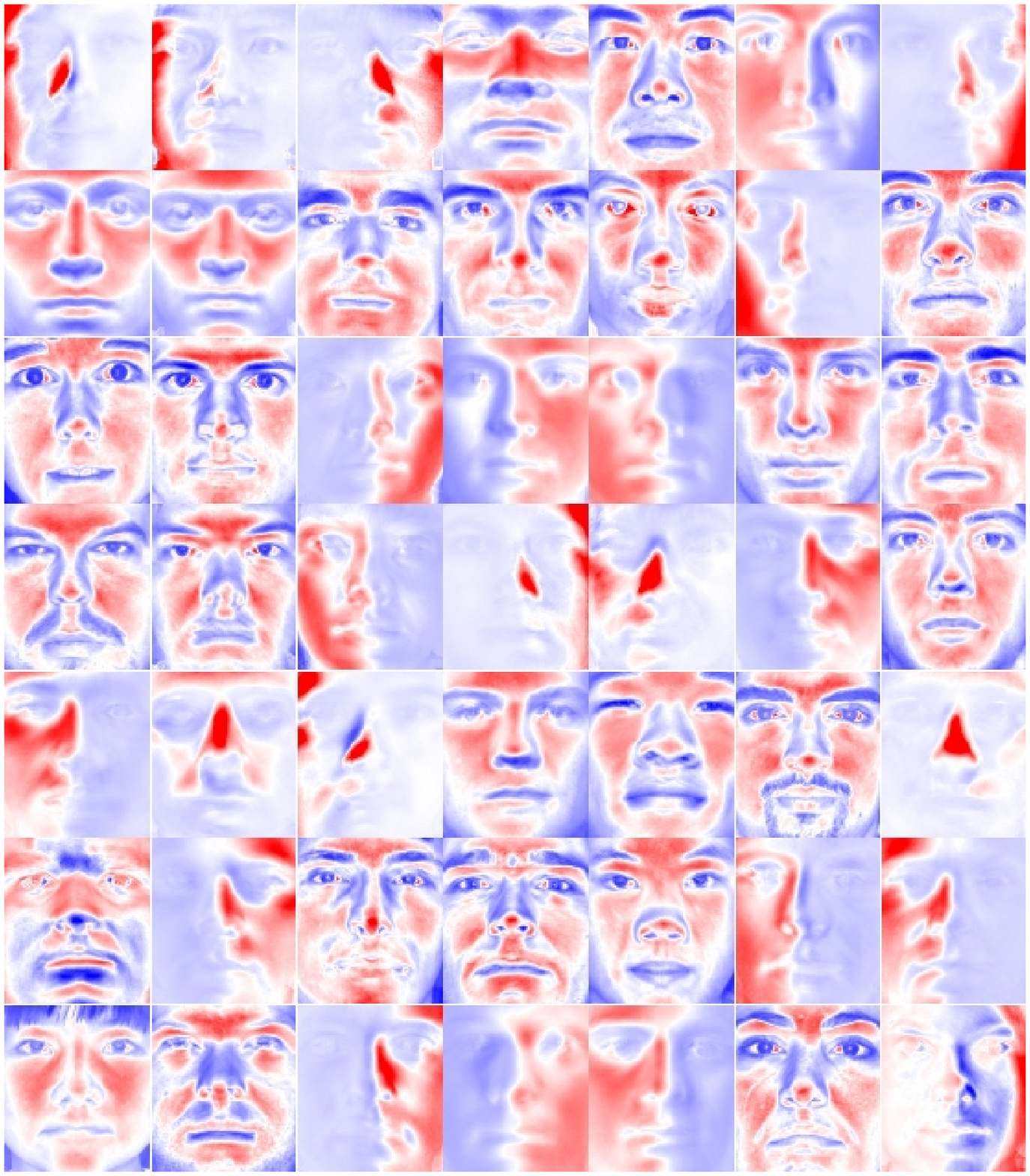}}} 
   \subfigure[SPCA, $\tau=10\%$]{\fbox{\includegraphics[height=0.39\linewidth]{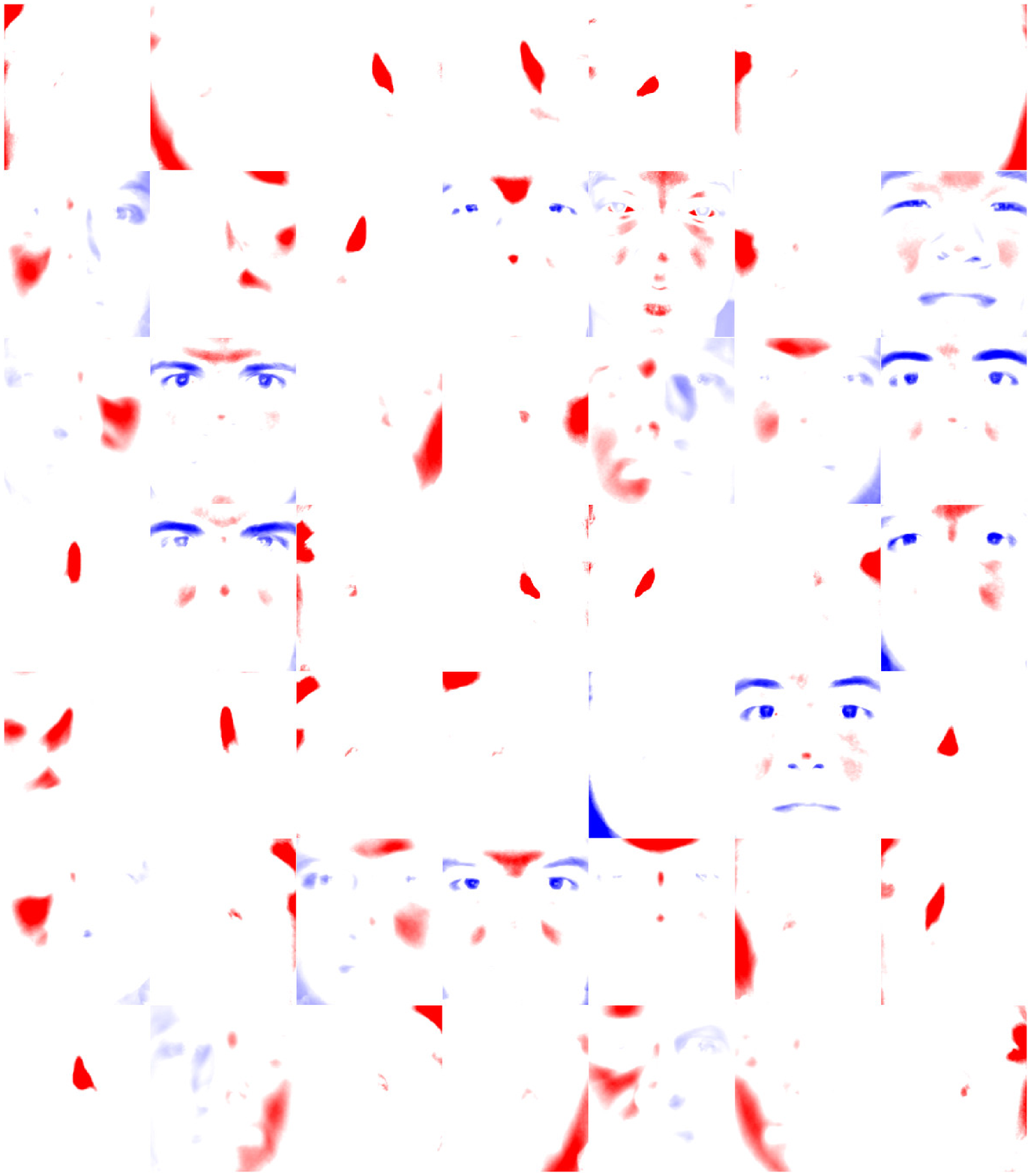}}}
   \caption{Results obtained by PCA, NMF, dictionary learning, SPCA for data set E.}
   \label{fig:spca:data2}
\end{figure}
\begin{figure}
   \centering
   \subfigure[PCA]{\fbox{\includegraphics[width=0.39\linewidth]{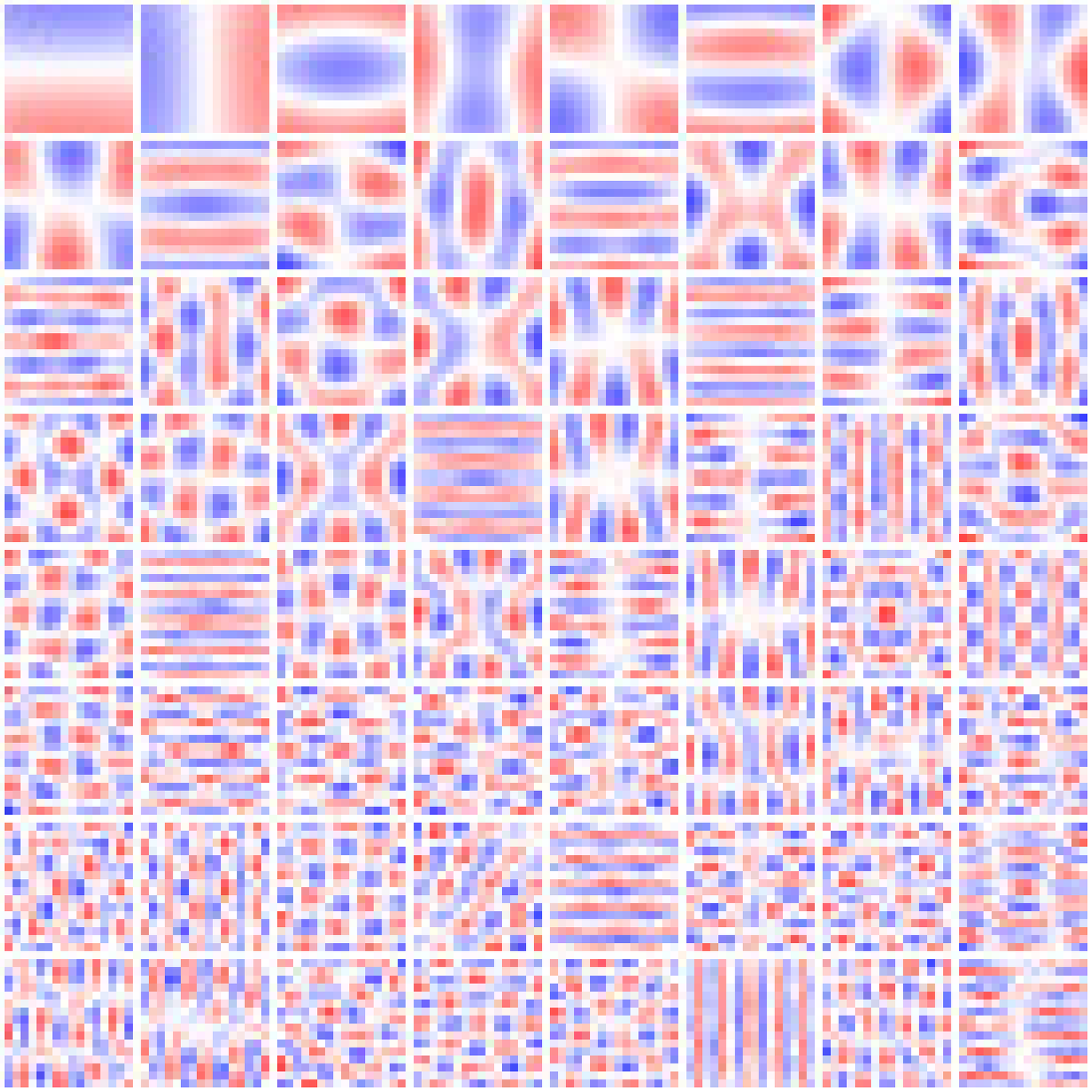}}}
   \subfigure[SPCA, $\tau=70\%$]{\fbox{\includegraphics[width=0.39\linewidth]{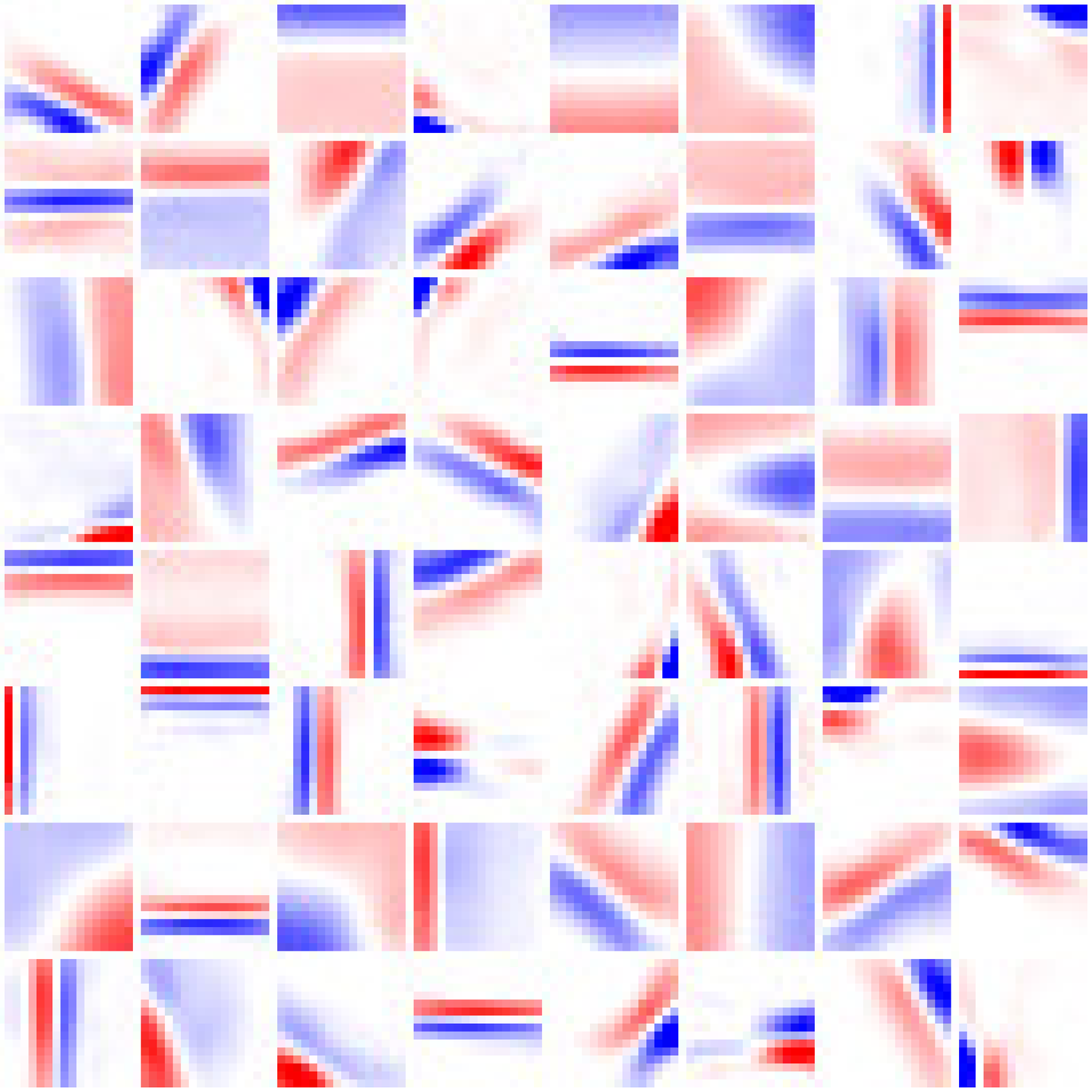}}} \\
   \subfigure[NMF]{\fbox{\includegraphics[width=0.39\linewidth]{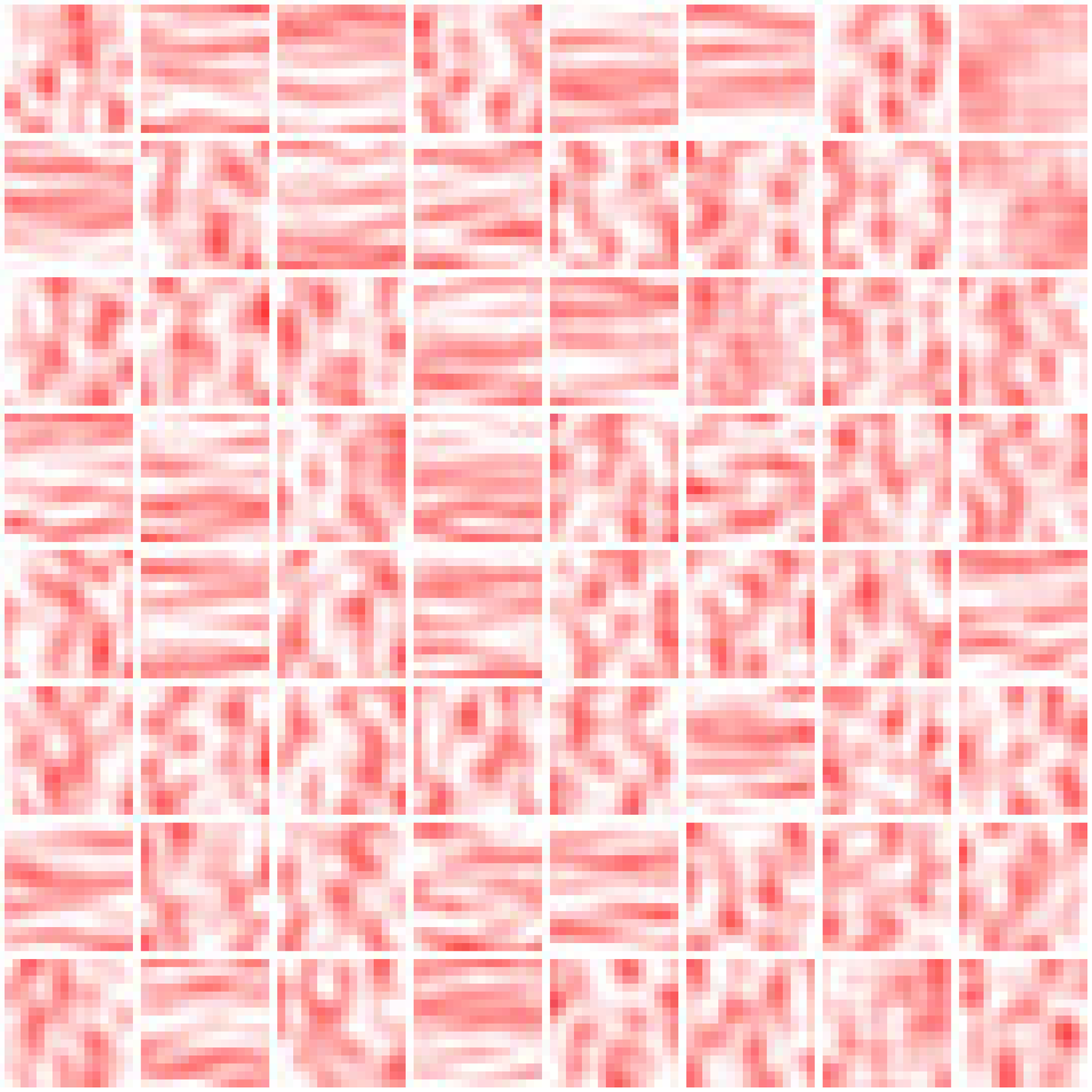}}}
   \subfigure[SPCA, $\tau=30\%$]{\fbox{\includegraphics[width=0.39\linewidth]{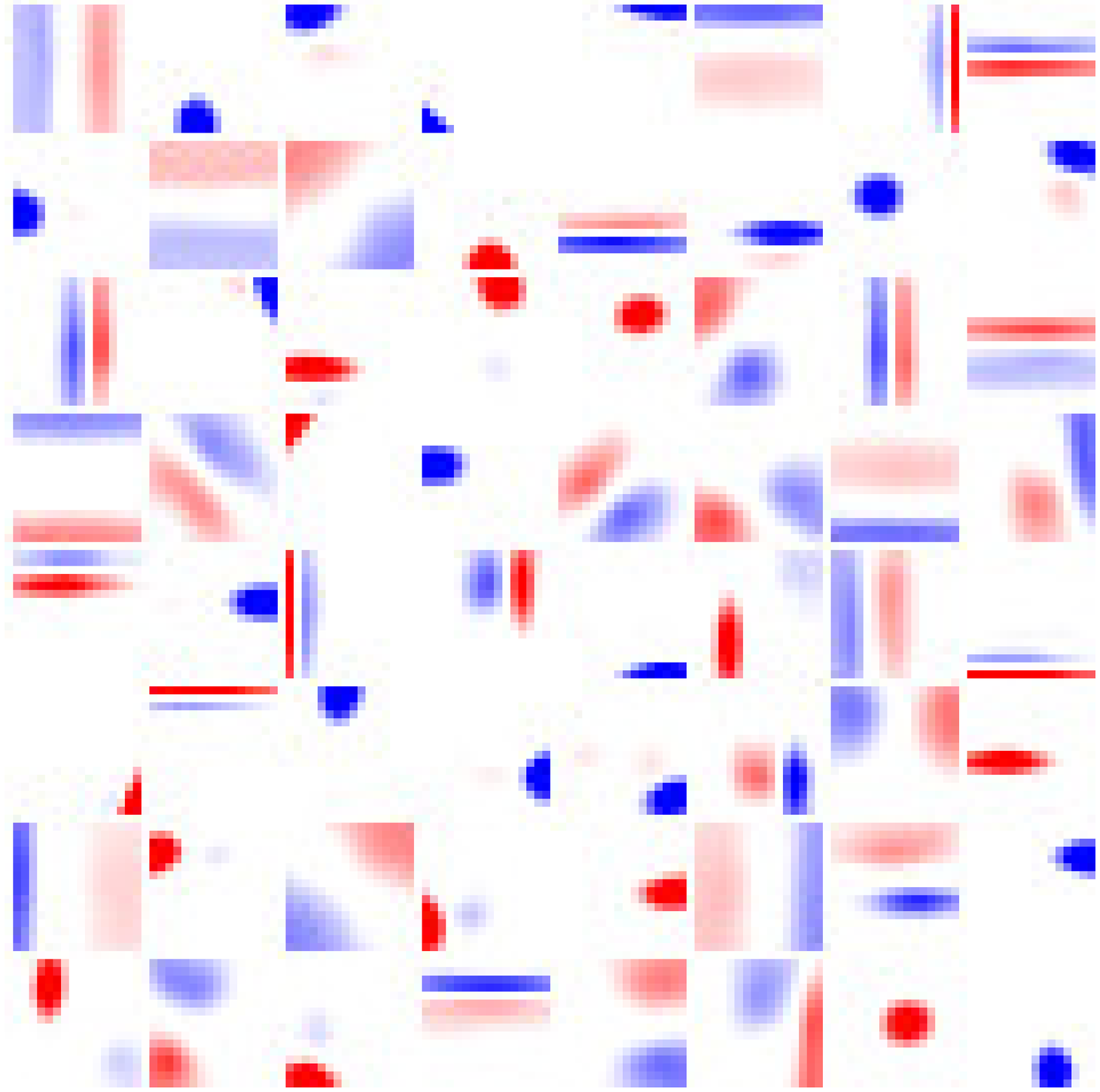}}} \\
   \subfigure[Dictionary Learning]{\fbox{\includegraphics[width=0.39\linewidth]{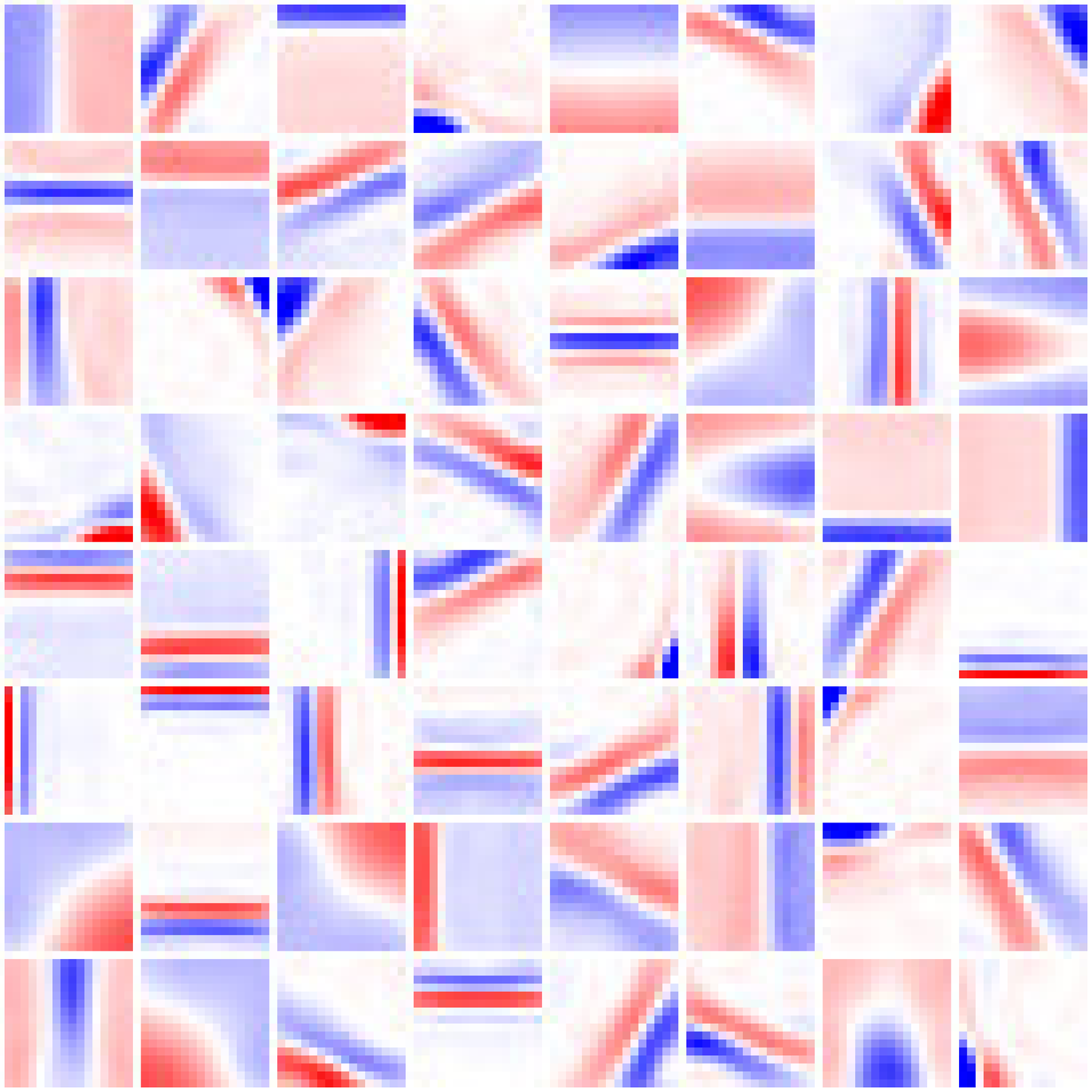}}} 
   \subfigure[SPCA, $\tau=10\%$]{\fbox{\includegraphics[width=0.39\linewidth]{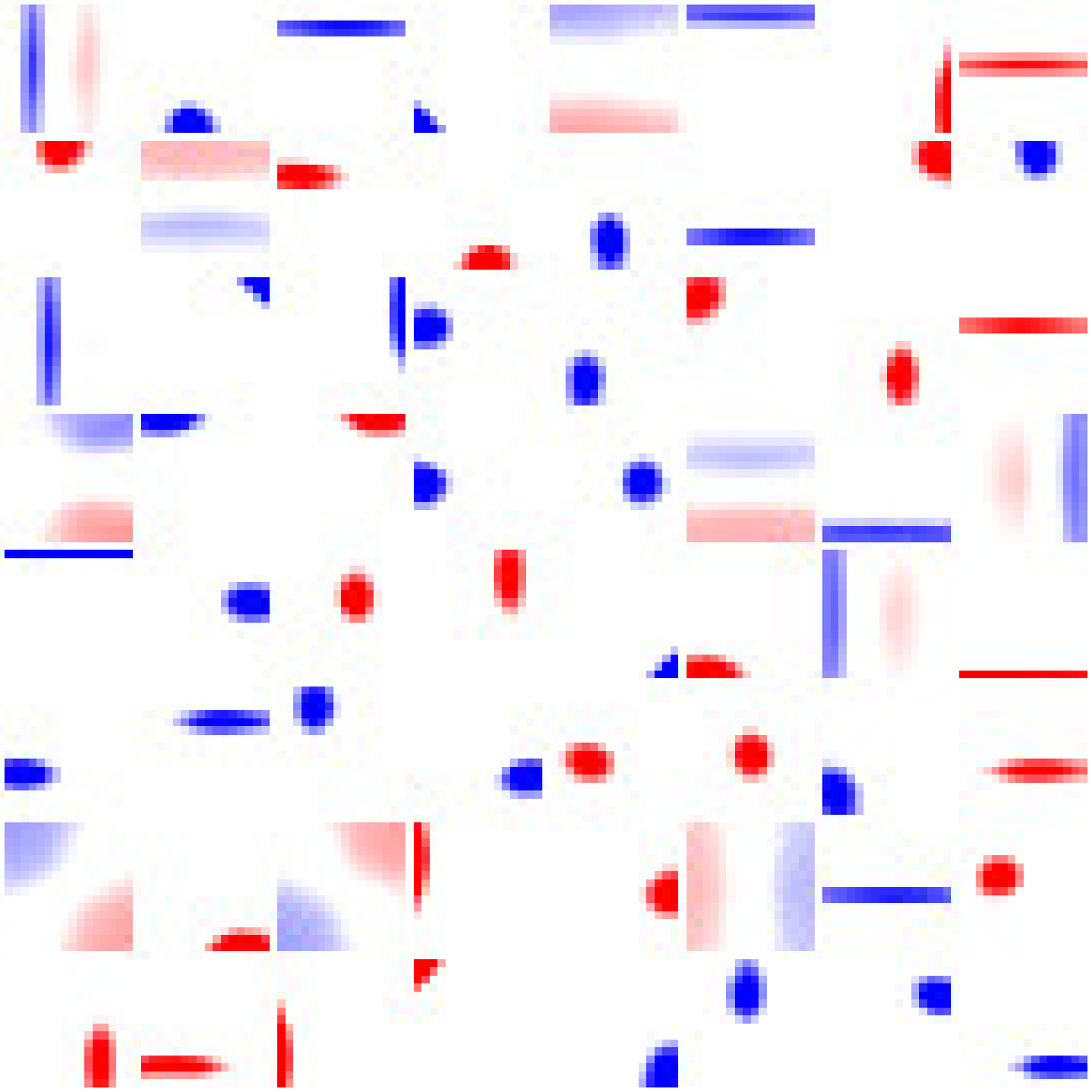}}}
   \caption{Results obtained by PCA, NMF, dictionary learning, SPCA for data set F.}
   \label{fig:spca:data3}
\end{figure}
\subsubsection{Genomic Data}
This experiment follows \citet{witten} and demonstrates that our matrix
decomposition technique can be used for analyzing genomic data.
Gene expression measurements and DNA copy number changes (comparative genomic hybridization CGH) are two popular
types of data in genomic research, which can be used to characterize
a set of abnormal tissue samples for instance.
When these two types of data are available, a recent line of research
tries to analyze the correlation between them---that is, to determine
sets of expression genes which are correlated with sets of chromosomal
gains or losses (see \citealp{witten} and references therein).
Let us suppose that for $n$ tissue samples, we have a matrix $\X$ in $\Real^{n
\times p}$ of gene expression measurements and a matrix $\Y$ in $\Real^{n
\times q}$ of CGH measurements. In order to analyze the
correlation between these two sets of data, recent works have suggested the
use of canonical correlation analysis \citep{hotelling}, which
solves\footnote{Note that when more than one couple of factors are needed, two
sequences
$\u_1,\u_2,\ldots$ and $\v_1,\v_2,\ldots$ of factors can be obtained
recursively subject to orthogonality constraints of the sequences
$\X\u_1,\X\u_2,\ldots$ and $\Y\v_1,\Y\v_2,\ldots$.}
\begin{displaymath}
   \min_{\u \in \Real^p, \v \in \Real^q} \cov(\X\u,\Y\v) \st ||\X\u||_2 \leq 1
   ~~\text{and}~~ ||\Y\v||_2 \leq 1.
\end{displaymath}
When $\X$ and $\Y$ are centered and normalized, it has been further shown that
with this type of data, good results can be obtained by treating the
covariance matrices $\X^T\X$ and $\Y^T\Y$ as diagonal, leading to a rank-one
matrix decomposition problem 
\begin{displaymath}
   \min_{\u \in \Real^p, \v \in \Real^q} ||\X^T\Y-\u\v^T||_F^2 \st ||\u||_2 \leq
   1,~~\text{and}~~ ||\v||_2 \leq 1.
\end{displaymath}
Furthermore, as shown by \citet{witten}, this method can benefit from
sparse regularizers such as the $\ell_1$ norm for the gene expression
measurements and a fused lasso for the CGH arrays, which are classical
choices used for these data. The formulation we have chosen is slightly different 
from the one used by \citet{witten} and can be addressed using our algorithm:
\begin{equation}
   \min_{\u \in \Real^p, \v \in \Real^q} ||\Y^T\X-\v\u^T||_F^2 + \lambda ||\u||_2 
   \st ||\v||_2^2 + \gamma_1||\v||_1 + \gamma_2 \FL(\v) \leq 1.\label{eq:cancer}
\end{equation}
In order to assess the effectivity of our method, we have conducted the same
experiment as \citet{witten} using the breast cancer data set described by
\citet{chin}, consisting of $q=2, 148$ gene expression measurements and
$p=16, 962$ CGH measurements for $n=89$ tissue samples.  The matrix
decomposition problem of Eq.~(\ref{eq:cancer}) was addressed once for each of
the $23$ chromosomes, using each time the CGH data available for the
corresponding chromosome, and the gene expression of all genes. Following the
original choice of \citet{witten}, we have selected a regularization parameter
$\lambda$ resulting in about $25$ non-zero coefficients in $\u$, and selected
$\gamma_1=\gamma_2=1$, which results in sparse and piecewise-constant vectors $\v$.
The
original matrices $(\X,\Y)$ are divided into a training set
$(\X_{tr},\Y_{tr})$ formed with $3/4$ of the $n$ samples, keeping the rest
$(\X_{te},\Y_{te})$ for testing. This experiment is repeated for $10$ random
splits, for each chromosome a couple of factors $(\u,\v)$ are computed, and
the correlations $\corr(\X_{tr}\u,\Y_{tr}\v)$ and
$\corr(\X_{te}\u,\Y_{te}\v)$ are reported on Figure \ref{eq:correlation}. The
average standard deviation of the experiments results was $0.0339$ for the
training set and $0.1391$ for the test set. 

Comparing with the original curves reported by \citet{witten} 
for their penalized matrix decomposition (PMD) algorithm, our
method exhibits in general a performance similar as PMD.\footnote{The curves for PMD were 
generated with the R software package available at \url{http://cran.r-project.org/web/packages/PMA/index.html}
and a script provided by \citet{witten}.} Nevertheless, the purpose
of this section is more to demonstrate that our method can be used with 
genomic data than comparing it carefully with PMD. To draw substantial
conclusions about the performance of both methods, more experiments would of course be needed.
\begin{figure}[hbtp]
   \centerline{\includegraphics[width=0.8\linewidth]{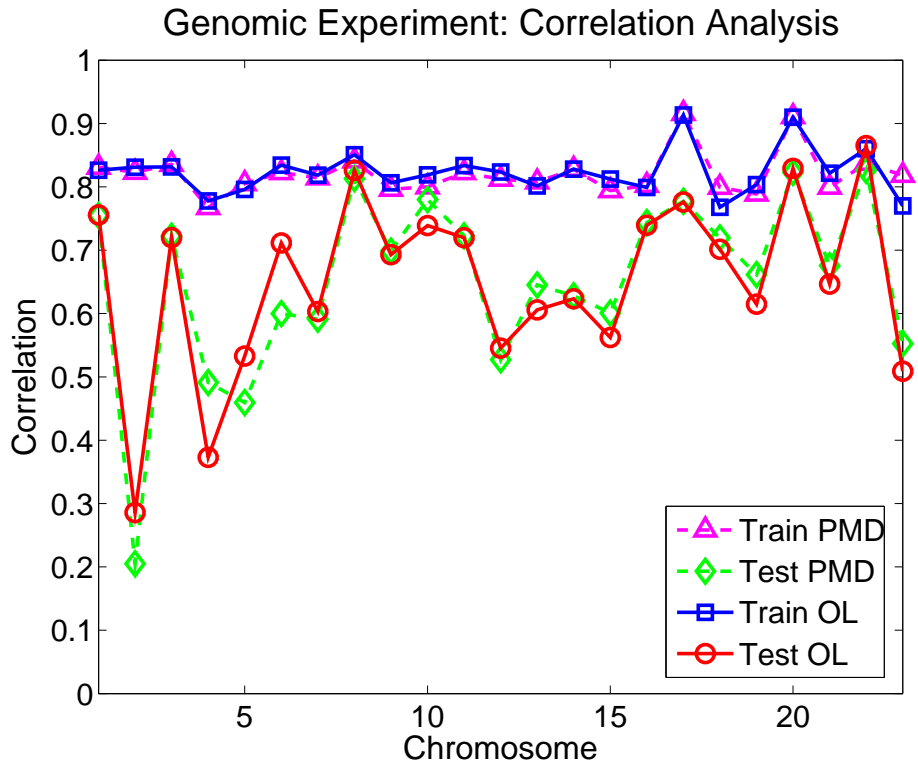}}
   \caption{SPCA was applied to the covariance matrix obtained from the breast cancer data \citep{chin}. A fused lasso regularization is used for the CGH data. $3 / 4$ of the $n$ samples are used as a training set, keeping the rest for testing. Average correlations from $10$ random splits are reported for each of the $23$ chromosomes, for \textsf{PMD} \citep{witten} and our method denoted by \textsf{OL}.
}\label{eq:correlation}
\end{figure}
\subsection{Application to Large-Scale Image Processing}
We demonstrate in this section that our algorithm can be used for a
difficult large-scale image processing task, namely, removing the text
({\em inpainting}) from the damaged $12$-Megapixel image of Figure
\ref{fig:fig3}. Using a multi-threaded version of our implementation,
we have learned a dictionary with $256$ elements from the roughly
$7\times 10^6$ undamaged $12\times12$ color patches in the image with
two epochs in about $8$ minutes on a 2.4GHz machine with eight
cores. Once the dictionary has been learned, the text is removed using
the sparse coding technique for inpainting
of \citet{mairal}. 
Our intent here is of course {\em not} to evaluate our
learning procedure in inpainting tasks, which would require a thorough
comparison with state-the-art techniques on standard
data sets. Instead, we just wish to demonstrate that it can indeed be
applied to a realistic, non-trivial image processing task on a large
image. Indeed, to the best of our knowledge, this is the first time
that dictionary learning is used for image restoration on such
large-scale data. For comparison, the dictionaries used for inpainting
in \citet{mairal} are learned (in batch
mode) on 200,000 patches only.
\begin{figure}[hbtp]
\centerline{\includegraphics[width=0.49\linewidth]{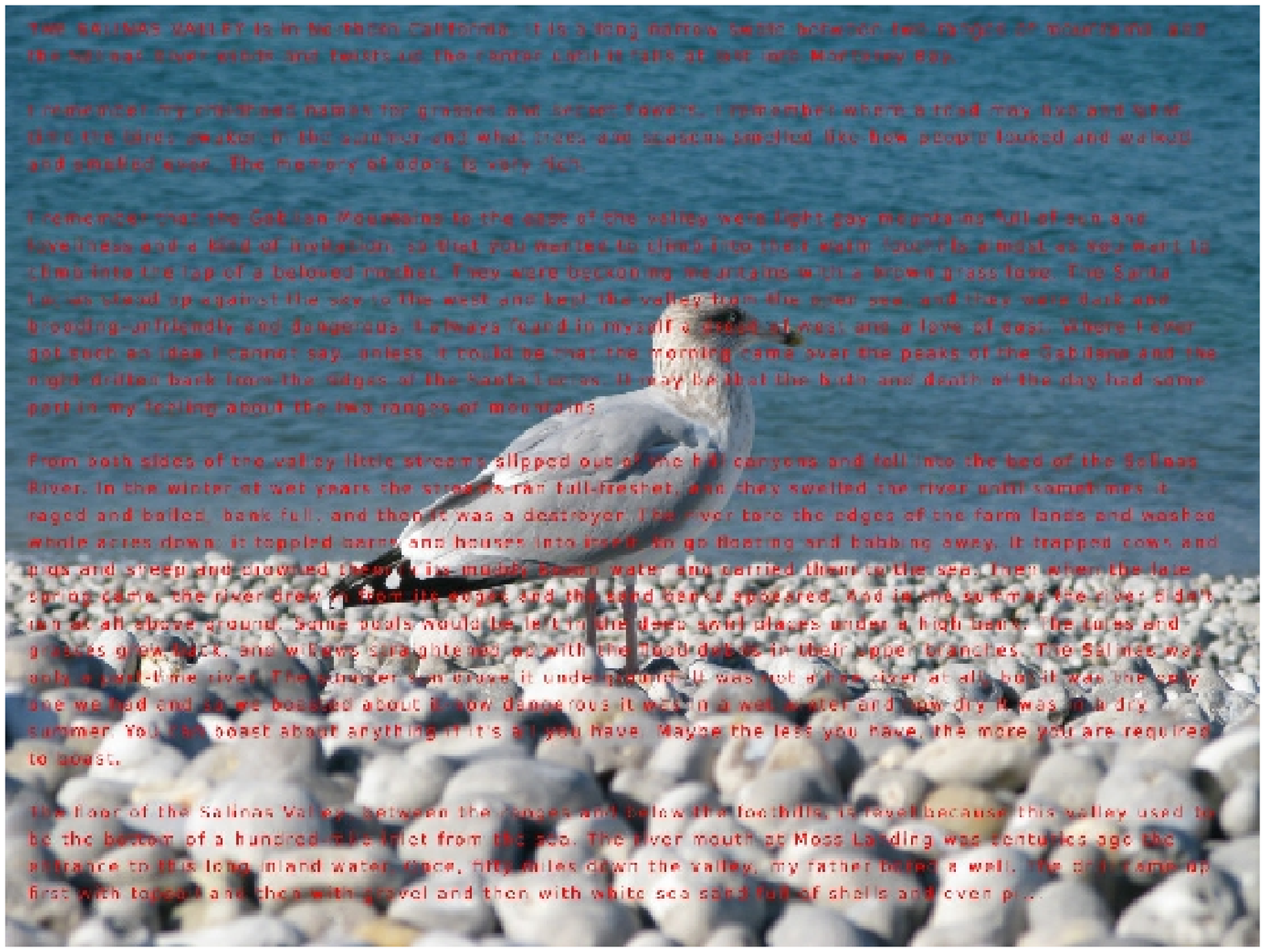}
\includegraphics[width=0.49\linewidth]{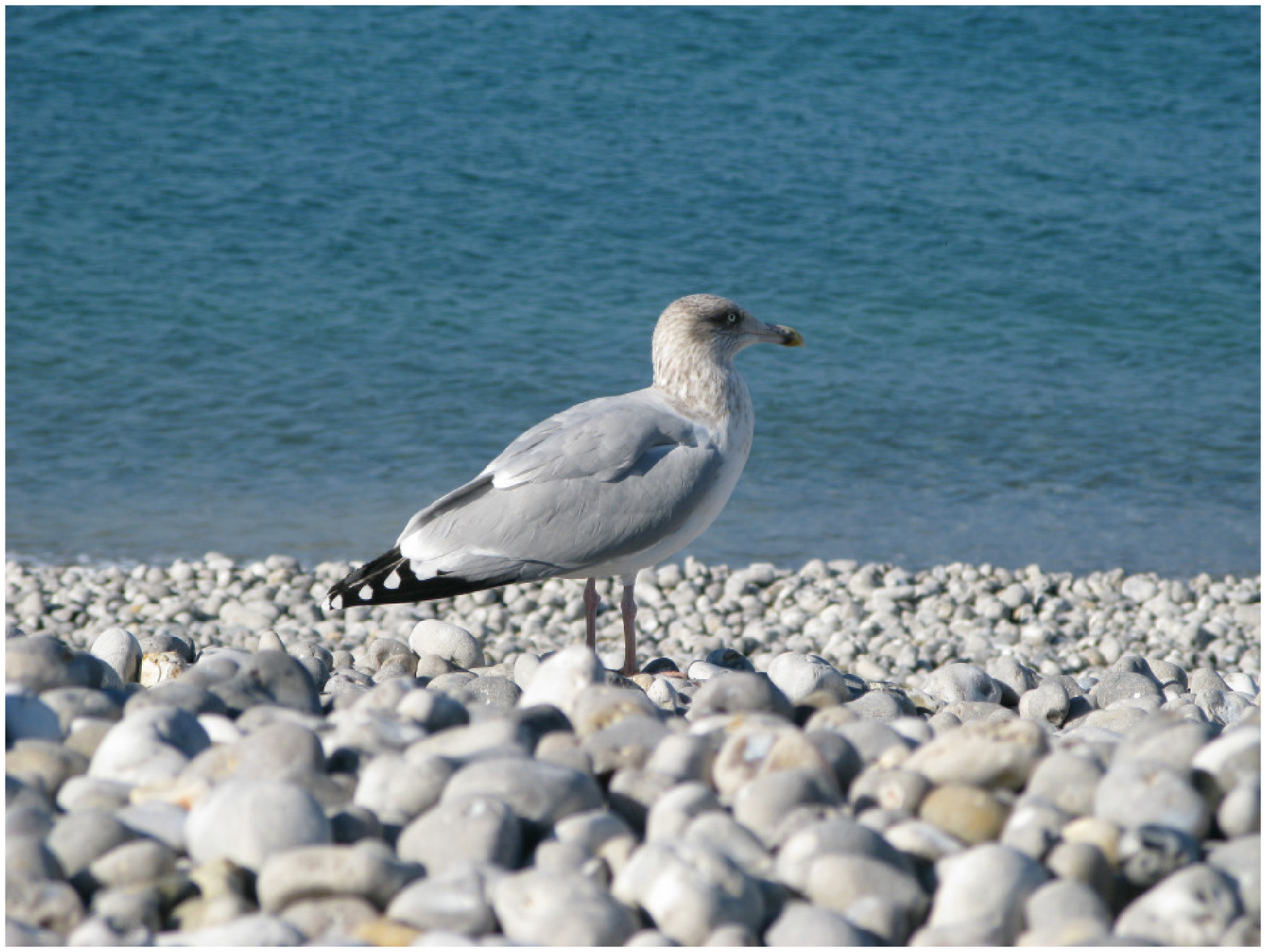}}
\vspace*{0.1cm}
\centerline{\includegraphics[width=0.49\linewidth]{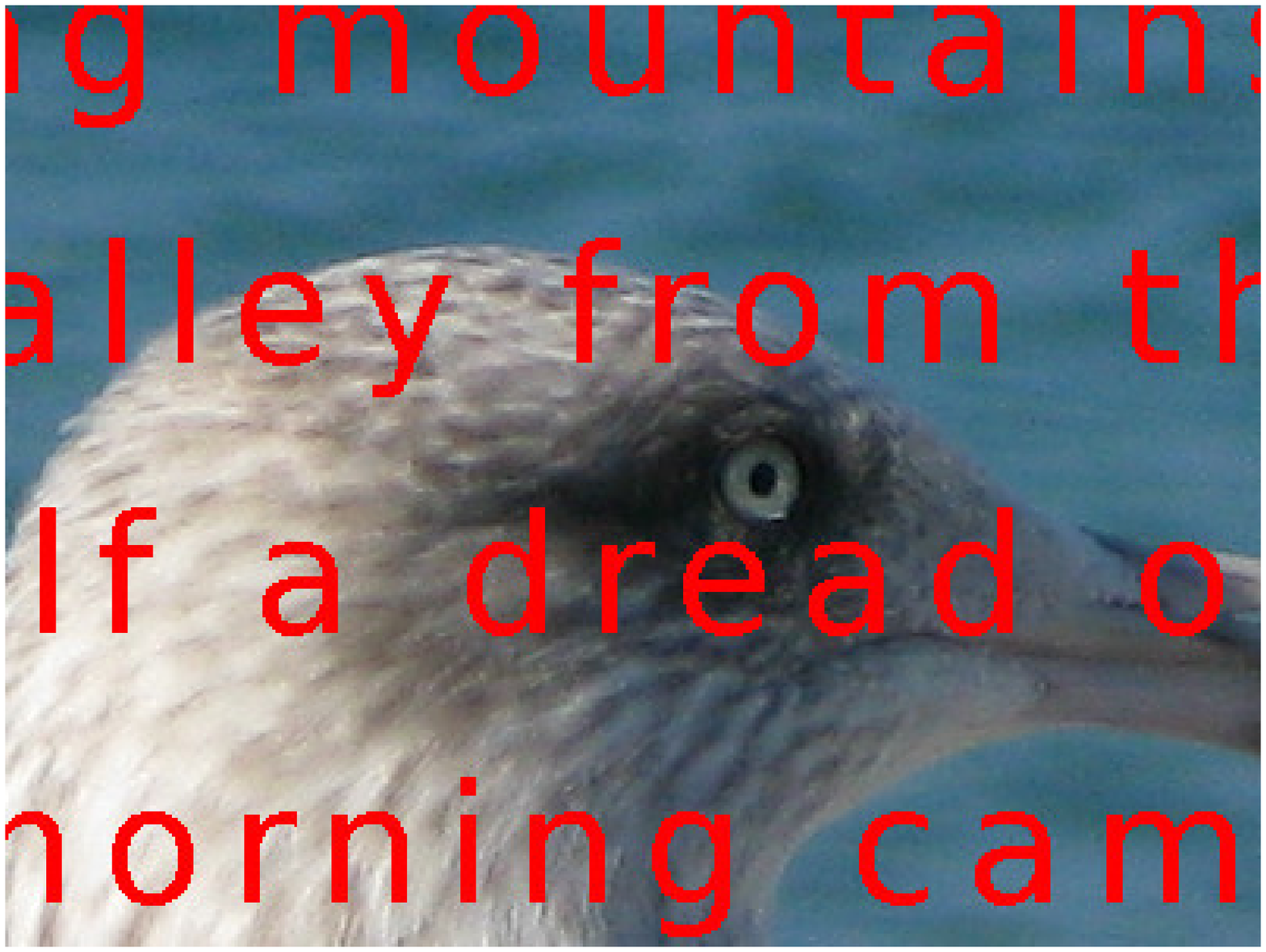}
\includegraphics[width=0.49\linewidth]{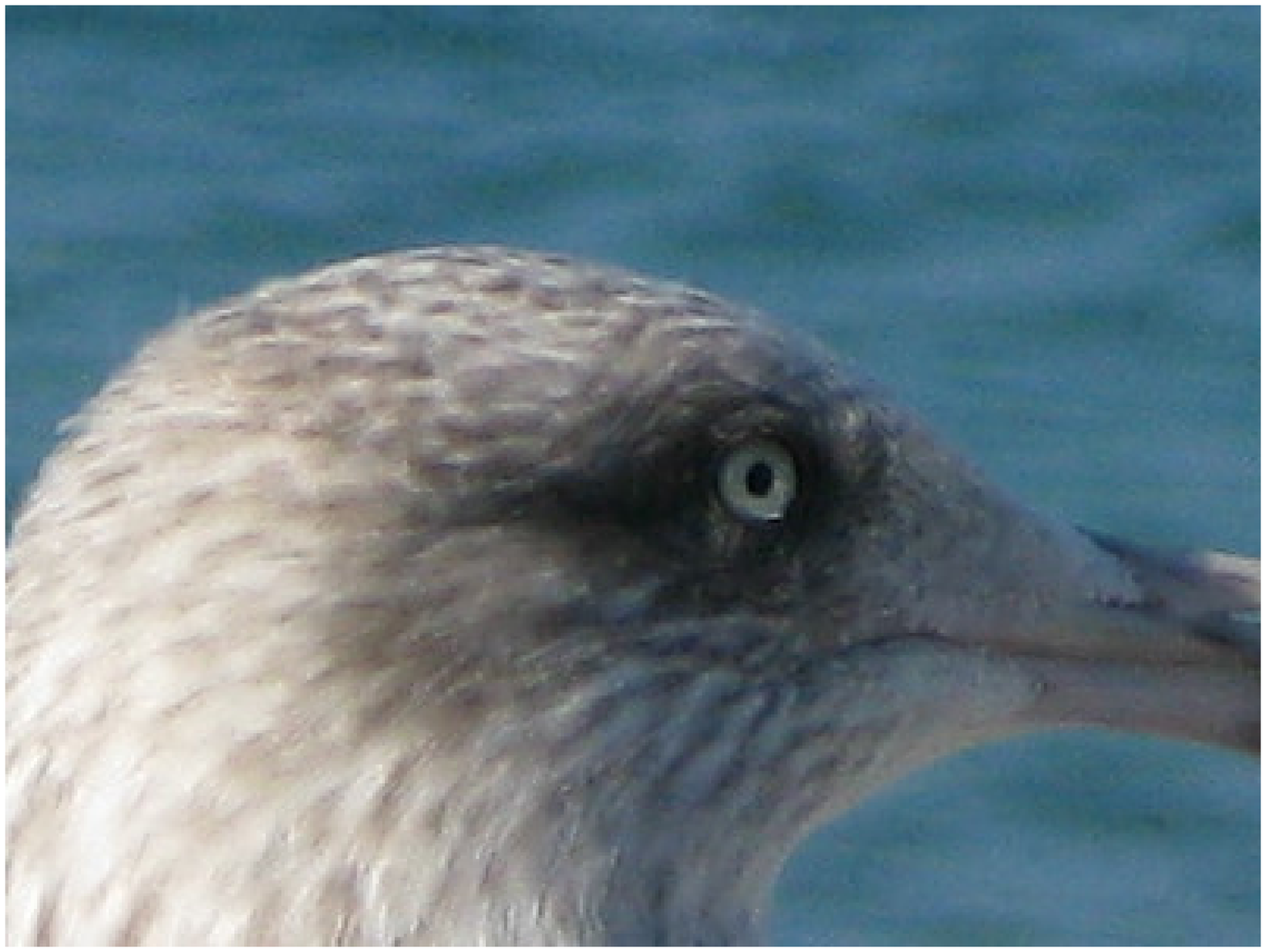}}
\caption{Inpainting example on a $12$-Megapixel image. Top: Damaged
and restored images. Bottom: Zooming on the damaged and restored
images. Note that the pictures
presented here have been scaled down for display. (Best seen in color).}\label{fig:fig3}
\end{figure}
\section{Conclusion}
We have introduced in this paper a new stochastic online algorithm for learning
dictionaries adapted to sparse coding tasks, and proven its convergence.
Experiments demonstrate that it is significantly faster than batch alternatives
such as \citet{engan}, \citet{aharon} and \citet{lee}
on large data sets that may contain millions of training examples, yet it does not
require a careful learning rate tuning like regular stochastic gradient descent
methods.  Moreover, we have extended it to other matrix factorization problems
such as non negative matrix factorization, and we have proposed a formulation
for sparse principal component analysis which can be solved efficiently using
our method. Our approach has already shown to be useful for
image restoration tasks such as denoising \citep{mairal8};
more experiments are of course needed to better assess its promise
in bioinformatics and signal processing.
Beyond this, we plan to use the proposed learning framework for
sparse coding in computationally demanding video restoration
tasks \citep{protter}, with dynamic data sets whose size is not fixed, and
extending this framework to different loss functions \citep{mairal6} to address
discriminative tasks such as image classification, which are more sensitive to
overfitting than reconstructive ones.
\acks{This paper was supported in part by ANR under grant MGA ANR-07-BLAN-0311.  The work of
Guillermo Sapiro is partially supported by ONR, NGA, NSF, ARO, and DARPA. 
The authors would like to thank Sylvain Arlot, L\'eon Bottou,
Jean-Philippe Vert, and the members of the Willow project-team
for helpful discussions, and Daniela Witten for providing us with her code to generate the 
curves of Figure \ref{eq:correlation}.}
\appendix
\section{Theorems and Useful Lemmas} \label{appendix:th}
We provide in this section a few theorems and lemmas from the optimization and probability literature,
which are used in this paper.
\begin{theorem}{ \bf [Corollary of Theorem 4.1 from \citet{bonnans}, due to \citet{danskin}].\\}
   \label{theo:bonnans}
   Let $f: \Real^p \times \Real^q \to \Real$. 
   Suppose that for all $\x \in \Real^p$ the function $f(\x,.)$ is
   differentiable, and that $f$ and $\nabla_{\u}f(\x,\u)$ the derivative of
   $f(\x,.)$ are continuous on $\Real^p \times \Real^q$. Let $v(\u)$ be the
   optimal value function $v(\u) = \min_{\x \in C} f(\x,\u)$,
   where $C$ is a compact subset of $\Real^p$. Then $v(\u)$ is
   directionally differentiable.  Furthermore, if for $\u_0 \in \Real^q$,
   $f(.,\u_0)$ has a unique minimizer $\x_0$ then $v(\u)$ is differentiable in 
   $\u_0$ and $\nabla_{\u} v(\u_0) = \nabla_{\u} f (\x_0,\u_0)$.
\end{theorem}
\begin{theorem}{\bf [Sufficient condition of convergence for a stochastic process, see \citet{bottou2} and references therein \citep{metivier,fisk}].\\}
   \label{theo:martingales}
   Let ($\Omega$,$\mathcal F$,$P$) be a measurable probability space, $u_t$, for $t \geq 0$,
   be the realization of a stochastic process and ${\mathcal F}_t$ be the
   filtration determined by the past information at time $t$. Let 
   \begin{displaymath} 
      \delta_t = \left\{ 
      \begin{array}{ll}
         1 & ~~\text{if}~~ \E[u_{t+1}-u_t | {\mathcal F}_t]>0, \\
         0 & ~~\text{otherwise.}
      \end{array}
      \right.
   \end{displaymath}
   If for all $t$, $u_t \geq 0$ and $\sum_{t=1}^\infty
   \E[\delta_t(u_{t+1}-u_t)] < \infty$, then $u_t$ is a quasi-martingale and
   converges almost surely. Moreover,
   \begin{displaymath}
      \sum_{t=1}^\infty |\E[u_{t+1}-u_t | {\mathcal F}_t]| < +\infty \as 
   \end{displaymath}
\end{theorem}
\begin{lemma}{\bf [A corollary of Donsker theorem \citealp[see][chap. 19.2, lemma
   19.36 and example 19.7]{vaart}].\\} \label{lemma:donsker}
   Let $F =\{ f_\theta: \chi \to \Real,  \theta \in \Theta \}$ be a
   set of measurable functions indexed by a bounded subset $\Theta$ of
   $\Real^d$. Suppose that there exists a constant $K$ such that
   \begin{displaymath}
      |f_{\theta_1}(x)-f_{\theta_2}(x)| \leq K ||\theta_1-\theta_2||_2,
   \end{displaymath}
   for every $\theta_1$ and $\theta_2$ in $\Theta$ and $x$ in $\chi$.
   Then, $F$ is P-Donsker (see \citealp[chap. 19.2]{vaart}).
   For any $f$ in $F$, Let us define $\PPP_n f$, $\PPP f$ and $\G_n f$ as
    \begin{displaymath}
       \PPP_n f = \frac{1}{n}\sum_{i=1}^n f(X_i), ~~~~  \PPP f = \E_X[f(X)], ~~~~ \G_n f = \sqrt{n}(\PPP_n f - \PPP f).
    \end{displaymath}
   Let us also suppose that for all $f$,
   $\PPP f^2 < \delta^2$ and $||f||_\infty < M$ and that the random elements
   $X_1,X_2,\ldots$ are Borel-measurable.
   Then, we have
   \begin{displaymath}
      \E_P||\G_n||_F = \O(1),
   \end{displaymath}
   where $||\G_n||_F = \sup_{f \in F} |\G_n f|$.
   For a more general variant of this
   lemma and additional explanations and examples, see \citet{vaart}.
\end{lemma}
\begin{lemma}{\bf [A simple lemma on positive converging sums].\\ } \label{lemma:converg}
   Let $a_n$, $b_n$ be two
   real sequences such that for all $n$, $a_n \geq 0, b_n \geq 0$, $\sum_{n=1}^\infty a_n = \infty$, $\sum_{n=1}^\infty a_n b_n < \infty$, $\exists K >0 \st |b_{n+1}-b_n| < K a_n$.
   Then, $\lim_{n \to +\infty} b_n = 0$.
\end{lemma}
\begin{proof}
    The proof is similar to \citet[prop 1.2.4]{bertsekas}.
\end{proof}

\section{Efficient Projections Algorithms} \label{appendix:proj}
In this section, we address the problem of efficiently projecting a vector
onto two sets of constraints, which allows us to extend our algorithm to
various other formulations.
\subsection{A Linear-time Projection on the Elastic-Net Constraint} \label{appendix:sec:elas}
Let $\b$ be a vector of $\Real^m$. We consider the problem of projecting this vector
onto the elastic-net constraint set:
\begin{equation}
   \min_{\u \in \Real^m} \frac{1}{2}||\b-\u||_2^2 \st ||\u||_1 +
   \frac{\gamma}{2}||\u||_2^2 \leq \tau.\label{eq:proj}
\end{equation}
To solve efficiently the case $\gamma > 0$, we 
propose Algorithm \ref{fig:proj_elastic}, which extends \citet{maculan} and \citet{duchi}, and the following lemma
which shows that it solves our problem.
\begin{lemma}{\bf [Projection onto the elastic-net constraint set].\\}
For $\b$ in $\Real^m$, $\gamma \geq 0$ and $\tau > 0$,
Algorithm \ref{fig:proj_elastic} solves Eq.~(\ref{eq:proj}).
\end{lemma}
\begin{proof}
   First, if $\b$ is a feasible point of (\ref{eq:proj}), then $\b$ is a
   solution.  We suppose therefore that it is not the case---that is,$||\b||_1
   + \frac{\gamma}{2}||\b||_2^2 > \tau$.  Let us define the Lagrangian of
   (\ref{eq:proj})
   \begin{displaymath}
      \L(\u,\lambda) = \frac{1}{2}||\b-\u||_2^2 + \lambda\big(||\u||_1 +
      \frac{\gamma}{2}||\u||_2^2 - \tau\big).
   \end{displaymath}
   For a fixed $\lambda$, minimizing the Lagrangian with respect to $\u$ admits
   a closed-form solution~$\u^\star(\lambda)$, and a simple calculation shows that, for all $j$,
   \begin{displaymath}
      \u^\star(\lambda)[j] = \frac{\sign(\b[j])(|\b[j]|-\lambda)^+}{1+\lambda\gamma}.
   \end{displaymath}
   Eq.~(\ref{eq:proj}) is a convex optimization problem. Since Slater's
   conditions are verified and strong duality holds, it is
   equivalent to the dual problem
   \begin{displaymath}
      \max_{\lambda \geq 0} \L(\u^\star(\lambda),\lambda).
   \end{displaymath}
   Since $\lambda=0$ is not a solution, denoting by $\lambda^\star$ the solution,
   the complementary slackness condition implies that 
   \begin{equation}
      ||\u^\star(\lambda^\star)||_1 +
      \frac{\gamma}{2}||\u^\star(\lambda^\star)||_2^2 = \tau.
      \label{eq:projunic}
   \end{equation}
   Using the closed form of $\u^\star(\lambda)$ is possible to show that the function
   $\lambda \to ||\u^\star(\lambda)||_1 + \frac{\gamma}{2}||\u^\star(\lambda)||_2^2$,
   is strictly decreasing with $\lambda$ and thus Eq.~(\ref{eq:projunic}) is
   a necessary and sufficient condition of optimality for~$\lambda$.
   After a short calculation, one can show that this optimality condition is equivalent to
   \begin{displaymath}
      \frac{1}{(1+\lambda\gamma)^2} \sum_{ j \in S(\lambda)} \Big( |\b[j]| +
      \frac{\gamma}{2}|\b[j]|^2 -
      \lambda\big(1+\frac{\gamma\lambda}{2}\big)\Big)= \tau,
   \end{displaymath}
   where $S(\lambda) = \{ j \st |\b[j]| \geq \lambda \}$.
   Suppose that $S(\lambda^\star)$ is known, then $\lambda^\star$ can be
   computed in closed-form.
   To find $S(\lambda^\star)$, it is then sufficient to find the index
   $k$ such that $S(\lambda^\star) = S(|\b[k]|)$, 
   which is the solution of
   \begin{displaymath}
      \max_{k \in \{1,\ldots,m\}} |\b[k]| \st \frac{1}{(1+|\b[k]|\gamma)^2} \sum_{ j \in S(|\b[k]|)}\Big( |\b[j]| +
      \frac{\gamma}{2}|\b[j]|^2 -
      |\b[k]|\big(1+\frac{\gamma|\b[k]|}{2}\big)\Big) < \tau.
   \end{displaymath}
   Lines 4 to 14 of Algorithm \ref{fig:proj_elastic} are a modification of \citet{duchi} to 
   address this problem. A similar proof as \citet{duchi} shows the convergence to the solution
   of this optimization problem in $\O(m)$ in the average case,
   and lines 15 to 18 of Algorithm \ref{fig:proj_elastic}) compute $\lambda^\star$ after that $S(\lambda^\star)$ 
   has been identified.
   Note that setting $\gamma$ to $0$ leads exactly to the algorithm of \citet{duchi}.
\end{proof}

\begin{algorithm}
   \caption{Efficient projection on the elastic-net constraint.}
   \label{fig:proj_elastic}
   \begin{algorithmic}[1]
      \REQUIRE $\tau \in \Real$; $\gamma \in \Real$; $\b \in \Real^m$;
      \IF{$||\b||_1 + \frac{\gamma}{2}||\b||_2^2 \leq \tau$}
         \RETURN $\u \leftarrow \b$.
      \ELSE
      \STATE $U \leftarrow \{1,\ldots,m\}$; $s \leftarrow 0$; $\rho \leftarrow 0$.
      \WHILE{$U \neq \emptyset$}
      \STATE Pick $k \in U$ at random.
      \STATE Partition $U$:
      \begin{displaymath}
         \begin{split}
            G &= \{ j \in U \st |\b[j]| \geq |\b[k]| \}, \\
            L &= \{ j \in U \st |\b[j]| < |\b[k]| \}. \\
         \end{split}
      \end{displaymath}
      \STATE $\Delta\rho \leftarrow |G|$; $\Delta s \leftarrow \sum_{j \in G} |\b[j]|+\frac{\gamma}{2}|\b[j]|^2$.
      \IF{ $s + \Delta s - (\rho + \Delta \rho)(1+ \frac{\gamma}{2}|\b[k]|)|\b[k]|  < \tau (1+\gamma |\b[k]|)^2$}
      \STATE $s \leftarrow s + \Delta s; \rho \leftarrow \Delta \rho; U \leftarrow L$.
      \ELSE
      \STATE $U \leftarrow G \setminus \{k \}$.
      \ENDIF
      \ENDWHILE
      \STATE $a \leftarrow \gamma^2\tau+\frac{\gamma}{2}\rho$,
      \STATE $b \leftarrow 2\gamma\tau+\rho$,
      \STATE $c \leftarrow \tau-s$,
      \STATE $\lambda \leftarrow \frac{-b+\sqrt{b^2-4 a c}}{2a}$
      \STATE $$\forall j=1,\ldots,n, \u[j] \leftarrow \frac{\sign(\b[j])(|\b[j]|-\lambda)^+}{1+\lambda\gamma}$$
      \RETURN $\u$.
      \ENDIF
   \end{algorithmic}
\end{algorithm}
As for the dictionary learning problem, a simple modification to
Algorithm \ref{fig:proj_elastic} allows us to handle the non-negative
case, replacing the scalars $|\b[j]|$ by $\max(\b[j],0)$ in the algorithm.
\subsection{A Homotopy Method for Solving the Fused Lasso Signal Approximation} \label{appendix:sec:fused}
Let $\b$ be a vector of $\Real^m$. We define, following \citet{friedman}, the fused
lasso signal approximation problem $\P(\gamma_1,\gamma_2,\gamma_3)$:
\begin{equation}
   \min_{\u \in \Real^m} \frac{1}{2}||\b-\u||_2^2 + \gamma_1||\u||_1 +
   \gamma_2 \FL(\u) + \frac{\gamma_3}{2}||\u||_2^2, \label{eq:flsa}
\end{equation}
the only difference with \citet{friedman} being the addition of the last
quadratic term.
The method we propose to this problem is a homotopy, which solves
$\P(\tau\gamma_1,\tau\gamma_2,\tau\gamma_3)$ for all
possible values of $\tau$. In particular, for all $\varepsilon$,
it provides the solution of the constrained problem
\begin{equation}
   \min_{\u \in \Real^m} \frac{1}{2}||\b-\u||_2^2 \st \gamma_1||\u||_1 +
   \gamma_2 \FL(\u) + \frac{\gamma_3}{2}||\u||_2^2 \leq \varepsilon.  \label{eq:fusedConstrained}
\end{equation}

The algorithm relies on the following lemma
\begin{lemma}
   Let $\u^\star(\gamma_1,\gamma_2,\gamma_3)$ be the solution of
   Eq.~(\ref{eq:flsa}), for specific values of $\gamma_1,\gamma_2,\gamma_3$.
   Then
   \begin{itemize}
      \item $\u^\star(\gamma_1,\gamma_2,\gamma_3) =
         \frac{1}{1+\gamma_3}\u^\star(\gamma_1,\gamma_2,0)$.
      \item For all $i$, $\u^\star(\gamma_1,\gamma_2,0)[i] =
         \sign(\u^\star(0,\gamma_2,0)[i])\max(|\u^\star(0,\gamma_2,0)[i]|-\lambda_1,0)$---that
         is, $\u^\star(\gamma_1,\gamma_2,0)$ can be obtained by soft thresholding of $\u^\star(0,\gamma_2,0)$.
   \end{itemize}
\end{lemma}
The first point can be shown by short calculation. The second one is proven
in \citet{friedman} by considering subgradient optimality conditions.  This
lemma shows that if one knows the solution of $\P(0,\gamma_2,0)$, then
$\P(\gamma_1,\gamma_2,\gamma_3)$ can be obtained in linear time.

It is therefore natural to consider the simplified problem
\begin{equation}
   \min_{\u \in \Real^m} \frac{1}{2}||\b-\u||_2^2  +
   \gamma_2 \FL(\u).  \label{eq:fusedSimple2}
\end{equation}
With the change of variable $\v[1]=\u[1]$ and $\v[i]=\u[i]-\u[i-1]$ for $i > 1$,
this problem can be recast as a weighted Lasso
\begin{equation}
   \min_{\v \in \Real^m} \frac{1}{2}||\b-\D\v||_2^2  +
   \sum_{i=1}^m w_i |\v[i]|, \label{eq:fusedSimple}
\end{equation}
where $w_1=0$ and $w_i=\gamma_2$ for $i > 1$, and $\D[i,j] = 1$ if $i \geq j$
and $0$ otherwise. We propose to use LARS
\citep{efron} and exploit the specific structure of the matrix $\D$ to make this
approach efficient, by noticing that:
\begin{itemize}
   \item For a vector $\w$ in $\Real^m$, computing $\e=\D\w$ requires $O(m)$
      operations instead of $O(m^2)$, by using the recursive formula
      $\e[1]=\w[1]$, $\e[i+1]=\w[i]+\e[i]$.
   \item For a vector $\w$ in $\Real^n$, computing $\e=\D^T\w$ requires $O(m)$
      operations instead of $O(m^2)$, by using the recursive formula
      $\e[n]=\w[n]$, $\e[i-1]=\w[i-1]+\e[i]$.
   \item Let $\Gammab = \{a_1,\ldots,a_p\}$ be an active set and suppose $a_1 < \ldots < a_p$. 
      Then $(\D_\Gammab^T\D_\Gammab)^{-1}$ admits the closed form value
      \begin{displaymath}
         (\D_\Gammab^T\D_\Gammab)^{-1} = \left(
         \begin{array}{cccccc}
            c_1 & -c_1 & 0 & \ldots & 0 & 0\\
            -c_1 & c_1+c_2 & -c_2 & \ldots & 0 & 0 \\
            0 & -c_2 & c_2+c_3 & \ldots & 0 & 0\\
            \vdots & \vdots & \vdots & \ddots & \vdots & \vdots \\
            0 & 0 & 0 & \ldots & c_{p-2}+c_{p-1} & -c_{p-1} \\
            0 & 0 & 0 & \ldots & -c_{p-1} & c_{p-1}+c_{p}
         \end{array}
         \right),
      \end{displaymath}
      where $c_p = \frac{1}{n+1-a_p}$ and $c_i = \frac{1}{a_{i+1}-a_{i}}$ for $i < p$.
\end{itemize}
This allows the implementation of this homotopy method without using matrix inversion 
or Cholesky factorization, solving Eq.~(\ref{eq:fusedSimple}) in $O(m s)$ operations,
where $s$ is the number of non-zero values of the optimal solution
$\v$.\footnote{To be more precise, $s$ is the number of kinks of the
regularization path. In practice, $s$ is roughly the same as the number of
non-zero values of the optimal solution $\v$.}

Adapting this method for solving Eq.~(\ref{eq:fusedConstrained}) requires 
following the regularization path of the problems $\P(0,\tau\gamma_2,0)$ for
all values of $\tau$, which provides as well the regularization path of the problem
$\P(\tau\lambda_1,\tau\lambda_2,\tau\lambda_3)$ and stops whenever the constraint
becomes unsatisfied. This procedure still requires $O(ms)$ operations.

Note that in the case $\gamma_1=0$ and $\gamma_3=0$, when only the fused-lasso
term is present in Eq~(\ref{eq:flsa}), the same approach has been proposed in a
previous work by \citet{harchaoui}, and \citet{harchaoui2} to solve
Eq.~(\ref{eq:fusedSimple2}), with the same tricks for improving the efficiency
of the procedure.
\bibliography{mairal10a_arxiv}
\end{document}